\pdfoutput=1
\documentclass[review]{elsarticle}
\usepackage{lineno}
\usepackage{multirow}
\usepackage{etex}
\usepackage{framed}
\usepackage{amsmath}
\usepackage{amstext}
\usepackage{amsfonts}
\usepackage{verbatim}
\usepackage{amssymb}
\usepackage{a4wide,amscd}
\usepackage[linesnumbered,ruled]{algorithm2e}
\usepackage{algorithmic}      
\usepackage{graphicx}
\usepackage{bm}
\usepackage{color}
\usepackage{footnote}
\makesavenoteenv{tabular}
\makesavenoteenv{table}
\definecolor{vdarkred}{rgb}{0.6,0,0.2}
\definecolor{vdarkblue}{rgb}{0,0.2,0.6}
\usepackage{atbegshi}
\AtBeginDocument{\AtBeginShipoutNext{\AtBeginShipoutDiscard}}
\usepackage{tikz}
\usepackage{pgfplots,filecontents}
\usetikzlibrary{spy}
\usetikzlibrary{shapes}
\usetikzlibrary{plotmarks}

\tikzset{new spy style/.style={spy scope={%
	magnification=5,
	size=1.25cm,
	connect spies,
	every spy on node/.style={
		rectangle,
		draw,
	},
	every spy in node/.style={
		draw,
		rectangle,
		fill=gray!40,
	}
}}}

\usepackage{optional}

\newcommand{\asto}{\overset{\rm a.s.}{\longrightarrow}}

\newcommand{\Var}{\mathrm{Var}}

\DeclareMathOperator{\tr}{tr}

\DeclareMathOperator{\diag}{\rm diag}
\DeclareMathOperator{\argmin}{argmin}

\newcounter{ctheorem}
\newtheorem{theorem}[ctheorem]{Theorem}
\newcounter{cassumption}
\newtheorem{assumption}[cassumption]{Assumption}

\newproof{proof}{Proof}

\newcounter{cproposition}
\newtheorem{proposition}[cproposition]{Proposition}
\newcounter{ccorollary}
\newtheorem{corollary}[ccorollary]{Corollary}
\newcounter{clemma}
\newtheorem{lemma}[clemma]{Lemma}

\newcounter{cremark}
\newtheorem{remark}[cremark]{Remark}

\definecolor{darkgreen}{rgb}{0.125,0.5,0.169}
\modulolinenumbers[5]

\journal{Journal of Multivariate Analysis}

\begin{document}


\title{Spectral community detection in heterogeneous large networks\tnoteref{t1}}
\tnotetext[t1]{The work is supported by the ANR Project RMT4GRAPH (ANR-14-CE28-0006).}

\author{Hafiz TIOMOKO ALI\corref{cor1}\fnref{fn1}}
\ead{hafiz.tiomokoali@centralesupelec.fr}
\author{Romain COUILLET\corref{cor2}\fnref{fn2}}
\ead{romain.couillet@centralesupelec.fr}
\address{CentraleSup\'elec -- LSS -- Universit\'e Paris Saclay\\ 3 rue Joliot Curie, Gif-sur-Yvette, France}
\thanks{The work is supported by the ANR Project RMT4GRAPH (ANR-14-CE28-0006).}





\begin{abstract}
In this article, we study spectral methods for community detection based on $ \alpha$-parametrized normalized modularity matrix hereafter called $ {\bf L}_\alpha $ in heterogeneous graph models. We show, in a regime where community detection is not asymptotically trivial, that $ {\bf L}_\alpha $ can be well approximated by a more tractable random matrix which falls in the family of spiked random matrices. The analysis of this equivalent spiked random matrix allows us to improve spectral methods for community detection and assess their performances in the regime under study. In particular, we prove the existence of an optimal value $ \alpha_{\rm opt} $ of the parameter $ \alpha $ for which the detection of communities is best ensured and we provide an on-line estimation of $ \alpha_{\rm opt} $ only based on the knowledge of the graph adjacency matrix. Unlike classical spectral methods for community detection where clustering is performed on the eigenvectors associated with extreme eigenvalues, we show through our theoretical analysis that a regularization should instead be performed on those eigenvectors prior to clustering in heterogeneous graphs. Finally, through a deeper study of the regularized eigenvectors used for clustering, we assess the performances of our new algorithm for community detection. Numerical simulations in the course of the article show that our methods outperform state-of-the-art spectral methods on dense heterogeneous graphs.
\end{abstract}	








\maketitle


\section{Introduction and Motivations}
\label{sec:intro}
The advent of the big data era is creating an unprecedented need for automating large network analysis. Community detection is among the most important tasks in automated network mining \cite{fortunato2010community}. Given a network graph, detecting communities consists in retrieving hidden clusters of nodes based on some similarity metric (the edges are dense inside communities and sparse across communities). While quite simple to define, community detection is usually not an easy task and many methods arising from different fields have been proposed to carry it out. The most important of them are statistical inference, modularity maximization and graph partitioning methods. Statistical inference methods consist in fitting the observed network to a structured network model and infer its parameters (among which the assignment of the nodes to the communities)\cite{hastings2006community,newman2007mixture}. Modularity maximization algorithms rely on the modularity metric which quantifies the subdivision of networks into communities \cite{fortunato2010community}.\footnote{Precisely, the modularity is defined as the difference between the total number of edges inside the communities for a given partition and the total number of edges if the partition was created randomly in the graph.} The maximization of this quantity over all possible partitions in the graph gives the best possible subdivision of this graph in the modularity measure sense. However, this is generally an NP-hard problem and many approximation methods have been proposed based on some polynomial-time heuristics: greedy methods \cite{newman2004fast}, simulated annealing \cite{guimera2004modularity}, extremal optimization \cite{duch2005community} and spectral methods \cite{newman2006modularity}. Spectral algorithms consist in retrieving the communities from the eigenvectors associated with the extreme eigenvalues of some matrix representation of the graph structure (adjacency matrix, modularity matrix, Laplacian matrix). By relaxing the modularity optimization problem from discrete values of the community memberships to continuous values, it is shown that approximate modularity maximization and even statistical inference methods can be performed via a low dimensional clustering of the entries of the dominant eigenvectors of the Laplacian matrix \cite{ng2002spectral,newman2016community} in polynomial time. We focus in this article on the latter methods. Precisely, spectral methods for community detection generally follow successively those steps 
\begin{enumerate}
\item Compute the, say, $\ell$ eigenvectors corresponding to the extreme (largest or smallest) eigenvalues of one of the matrix representations of the network (adjacency, modularity, Laplacian).
\item Stack those eigenvectors column-wise in a matrix $ {\bf W} \in \mathbb{R}^{n \times \ell}$ with $ n $ the number of nodes, or correspondingly the size of the matrix representation of the network.
\item Take each row of $ {\bf W} $ as a (feature) vector in a $ \ell$-dimensional (feature) space.
\item Cluster those $ n $ vectors in $ K $ groups using a standard classification algorithm e.g., k-means or expectation maximization (EM). The EM algorithm, for example, aims to roughly identify clusters at first before to sequentially update the individual cluster means and covariances until convergence. \label{four2}
\end{enumerate}

Real world networks are in general sparse in that the number of connections of each node (degree) scales in $ \mathcal{O}(1) $ when the number of nodes $ n $ grows large. When the degrees scale instead like $ \mathcal{O}(\log n) $ or $ \mathcal{O}(n) $, the network is said to be dense. The standard spectral algorithms based on the network matrix (adjacency, modularity, Laplacian) of sparse graphs are generally suboptimal in the sense that they fail to detect the communities down to the transition point where the detection is theoretically feasible \cite{krzakala2013spectral}. New operators (non-backtracking \cite{krzakala2013spectral}, Bethe Hessian \cite{saade2014spectral}) based on statistical physics have recently been proposed and have been shown to perform well down to the aforementioned sparse regime. We focus however in this article on dense networks for which spectral methods are often optimal \cite{krzakala2013spectral}. 

Most of the works proposing statistical analysis of the performance of community detection (for dense as well as sparse networks) consider the basic Stochastic Block Model (SBM) as a model for networks decomposable into communities. Denoting $\mathcal{G}$ a $K$-class graph of $n$ vertices with communities $\mathcal{C}_1,\ldots,\mathcal{C}_K$ with $ g_i $ the group assignment of node $ i $, the SBM assumes an adjacency matrix ${\bf A}\in\{0,1\}^{n\times n}$, with $A_{ij}$ independent Bernoulli random variables with parameter $P_{g_ig_j}$ where $ P_{ab} $ represents the probability that any node of class $ \mathcal{C}_a $ is connected to any node of class $ \mathcal{C}_b $. The main limitation of this model is that it is only suited to homogeneous graphs where all nodes have the same average degree in each community (besides, class sizes are often taken equal). A more realistic model, the Degree-Corrected SBM (DCSBM), was proposed in \cite{coja2009finding, karrer2011stochastic} to account for degree heterogeneity inside communities. For the same graph $\mathcal{G}$ defined above, by letting $q_i$, $ 1\leq i\leq n$, be some intrinsic weights which affect the probability for node $i$ to connect to any other network node, the adjacency matrix ${\bf A}\in\{0,1\}^{n\times n}$ of the graph generated by the DCSBM is such that  $A_{ij}$ are independent Bernoulli random variables with parameter $q_{i}q_{j}C_{g_ig_j}$, where $C_{g_ig_j}$ is a class-wise correction factor.

The main motivation of this work arises from the observation that classical spectral algorithms based on the adjacency matrix (modularity, Laplacian) may drastically fail to detect the genuine communities in some synthetic graphs generated using the DCSBM. The same observation is made even for the aforementioned and very competitive Bethe Hessian (BH) method.\footnote{The Bethe Hessian (BH) spectral method \cite{saade2014spectral} is based on the union of the eigenvectors associated to the negative eigenvalues of $ H(r_c) $ and $ H(-r_c) $ respectively where $ H(r)=(r^2-1){\bf I}_n-r{\bf A}+{\bf D} $ with $ {\bf I}_n $ the identity matrix of size $ n $, $ {\bf D} $ a diagonal matrix containing the degrees on the diagonal and $ r_c=\frac{\sum_{i} d_i^2}{\sum_{i} d_i}-1$ with $d_i$ the degree of node $ i $.} To illustrate those limitations of spectral methods under the DCSBM, the two graphs of Figure~\ref{fig:2Dplot} provide 2D representations of dominant eigenvector~1 versus eigenvector~2 for the standard modularity matrix and the BH matrix, when half the nodes connect with low weight $q_{(1)}$ and half the nodes with high weight $q_{(2)}$. For both methods, it is clear that k-means or EM alike would erroneously induce the detection of extra communities and even a confusion of genuine communities in the BH approach. We have come to understand that those extra communities are produced by some biases created by the heterogeneity of the intrinsic weights $ q_{i} $'s; intuitively, nodes sharing the same intrinsic connection weights tend to create their own sub-cluster inside each community, thereby forming additionnal sub-communities inside the genuine communities.

\begin{figure}[h!]
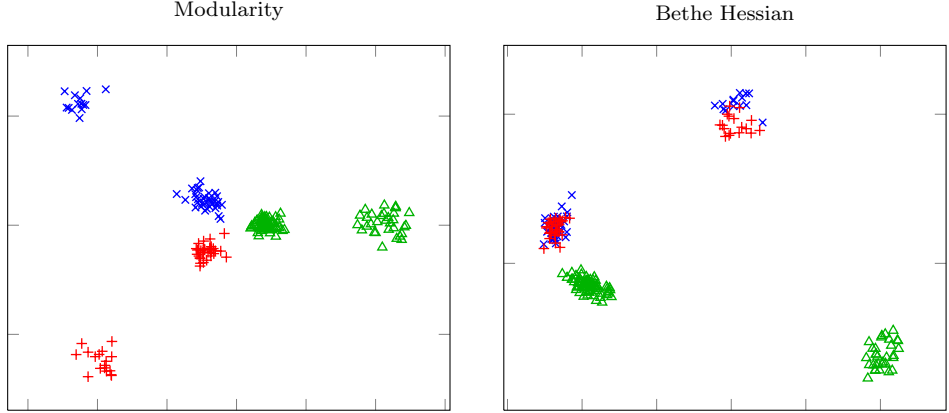

  \centering

  \caption{Two dominant eigenvectors (x-y axes) for $n=2000$, $K=3$ classes $ \mathcal C_1 $, $ \mathcal C_2 $ and $ \mathcal C_3 $ of sizes $|\mathcal{C}_{1}|=|\mathcal{C}_{2}|=\frac{n}{4}$, $ |\mathcal{C}_{3}|=\frac{n}{2} $, $\frac34$ of the nodes having $ q_i=0.1$ and $\frac14$ of the nodes having $ q_i=0.5$, matrix of weights ${\bf C}={\bf 1}_{3}{\bf 1}_{3}^{\sf T}+\frac{100}{\sqrt{n}} {\bf I}_3$. Colors and shapes correspond to ground truth classes.
  }
  \label{fig:2Dplot}
\end{figure}

In order to understand the aforementioned limitations and the different mechanisms into play when using spectral methods based on matrices derived from the adjacency matrix $ {\bf A} $ (such as the modularity matrix), we study here a generalized version of the normalized modularity matrix, given for $ \alpha \in \mathcal{A} \subset \mathbb{R} $, by
\begin{align}
\label{eq:L}
	{\bf L}_{\alpha} =(2m)^{\alpha}\frac{1}{\sqrt{n}}\mathbf{D}^{-\alpha} \left[ {\bf A} - \frac{{\bf d}{\bf d}^{\sf T}}{2m} \right] \mathbf{D}^{-\alpha}
\end{align}
where $ {\bf d} $ is the vector of degrees ($ d_i =\sum_{j=1}^{n} A_{ij}$), $ {\bf D} $ is the diagonal matrix of degrees (containing $ {\bf d} $ on the main diagonal) and $ m=\frac{1}{2}{\bf d}^{\sf T}{\bf 1}_n $ is the number of edges in the network. In particular, $ {\bf L}_{0} $ is the modularity matrix, $ {\bf L}_{\frac12} $ is a modularity equivalent to the Laplacian matrix and $ {\bf L}_{1} $ was studied in \cite{gulikers2015spectral,ali2016performance}.

In the dense DCSBM model where $ q_i=\mathcal{O}(1) $ (with growing $ n $), when the correction factors ${C_{g_ig_j}} $ differ by $ \mathcal{O}(1) $, the classification is trivial as asymptotically vanishing error rates are easily guaranteed. We thus place ourselves in the more interesting regime where $ {C_{g_ig_j}}=\mathcal{O}(1)$ individually but the ${C_{g_ig_j}} $'s differ by $ \mathcal{O}(n^{-\frac12}) $. We study the dominant eigenvalues and associated eigenvectors (used for classification) of $ {\bf L}_{\alpha} $ for large dimensional dense graphs following the DCSBM in the aforementioned ``non-trivial" regime.

In a nutshell, our main findings are as follows
\begin{itemize}
	\item We show that, as $ n \to \infty $, $ {\bf L}_{\alpha} $ can be arbitrarily well approximated by a theoretically tractable random matrix $ \tilde{{\bf L}}_{\alpha} $ which falls in the family of so-called spiked random matrix models and which allows for a thorough understanding of spectral methods based on $ {\bf L}_{\alpha} $. Those random matrices generally exhibit a phase transition beyond which useful information can be extracted from the eigenvectors associated to outlying eigenvalues (and below which nothing can be said). In our context, this phase transition corresponds to a \emph{community detectability threshold}, common in community detection algorithms analysis. We characterize exactly this phase transition for each value of $ \alpha $.
	\item We prove the existence of an optimal value $ \alpha_{\rm opt} $ of $ \alpha $ for which the aforementioned phase transition is maximally achievable.
	\item We provide a consistent estimator $ \hat{\alpha}_{\rm opt} $ of $ \alpha_{\rm opt} $ based on $ {\bf d} $ alone.
	\item A thorough analysis of the spiked random matrix model then shows that, to achieve consistent clustering in the DCSBM model, the dominant eigenvectors used for clustering should be pre-multiplied by $\mathbf{D}^{\alpha-1}$ prior to the low dimensional classification (step~(\ref{four2}) of the spectral algorithm described previously).\footnote{Our results generalize the work in \cite{gulikers2015spectral} where it was shown that for $ \alpha=1 $, a consistent clustering can be obtained by using the dominant eigenvectors of $ {\bf L}_\alpha $ without any regularization since for $ \alpha=1 $, $\mathbf{D}^{\alpha-1}={\bf I}_n$ where $ {\bf I}_n $ is the identity matrix of size $ n .$}
	\item A deeper study of those regularized eigenvectors allows us to
	\begin{itemize}
	\item improve the initial setting of the EM algorithm (in the step (\ref{four2}) of the spectral algorithm described above) in comparison with a random setting.
	\item find the theoretical clustering error rate of spectral community detection in the regime under study.
	\end{itemize}
	\item Numerical simulations (throughout the article) show that our methods outperform state-of-the-art techniques both on synthetic graphs and on real world networks.
\end{itemize}

\medskip

{\it Notations:} Vectors (matrices) are denoted by lowercase (uppercase) boldface letters. $ \left\{ {\bf v}_{a}\right\}_{a=1}^{n} $ is the column vector $ \mathbf{v} $ with (scalar or vector) entries $v_{a} $ and $ \left\{ {\bf V}_{ab}\right\}_{a,b=1}^{n} $ is the matrix $ \mathbf{V} $ with (scalar or matrix) entries $V_{ab} $. For a vector $ {\bf v} $, the operator $  \mathcal{D}(\mathbf{v})=\mathcal{D}\left( \left\{ {v}_{a} \right\}_{a=1}^{n} \right) $ is the diagonal matrix having the scalars $v_a$ down its diagonal and for a matrix $ {\bf V} $, $  \mathcal{D}(\bf{V})$ is the vector containing the diagonal entries of $ {\bf V} $. The vector $ \mathbf{1}_{n} \in \mathbb{R}^{n}$ stands for the column vector filled with ones. The Dirac measure at $ x $ is $ \delta_{x} $. The vector $ \mathbf{j}_{a} $ is the canonical vector of class $ \mathcal{C}_{a} $ defined by $(\mathbf{j}_{a})_{i}=\mathbf{\delta}_{i \in \mathcal{C}_{a}}$ and $\mathbf{J}=\left[\mathbf{j}_{1},\ldots,\mathbf{j}_{K}\right] \in \{0,1\}^{n \times K}$. The set $\mathbb{C}^+$ is $\{z\in\mathbb{C},~\Im[z]>0\}$.

\section{Preliminaries}
\label{sec:prelim}
This section describes the network model under study, which is based on the DCSBM defined in the previous section, and provides preliminary technical results.
\bigskip

We consider an $n$-node random graph with $K$ classes $\mathcal{C}_1,\ldots,\mathcal C_K$ of sizes $|\mathcal{C}_{k}|=n_k$. Each node is characterized by an intrinsic connexion weight $ q_{i} $ which affect the probability that this node gets attached to another node in the graph. A null model would consider that the existence of an edge between $ i $ and $ j $ has probability $ q_{i}q_{j} $. In order to take into account the membership of the nodes to some group, we define $ {\bf C}\in \mathbb{R}^{K\times K} $ as a matrix of class weights $C_{ab}$, independent of the $ q_{i} $'s, affecting the connection probability between nodes in $\mathcal C_a$ and nodes in $\mathcal C_b$. 

As such, following \cite{karrer2011stochastic}, the adjacency matrix $ {\bf A} $ of the graph generated from a DCSBM model has independent entries (up to symmetry) which are Bernoulli random variables with parameter $ P_{ij}=q_{i}q_{j}{C_{g_ig_j}}\in (0,1)$ where $ g_i $ is the group assignment of node $ i $. We set $ A_{ii}=0 $ for all $i$. In the dense regime under consideration, $ q_{i}=\mathcal{O}(1)$ and $ C_{g_ig_j}=\mathcal{O}(1)$ as $ n \rightarrow \infty $. For convenience of exposition and without loss of generality, we discard the nodes having no neighbor and we assume that node indices are sorted by clusters i.e, nodes $ 1 $ to $ n_1 $ constitutes $ \mathcal{C}_1 $, nodes $ n_1+1 $ to $ n_1+n_2 $ form $ \mathcal{C}_2 $ and so on.

The matrix under study is given by 
\begin{align}
\label{eq:L}
	{\bf L}_{\alpha} =(2m)^{\alpha}\frac{1}{\sqrt{n}}\mathbf{D}^{-\alpha} \left[ {\bf A} - \frac{{\bf d}{\bf d}^{\sf T}}{2m} \right] \mathbf{D}^{-\alpha}
\end{align}
for some $\alpha\in \mathcal{A}$, a compact subset of $\mathbb{R}$, where $ {\bf d}={\bf A1}_n $ , $ {\bf D}=\mathcal{D}({\bf d})$ and $ m=\frac{1}{2}{\bf d}^{\sf T}{\bf 1}_n $.

As stated in the introduction, we are mainly interested in a dense network regime where clustering is not asymptotically trivial. This regime is ensured by the following growth rate conditions.
\begin{assumption}\label{as1} 
As $ n \to \infty $, $ K $ remains fixed and, for all $i,j \in \{1,\ldots,n\}$
\begin{enumerate}
	\item $ {C_{g_ig_j}}=1+\frac{M_{g_ig_j}}{\sqrt{n}} $, where $ M_{g_ig_j}=\mathcal{O}(1) $; we shall denote ${\bf M}=\{M_{ab}\}_{a,b=1}^K$.
\item $q_{i}$ are i.i.d.\@ random variables with measure $\mu$ having compact support in $ (0,1) $.
\item $\frac{n_{i}}{n} \rightarrow c_{i}>0$ and we will denote $ \mathbf{c}=\left\{c_k\right\}_{k=1}^{K} $ (in particular, $ {\bf c}^{\sf T}{\bf 1}_K=1 $).
\end{enumerate}
\end{assumption}

Before delving into the main technical results, we provide a uniform consistent estimator of the (a priori unknown) intrinsic weight $ q_i $ which shall be used in the course of the article. 
\begin{lemma}
\label{lm:consistentEstimate}
Under Assumption~\ref{as1},
\begin{align} \label{eq:estimator}
	\max_{1\leq i\leq n}\left|q_{i}-\hat{q}_i\right|\to 0
\end{align}
almost surely, where $\hat{q}_i=\frac{d_i}{\sqrt{{\bf d}^{\sf T}{\bf 1}_n}} $. 
\end{lemma}
Note that $q_i$ can be retrieved from the empirical graph degrees irrespective of the class matrix ${\bf C}$, which is a direct (and important) consequence of Assumption~\ref{as1}-(1).
\bigskip

The first goal of the article is to study deeply the eigenstructure of $ {\bf L}_{\alpha} $. As can be observed, $ {\bf L}_{\alpha} $ has non independent entries as $ {\bf D} $ (and $ {\bf d} $) depend on $ {\bf A} $, and it does not follow a standard random matrix model. Our strategy is to approximate $ {\bf L}_{\alpha} $ by a more tractable random matrix which asymptotically preserves the eigenvalue distribution and isolated eigenvectors of $ {\bf L}_{\alpha} $. Before providing the complete proof in Section~ \ref{proofApprox}, let us give the main steps of the approximation. As per our model and assumptions, the random variable $ A_{ij} $ is Bernoulli distributed with parameter $ P_{ij}=q_{i}q_{j}(1+M_{g_ig_j}/ \sqrt{n}) $ so that its mean is $q_{i}q_{j}(1+M_{g_ig_j}/\sqrt{n})$ and variance $ q_{i}q_{j}(1+M_{g_ig_j}/\sqrt{n})(1-q_{i}q_{j}(1+M_{g_ig_j}/\sqrt{n})) $. We may thus write 
\[A_{ij}=q_{i}q_{j}+q_{i}q_{j}\frac{M_{g_ig_j}}{\sqrt{n}}+X_{ij}\]
where $ X_{ij} $, $ 1\leq i<j \leq n $, are independent zero mean random variables of variance $ q_{i}q_{j}(1-q_{i}q_{j}) +\mathcal{O}(n^{-\frac12})$. The normalized adjacency matrix is thus\footnote{Here subscript `$d,n^k$' stands for deterministic term of order $n^k$ and `$r,n^k$' for random term (of operator norm) of order $n^k$.} 
\begin{align}
\label{eq:A}
	\frac1{\sqrt{n}} \mathbf{A} &= \underbrace{\frac{1}{\sqrt{n}}\mathbf{q}\mathbf{q}^{\sf T} }_{\mathbf{A}_{d,\sqrt{n}}} + \underbrace{\frac{1}{n} \left\{ \mathbf{q}_{(a)}\mathbf{q}_{(b)}^{\sf T}M_{ab} \right\}_{a,b=1}^{K}}_{\mathbf{A}_{d,1}} + \underbrace{\frac{1}{\sqrt{n}}\mathbf{X}}_{\mathbf{A}_{r,1}},
 \end{align}
where $ \mathbf{q}_{(i)}=[ q_{n_{1}+\ldots+n_{i-1}+1},\ldots,q_{n_{1}+\ldots+n_{i}}]^{\sf T}\in \mathbb{R}^{n_i}$ ($n_{0}=0$) and $\mathbf{X}=\left\{X_{ij}\right\}_{i,j=1}^n$. Note that the right-hand side of \eqref{eq:A} is composed of a dominant (in terms of operator norm) matrix ${\bf A}_{d,\sqrt{n}}$ of order $ \|{\bf A}_{d,\sqrt{n}}\|=\mathcal{O}(\sqrt{n}) $ and of smaller order terms. From there, we may then provide a Taylor expansion of $ {\bf dd}^{\sf T} $, $ (2m)^{-1}=({\bf d}^{\sf T}{\bf 1}_n)^{-1} $, $(2m)^{\alpha}=({\bf d}^{\sf T}{\bf 1}_n)^{\alpha} $ and $ {\bf D}^{-\alpha} $ around their dominant terms. By grouping all those expansions consistently following the structure of Equation \eqref{eq:L} and by only keeping non-vanishing operator norm terms, we obtain the corresponding approximate of $ {\bf L}_\alpha $ as follows
 
 \begin{theorem} \label{approx}
	Let Assumption \ref{as1} hold and let $ {\bf L}_{\alpha} $ be given by \eqref{eq:L}. Then, for $ \mathbf{D}_{q}=\mathcal{D}({\bf q}) $, as $n\to \infty $, $ \| {\bf L}_{\alpha} - \tilde{\mathbf{L}}_{\alpha} \| \to 0 $ in operator norm, almost surely, where
	\begin{align*}
		\tilde{\mathbf{L}}_{\alpha} &=\frac{1}{\sqrt{n}}\mathbf{D}_{q}^{-\alpha}\mathbf{X}\mathbf{D}_{q}^{-\alpha}+\mathbf{U}{\bm \Lambda}\mathbf{U}^{\sf T},\\
	\mathbf{U} &= \begin{bmatrix}
	\frac{\mathbf{D}_{q}^{1-\alpha}{\bf J}}{\sqrt{n}} & \frac{\mathbf{D}_{q}^{-\alpha}{\bf X}{\bf 1}_{n}}{{\bf q}^{\sf T}{\bf 1}_{n}}
	\end{bmatrix}, \\
	{\bm \Lambda} &= \begin{bmatrix}
	\left({\bf I}_{K}-{\bf 1}_{K}{\bf c}^{T}\right){{\bf M}}\left({\bf I}_{K}-{\bf c}{\bf 1}_{K}^{T}\right) & -{\bf 1}_{K} \\
	-{\bf 1}_{K}^{T} & 0
	\end{bmatrix},
	\end{align*}
	where we recall that $\mathbf{J}=\left[\mathbf{j}_{1},\ldots,\mathbf{j}_{K}\right] \in \{0,1\}^{n \times K}$ and $(\mathbf{j}_{a})_{i}=\mathbf{\delta}_{g_i=a}$.
\end{theorem}
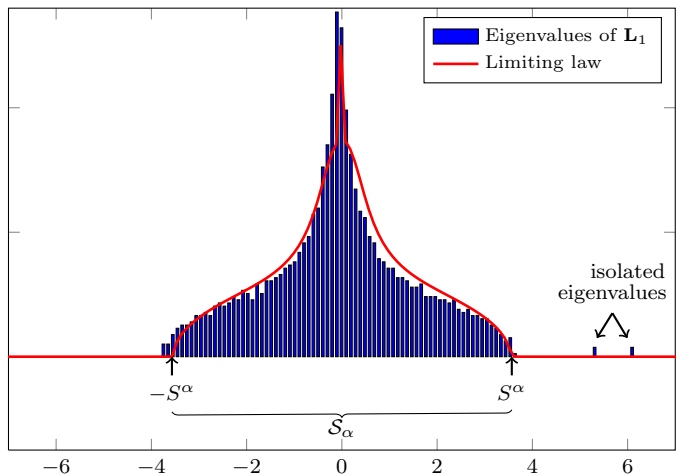
\begin{figure}[h!]
  \centering
  \begin{tikzpicture}[font=\footnotesize]
    \renewcommand{\axisdefaulttryminticks}{4} 
    \tikzstyle{every major grid}+=[style=densely dashed]       
    \tikzstyle{every axis y label}+=[yshift=-10pt] 
    \tikzstyle{every axis x label}+=[yshift=5pt]
    \tikzstyle{every axis legend}+=[cells={anchor=west},fill=white,
        at={(0.98,0.98)}, anchor=north east, font=\scriptsize ]
    \begin{axis}[
      height=0.5\textwidth,
      width=0.7\textwidth,
      xmin=-7,
      ymin=-0.15,
      xmax=7,
      ymax=0.56,
      yticklabels={},
      bar width=1pt,
      grid=none,
      ymajorgrids=false,
      scaled ticks=false,
      ]
      \addplot[area legend,ybar,mark=none,color=black,fill=blue] coordinates{
(-3.747266,0.020322)(-3.648850,0.020322)(-3.550434,0.035563)(-3.452018,0.045724)(-3.353603,0.050805)(-3.255187,0.050805)(-3.156771,0.055885)(-3.058355,0.066046)(-2.959939,0.066046)(-2.861523,0.071127)(-2.763107,0.066046)(-2.664691,0.081288)(-2.566275,0.086368)(-2.467859,0.081288)(-2.369443,0.086368)(-2.271027,0.096529)(-2.172611,0.091449)(-2.074195,0.106690)(-1.975779,0.101610)(-1.877363,0.091449)(-1.778947,0.116851)(-1.680531,0.101610)(-1.582115,0.121931)(-1.483699,0.121931)(-1.385283,0.127012)(-1.286867,0.132092)(-1.188451,0.142253)(-1.090035,0.147334)(-0.991619,0.152414)(-0.893203,0.167656)(-0.794788,0.182897)(-0.696372,0.193058)(-0.597956,0.228621)(-0.499540,0.238782)(-0.401124,0.304829)(-0.302708,0.340392)(-0.204292,0.421680)(-0.105876,0.553772)(-0.007460,0.528370)(0.090956,0.396277)(0.189372,0.325151)(0.287788,0.269265)(0.386204,0.233702)(0.484620,0.223541)(0.583036,0.193058)(0.681452,0.182897)(0.779868,0.157495)(0.878284,0.152414)(0.976700,0.142253)(1.075116,0.142253)(1.173532,0.127012)(1.271948,0.127012)(1.370364,0.121931)(1.468780,0.111770)(1.567196,0.111770)(1.665611,0.116851)(1.764027,0.096529)(1.862443,0.096529)(1.960859,0.096529)(2.059275,0.091449)(2.157691,0.091449)(2.256107,0.086368)(2.354523,0.091449)(2.452939,0.076207)(2.551355,0.071127)(2.649771,0.076207)(2.748187,0.066046)(2.846603,0.066046)(2.945019,0.060966)(3.043435,0.060966)(3.141851,0.050805)(3.240267,0.045724)(3.338683,0.035563)(3.437099,0.025402)(3.535515,0.030483)(3.633931,0.005080)(3.732347,0.000000)(3.830763,0.000000)(3.929179,0.000000)(4.027595,0.000000)(4.126011,0.000000)(4.224426,0.000000)(4.322842,0.000000)(4.421258,0.000000)(4.519674,0.000000)(4.618090,0.000000)(4.716506,0.000000)(4.814922,0.000000)(4.913338,0.000000)(5.011754,0.000000)(5.110170,0.000000)(5.208586,0.000000)(5.307002,0.015080)(5.405418,0.000000)(5.503834,0.000000)(5.602250,0.000000)(5.700666,0.000000)(5.799082,0.000000)(5.897498,0.000000)(5.995914,0.000000)(6.094330,0.015080)
      };
      \addplot[mark=none,color=red,line width=1pt] coordinates{
(-10,0)(-4.747266,0.000022)(-4.648850,0.000024)(-4.550434,0.000026)(-4.452018,0.000028)(-4.353603,0.000030)(-4.255187,0.000034)(-4.156771,0.000037)(-4.058355,0.000042)(-3.959939,0.000049)(-3.861523,0.000059)(-3.763107,0.000074)(-3.664691,0.000105)(-3.566275,0.000254)(-3.467859,0.023133)(-3.369443,0.035697)(-3.271027,0.044721)(-3.172611,0.052119)(-3.074195,0.058541)(-2.975779,0.064297)(-2.877363,0.069565)(-2.778947,0.074463)(-2.680531,0.079077)(-2.582115,0.083471)(-2.483699,0.087699)(-2.385283,0.091806)(-2.286867,0.095836)(-2.188451,0.099829)(-2.090035,0.103825)(-1.991619,0.107868)(-1.893203,0.112004)(-1.794788,0.116284)(-1.696372,0.120771)(-1.597956,0.125535)(-1.499540,0.130662)(-1.401124,0.136258)(-1.302708,0.142457)(-1.204292,0.149426)(-1.105876,0.157384)(-1.007460,0.166615)(-0.909044,0.177497)(-0.810628,0.190521)(-0.712212,0.206310)(-0.613796,0.225546)(-0.515380,0.248648)(-0.416964,0.274962)(-0.318548,0.301824)(-0.220132,0.325096)(-0.121716,0.341227)(-0.023300,0.500002)(0.075116,0.346223)(0.173532,0.333778)(0.271948,0.313545)(0.370364,0.287854)(0.468780,0.260818)(0.567196,0.236005)(0.665611,0.214950)(0.764027,0.197610)(0.862443,0.183361)(0.960859,0.171535)(1.059275,0.161576)(1.157691,0.153055)(1.256107,0.145649)(1.354523,0.139109)(1.452939,0.133246)(1.551355,0.127911)(1.649771,0.122987)(1.748187,0.118379)(1.846603,0.114009)(1.945019,0.109812)(2.043435,0.105731)(2.141851,0.101718)(2.240267,0.097729)(2.338683,0.093722)(2.437099,0.089656)(2.535515,0.085490)(2.633931,0.081181)(2.732347,0.076679)(2.830763,0.071925)(2.929179,0.066844)(3.027595,0.061337)(3.126011,0.055258)(3.224426,0.048374)(3.322842,0.040248)(3.421258,0.029783)(3.519674,0.012672)(3.618090,0.000138)(3.716506,0.000085)(3.814922,0.000065)(3.913338,0.000053)(4.011754,0.000045)(4.110170,0.000040)(4.208586,0.000035)(4.307002,0.000032)(4.405418,0.000029)(4.503834,0.000027)(4.602250,0.000025)(4.700666,0.000023)(4.799082,0.000021)(4.897498,0.000020)(4.995914,0.000019)(5.094330,0.000018)(5.192746,0.000017)(5.291162,0.000016)(5.389578,0.000015)(5.487994,0.000014)(5.586410,0.000013)(5.684825,0.000013)(5.783241,0.000012)(5.881657,0.000012)(5.980073,0.000011)(6.078489,0.000011)(6.176905,0.000010)(6.275321,0.000010)(6.373737,0.000010)(6.472153,0.000009)(6.570569,0.000009)(6.668985,0.000009)(6.767401,0.000008)(6.865817,0.000008)(6.964233,0.000008)(7.062649,0.000007)(10,0)
      };
      \draw[->,thick] (axis cs:-3.5680,-0.03) -- (axis cs:-3.5680,0) node [below,pos=0,font=\footnotesize] (ms) {$-S^\alpha$};
      \draw[->,thick] (axis cs:3.5680,-0.03) -- (axis cs:3.5680,0) node [below,pos=0,font=\footnotesize] (ps) {$S^\alpha$};
     \draw[decorate,decoration={brace,mirror}] (ms.south) -- node[below] {$\mathcal S_\alpha$} (ps.south);
\draw[->,thick] (axis cs:5.68,0.07) -- (axis cs:5.354324,0.03)node [above,pos=-1,font=\footnotesize] (fs) {isolated};
\draw[->,thick] (axis cs:5.68,0.07) -- (axis cs:5.995914,0.03)node [above,pos=0,font=\footnotesize] (fs) {eigenvalues};
      %
      %
      \legend{{Eigenvalues of ${\bf L}_1$},{Limiting law}}
    \end{axis}
  \end{tikzpicture}

\caption{Histogram of the eigenvalues of ${\bf L}_{1}$, $K=3$, $n=2000$, $c_1=0.3, c_2=0.3, c_3=0.4$, $\mu=\frac12\delta_{q_{(1)}}+\frac12\delta_{q_{(2)}}$, $q_{(1)}=0.4 $, $q_{(2)}=0.9 $, ${\bf M}$ defined by $M_{ii}=12$, $M_{ij}=-4, i\neq j$.}
  \label{fig:hist_L_synthetic}
\end{figure}

As far as the spectral analysis is concerned, $ \tilde{\mathbf{L}}_{\alpha} $ is asymptotically equivalent to $\mathbf{L}_\alpha$, as they asymptotically share the same set of eigenvalues and isolated eigenvectors.\footnote{That is, eigenvectors associated to eigenvalues found at non-vanishing distance of any other eigenvalue.} Thus, for large enough $ n $, the spectral analysis of $\mathbf{L}_\alpha$ can be performed through that of $ \tilde{\mathbf{L}}_{\alpha} $. Interestingly, $ \tilde{\mathbf{L}}_{\alpha} $ is an additive spiked random matrix \cite{baik2005phase} as it is the sum of a standard full rank random matrix $n^{-\frac12}\mathbf{D}_{q}^{-\alpha}\mathbf{X}\mathbf{D}_{q}^{-\alpha}$ (symmetric matrix having independent entries of zero mean and $ \mathcal{O}(n^{-1}) $ variances) and a low rank matrix $\mathbf{U}{\bm \Lambda}\mathbf{U}^{\sf T}$. As shown in Figure~\ref{fig:hist_L_synthetic}, the spectrum (eigenvalue distribution) of spiked random matrices is generally composed of (one or several) bulks of concentrated eigenvalues and, when a phase transition is met, of additionnal eigenvalues which isolate from the aforementioned bulks. The eigenvectors associated to those isolated eigenvalues contain important information related to the low rank matrix up to some noise; the more those eigenvalues isolate from the bulks, the lesser noise is contained in the corresponding eigenvectors. More specifically, the eigenvectors of the spiked random matrix become more correlated to the eigenvectors of the low rank matrix as the isolated eigenvalues are far away from the phase transition threshold. 

From Theorem~\ref{approx}, we see that the low rank matrix $ \mathbf{U}{\bm \Lambda}\mathbf{U}^{\sf T} $ contains the matrix $\mathbf{D}_{q}^{1-\alpha}{\bf J}$; so in our case, when the phase transition is met, the eigenvectors of $ \tilde{\mathbf{L}}_{\alpha} $ will be correlated to some extent to $\mathbf{D}_{q}^{1-\alpha}{\bf J}$. But, for a consistent clustering, one expects instead the vectors used for classification to be correlated to the canonical vectors $ {\bf j}_a $, $ 1 \leq a \leq K $. \emph{An important outcome of this first preliminary result is thus that the eigenvectors of $ {\bf L}_\alpha $ should be pre-multiplied by $ \mathbf{D}^{\alpha-1} $ prior to the classification}.\footnote{As far as the eigenvectors are concerned, we may freely replace $ {\bf D}_q $ (unknown in practice) by $ {\bf D} $ (which can be computed from the observed graph) since, from Lemma~\ref{lm:consistentEstimate}, the vector of degrees $ {\bf d} $ is, up to the scale factor $ \frac{1}{\sqrt{{\bf d}^{\sf T}{\bf 1}_n}} $, a uniform consistent estimator of the vector of intrinsic weights $ {\bf q} $ and thus $\|\frac{{\bf D}}{\sqrt{{\bf d}^{\sf T}{\bf 1}_n}}-{\bf D}_q\|\to 0$ almost surely.} This first result helps correcting the biases (creation of artificial classes) introduced by the degree heterogeneity when using classical spectral methods (as observed earlier in Figure~\ref{fig:2Dplot}). As shown in Figure~\ref{fig:correct}, which assumes the same setting as Figure~\ref{fig:2Dplot}, when the aforementioned eigenvector regularization is performed prior to EM or k-means classification, the genuine communities are correctly recovered. 
\begin{figure}[h!]
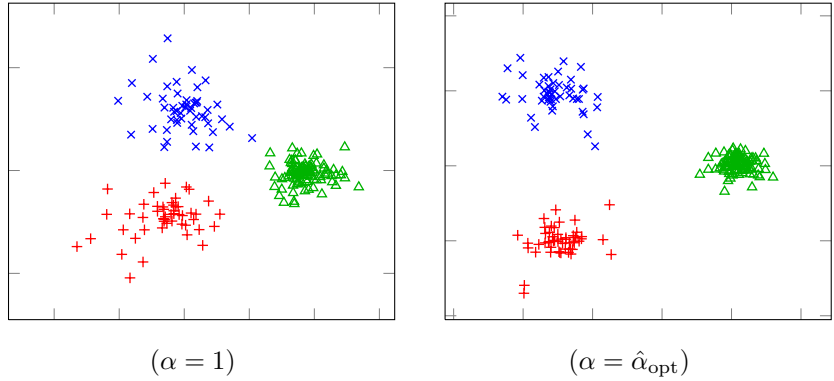

\centering

  \caption{Two dominant eigenvectors of $ {\bf L}_\alpha $ pre-multiplied by $ \mathbf{D}^{\alpha-1} $(x-y axes) for $n=2000$, $K=3$, $\mu=\frac34\delta_{q_{(1)}}+\frac14\delta_{q_{(2)}}$, $q_{(1)}=0.1$, $q_{(2)}=0.5$, $c_1=c_2=\frac14$, $c_3=\frac12$, ${\bf M}=100{\bf I}_3$ with $\hat{\alpha}_{\rm opt}$ defined in Section \ref{subsec:optimal}. Same setting as Figure~\ref{fig:2Dplot}.
  }
  \label{fig:correct}
\end{figure}

As is classical in the analysis of spiked random matrices, our next task is to study the isolated eigenvalues of $ {\bf L}_\alpha $ and their associated eigenvectors. This will in particular allow us to $\textit{i)}$ evidence the phase transition phenomenon discussed earlier which corresponds here to a \emph{community detectability threshold} and $\textit{ii)}$ evaluate the per-class average means and average covariances of the eigenvectors used for clustering, thereby leading to the clustering performances. This is the objective of the next sections.

\section{Main Results}
\label{sec:results}
\subsection{Eigenvalues}

In this section, we are interested in the localization of the eigenvalues of $ {\bf L}_\alpha $. Since $ {\bf L}_\alpha $ is asymptotically equivalent to a spiked random matrix, its eigenvalues are expected to be asymptotically the same as the eigenvalues of the full rank ``noise" matrix $n^{-\frac12}\mathbf{D}_{q}^{-\alpha}\mathbf{X}\mathbf{D}_{q}^{-\alpha} $ which are essentially concentrated in bulks, but possibly for finitely many of them which can isolate from those bulks when some eigenvalues  of the low rank matrix $ \mathbf{U}{\bm \Lambda}\mathbf{U}^{\sf T} $ are sufficiently large \cite{baik2006eigenvalues}.

To study the eigenvalues and eigenvectors of the spiked random matrix $ {\bf L}_\alpha $, we follow standard random matrix approaches \cite{benaych2012singular} and \\ \cite{hachem2013subspace}. We will first determine the support of the limiting eigenvalue distribution of $n^{-\frac12}\mathbf{D}_{q}^{-\alpha}\mathbf{X}\mathbf{D}_{q}^{-\alpha} $ where most of the eigenvalues of $ {\bf L}_\alpha $ concentrate (in bulks). Then, we will find the plausible isolated eigenvalues of $ {\bf L}_\alpha $ only induced by the low rank matrix $ \mathbf{U}{\bm \Lambda}\mathbf{U}^{\sf T} $ and which lie outside the aforementioned main support.

\begin{theorem}[Limiting spectrum]
\label{determinst2}
Let $ \pi^{\alpha}=\frac{1}{n}\sum_{i=1}^{n}\delta_{\lambda_{i}({\bf L}_\alpha)} $ be the empirical spectral distribution (e.s.d.) of $ {\bf L}_\alpha $. Then, as $ n \rightarrow \infty $, $ \pi^{\alpha} \rightarrow \bar{\pi}^{\alpha}$ almost surely where $ \bar{\pi}^{\alpha} $ is a probability measure with compact support $ \mathcal{S}^{\alpha} $ defined by its Stieltjes transform $E_0^\alpha(z)$ with $E_0^\alpha(z)\equiv \int \left(t-z\right)^{-1}\mathrm{d}\bar{\pi}^{\alpha}(t)=e_{00}^{\alpha}(z)$ ($ z \in \mathbb{C}^{+}\setminus \mathcal{S}^{\alpha} $) where, for $ a,b \in \mathbb{Z} $, 
\begin{equation}
\label{eq:StieltjesE}
e_{ab}^{\alpha}(z)=\int \frac{q^{a-2b\alpha}}{-z-E_1^\alpha(z)q^{1-2\alpha}+E_2^\alpha(z)q^{2-2\alpha}}\mu(dq)
\end{equation}
with $ \left(E_1^\alpha(z),E_2^\alpha(z)\right) $ for $ z \in \mathbb{C}^{+} $ the unique solution in $(\mathbb{C}^{+})^2$ of $ E_1^\alpha(z)=e_{11}^{\alpha}(z) $ and $  E_2^\alpha(z)=e_{21}^{\alpha}(z).$
\end{theorem}

\bigskip
For future use, we further define for $ z, \tilde{z} \in \mathbb{C} \setminus \mathcal{S}^{\alpha}$
\begin{equation}
\label{eDouble}
e_{ab;2}^{\alpha}(z,\tilde{z})=\int \frac{q^{a-2b\alpha}\mu(dq) }{(-z-E_1^\alpha(z)q^{1-2\alpha}+E_2^\alpha(z)q^{2-2\alpha})(-\tilde{z}-E_1^\alpha(\tilde{z})q^{1-2\alpha}+E_2^\alpha(\tilde{z})q^{2-2\alpha})}
\end{equation}
and
\begin{equation}
\label{eTriple}
e_{ab;3}^{\alpha}(z,\tilde{z})=\int \frac{q^{a-2b\alpha}\mu(dq) }{(-z-E_1^\alpha(z)q^{1-2\alpha}+E_2^\alpha(z)q^{2-2\alpha})^{2}(-\tilde{z}-E_1^\alpha(\tilde{z})q^{1-2\alpha}+E_2^\alpha(\tilde{z})q^{2-2\alpha})}.
\end{equation}
\begin{remark}
The support $ \mathcal{S}^{\alpha} $ is symmetric i.e., $ \bar{\pi}^{\alpha}([a,b])=\bar{\pi}^{\alpha}([-b,-a]) $. We have in particular $ S^{\alpha}_{-}=-S^{\alpha}_{+}=-S^{\alpha}$ where we denote $S^{\alpha}_{+} \triangleq \sup \mathcal{S}^{\alpha}$ and $S^{\alpha}_{-} \triangleq \inf \mathcal{S}^{\alpha}.$ See Figure~\ref{fig:hist_L_synthetic} for illustration.
\end{remark}
\begin{remark}[Semi-circle law]
For homogenous graphs where $ \forall i , q_i=q_0$, $ E_0^\alpha(z)=\frac{1}{-z-E_0^\alpha(z)} $ which is the Stieltjes transform of the well known semi-circle probability measure\footnote{Note that the limiting value $ E_0^\alpha(z) $ does not depend on $ \alpha $ in this case.} with support $ \mathcal{S}^{\alpha}=[-2,2] $ and density given by
$$ \bar{\pi}^{\alpha}(\mathrm{d}t)=\frac{1}{2\pi}\sqrt{(4-t^2)^{+}}\mathrm{d}t.$$
\end{remark}
\bigskip

Since $ \tilde{{\bf L}}_{\alpha} $ and $ \frac{1}{\sqrt{n}}\mathbf{D}_{q}^{-\alpha}\mathbf{X}\mathbf{D}_{q}^{-\alpha} $ only differ by a finite rank matrix from Theorem~\ref{approx}, the e.s.d.\@ of $ \frac{1}{\sqrt{n}}\mathbf{D}_{q}^{-\alpha}\mathbf{X}\mathbf{D}_{q}^{-\alpha} $ also converges weakly to $ \bar{\pi}^{\alpha} $ with support $ \mathcal{S}^{\alpha} $. In addition, we have
\begin{proposition}[No eigenvalues outside the support]
	\label{prop1}
	Following the statement of Theorem~\ref{determinst2}, let $ S^\alpha_{-} $ and $ S^\alpha_{+} $ be respectively the left and right edges of $ \mathcal{S}^\alpha $. Then, for any $ \epsilon>0 $, by letting $ \mathcal{S}^\alpha_{\epsilon}=[S^\alpha_{-}-\epsilon;S^\alpha_{+}+\epsilon] $ , for all large $ n $ almost surely, 
	$$ \Big\{\lambda_i\Big({\bf D}_q^{-\alpha}\frac{{\bf X}}{\sqrt{n}}{\bf D}_q^{-\alpha}\Big), 1\leq i \leq n\Big\} \cap (\mathbb{R} \setminus \mathcal{S}^\alpha_{\epsilon})=\emptyset. $$
	\end{proposition}
By Proposition~\ref{prop1}, we can define for $ \rho \notin \mathcal{S}^{\alpha} $ and for large $ n $ almost surely the resolvent $ {\bf Q}_\rho^{\alpha}=(n^{-\frac12}\mathbf{D}_{q}^{-\alpha}\mathbf{X}\mathbf{D}_{q}^{-\alpha}-\rho{\bf I}_n)^{-1} $. As per Theorem~\ref{approx}, since the hypothetically isolated eigenvalues of $ {\bf L}_\alpha $ are to be found outside the limiting support $ \mathcal{S}^{\alpha} $, we thus need to find those $ \rho $'s at a non-vanishing distance from $ \mathcal{S}^{\alpha} $ for which $0=\det({\bf L}_{\alpha}-\rho{\bf I}_n)=\det(\tilde{{\bf L}}_{\alpha}-\rho{\bf I}_n)+o(1)=\det(({\bf Q}_\rho^{\alpha})^{-1})\det({\bf I}_{K+1}+\mathbf{U}^{\sf T}{\bf Q}_\rho^{\alpha}\mathbf{U}{\bm \Lambda})+o(1)$. This leads to solving, for large $ n $, $ \det({\bf I}_{K+1}+\mathbf{U}^{\sf T}{\bf Q}_\rho^{\alpha}\mathbf{U}{\bm \Lambda})=0 $. We then show that $ \mathbf{U}^{\sf T}{\bf Q}_\rho^{\alpha}\mathbf{U}{\bm \Lambda} $ converges almost surely to a deterministic matrix and we have the following result
\begin{theorem}[Isolated Eigenvalues]\label{isoleigen}
	Let Assumption~\ref{as1} hold and, for $z\in \mathbb{C}\setminus \mathcal{S}^{\alpha}$ (given in Lemma~\ref{determinst2}), define the $ K \times K $ matrix 
	\begin{align*}
		\underline{\bf G}_{z}^{\alpha} &= \mathbf{I}_{K} + e_{21}^{\alpha}(z)\left(\mathcal{D}\left( \mathbf{c} \right) -\mathbf{cc}^{\sf T}\right)\mathbf{M}\left(\mathbf{I}_{K}-\mathbf{c}{\bf 1}_K^{T}\right)+\theta^{\alpha}(z)\mathbf{c}{\bf 1}_K^{T}
	\end{align*}
	with $ \theta^{\alpha}(z)=-1-\frac{ze_{10}^{\alpha}(z)}{m_{\mu}}+\frac{e_{21}^{\alpha}(z)}{m_{\mu}e_{10}^{\alpha}(z)}\left(v_{\mu}^{\alpha}+ze_{0,-1}^{\alpha}(z)\right)$  where the $e_{ij}^{\alpha}(z)$'s are defined in \eqref{eq:StieltjesE} and $ m_{\mu}=\int q\mu(\mathrm{d}q) $, $ v_{\mu}^{\alpha}=\int q^{2\alpha}\mu(\mathrm{d}q) $. Let $ \rho \in \mathbb{R}\setminus \mathcal{S}^{\alpha}$ be such that $\underline{\bf G}_{\rho}^{\alpha}$ has a zero eigenvalue of multiplicity $\eta$ (such a $ \rho $ may not exist). Then, there exist $\eta$ eigenvalues of ${\bf L}_{\alpha}$ converging to $\rho$, almost surely, as $ n \rightarrow \infty $.
\end{theorem}

\begin{remark}[Two types of isolated eigenvalues]\label{twocases}
From Theorem~\ref{isoleigen}, $ 1+\theta^{\alpha}(z) $ is an eigenvalue of $\underline{\bf G}_{z}^{\alpha}$ with associated left eigenvector $ {\bf 1}_{K} $ and right eigenvector $ {\bf c} $ since $ {\bf 1}_{K}^{T}\underline{\bf G}_z^{\alpha}=\left(1+\theta^{\alpha}(z)\right){\bf 1}_{K}^{T} $ and $ \underline{\bf G}_z^{\alpha}{\bf c}=\left(1+\theta^{\alpha}(z)\right){\bf c} $.

Letting $ \rho $ as in Theorem~\ref{isoleigen}, we can thus discriminate two cases
\begin{itemize}
\item $1+\theta^\alpha(\rho)=0$: isolated eigenvalues are found for those $ \rho \in \mathbb{R}\setminus \mathcal{S}^{\alpha} $ such that $1+\theta^\alpha(\rho)=0$. We shall denote by $ \tilde{\rho} $ such eigenvalues when they exist.
\item $1+\theta^{\alpha}(\rho) \neq 0$: the left and right eigenvectors associated to the zero eigenvalues of $ \underline{\bf G}_\rho^{\alpha} $ are respectively orthogonal to the right and left eigenvectors associated to the non-zero eigenvalues. So, by letting $ {\bf V}_{l} $, $ {\bf V}_{r} $ be matrices containing in columns the respectively left and right eigenvectors of $\underline{\bf G}_\rho^{\alpha} $ associated with the zero eigenvalues, we have $ {\bf V}_{l}^{\sf T}{\bf c}={\bf 0} $ and $ {\bf 1}_{K}^{\sf T}{\bf V}_{r}={\bf 0} $ since $ 1+\theta^{\alpha}(\rho) \neq 0 $. It is thus immediate that $({\bf V}_{l},{\bf V}_{r}) $ is also a pair of eigenvectors (with multiplicity) of $\mathbf{I}_{K} + e_{21}^{\alpha}(\rho)\left(\mathcal{D}\left( \mathbf{c} \right) -\mathbf{cc}^{\sf T}\right)\mathbf{M}\left(\mathbf{I}_{K}-\mathbf{c}{\bf 1}_K^{T}\right)$ associated to the zero eigenvalues.
\end{itemize}
\end{remark}

As we show in Appendix~\ref{sec:nonInfoEig}, for $1+\theta^\alpha(\tilde{\rho})=0$, the eigenvectors associated to the aforementioned isolated eigenvalues $ \tilde{\rho} $ will not contain information about the classes. This case is thus of no interest for clustering. It is nevertheless important from a practical viewpoint to note that, even in the absence of communities, spurious isolated eigenvalues may be found that may deceive the experimenter in suggesting the presence of node clusters. From now on, we will only consider the isolated eigenvalues $ \rho $ for which $ 1+\theta^{\alpha}(\rho) \neq 0 $. 

Since it is more convenient to work with symmetric matrices (having identical left and right eigenvectors), the following remark will be useful in what follows. 
\begin{remark}[Informative eigenvectors]\label{Mbar}
The next three statements are equivalent.
For $ \rho $ a limiting isolated eigenvalue of $ {\bf L}_{\alpha} $ such that $ 1+\theta^{\alpha}(\rho) \neq 0 $,
\begin{itemize}
\item $({\bf V}_{l},{\bf V}_{r}) $ is a set of left/right eigenvectors of $ \underline{\bf G}_\rho^{\alpha} $ associated to zero eigenvalues.
\item $({\bf V}_{l},{\bf V}_{r}) $ is a set of left/right eigenvectors of $ {\bf G}_{\rho}^{\alpha}=\left(\mathcal{D}\left( \mathbf{c} \right) -\mathbf{cc}^{\sf T}\right)\mathbf{M}\left(\mathbf{I}_{K}-\mathbf{c}{\bf 1}_K^{T}\right) $ associated to eigenvalue $ -\frac{1}{e_{21}^{\alpha}(\rho)} $ (with multiplicity).
\item ${\bf V}=\mathcal{D}({\bf c})^{\frac12}{\bf V}_{l}=\mathcal{D}({\bf c})^{-\frac12}{\bf V}_{r}$ is a set of eigenvectors of the symmetric matrix \\ $ \mathcal{D}({\bf c})^{\frac12}\left({\bf I}_{K} -\mathbf{1}_K{\bf c}^{\sf T}\right)\mathbf{M}\left(\mathbf{I}_{K}-\mathbf{c}{\bf 1}_K^{T}\right)\mathcal{D}({\bf c})^{\frac12} $ associated to the eigenvalue $ -\frac{1}{e_{21}^{\alpha}(\rho)} $ (with multiplicity).
\end{itemize}
\end{remark}
\bigskip
As $ {\bf L}_{\alpha} $ is asymptotically equivalent (through Theorem~\ref{approx}) to a spiked random matrix, the eigenvectors of $ {\bf L}_\alpha $ associated to isolated eigenvalues are expected to correlate (to some extent) to the eigenvectors of $ {\bf U}{\bm \Lambda}{\bf U}^{\sf T} $ (defined in Theorem~\ref{approx}) and thus to $ {\bf D}_q^{1-\alpha}{\bf J} $. Clustering based on the eigenvectors of $ {\bf L}_{\alpha} $ should then be possible when they are associated to such isolated eigenvalues. Based on this fact and with the help of Theorems~\ref{determinst2} and \ref{isoleigen}, we can exactly characterize the phase transition threshold beyond which community detection is indeed possible. 

\begin{corollary}[Phase transition]
\label{phasetransition}
Let Assumption \ref{as1} hold and let $ \lambda(\bar{{\bf M}}) $ be a non zero eigenvalue with multiplicity $ \eta $ of $\bar{{\bf M}}\equiv \left(\mathcal{D}({\bf c})-\mathbf{cc}^{\sf T}\right)\mathbf{M}.$ Then, for $ \alpha \in \mathcal{A}$, there exist corresponding isolated eigenvalues $ \lambda_i({\bf L}_{\alpha}),\dots,\lambda_{i+\eta}({\bf L}_{\alpha}) \in \mathbb{R} \backslash \mathcal{S}^\alpha$ of $ {\bf L}_{\alpha} $ all converging to a limiting eigenvalue $ \rho \in \mathbb{R}\setminus \mathcal{S}^{\alpha}$, as $ n \to \infty $, almost surely, if and only if \footnote{The limit $ \lim_{x\downarrow S^{\alpha}}E_2^{\alpha}(x) $ is well defined in $(-\infty,0]$ as $x \mapsto E_2^\alpha(x$) is a continuous growing negative function on the right side of $\mathcal{S}^\alpha$.}  $$ \left|\lambda(\bar{{\bf M}})\right|> \tau^{\alpha}\equiv -\lim_{x\downarrow S^\alpha}\frac{1}{E_2^{\alpha}(x)},$$ \\
	with $E_2^\alpha(x)$ defined in Theorem~\ref{determinst2}. In this case, $ \rho $ is defined by $$E_2^{\alpha}(\rho)=-\frac{1}{\lambda(\bar{{\bf M}})}.$$
\end{corollary}
\begin{remark}[Rank of $ \bar{{\bf M}}$ and maximum number of isolated eigenvalues] \label{rankG}
From Corollary~\ref{phasetransition}, there is a one-to-one mapping between isolated eigenvalues $ \rho $ of $ {\bf L}_\alpha $ and non zero eigenvalues of $ \bar{{\bf M}}=\left(\mathcal{D}({\bf c})-\mathbf{cc}^{\sf T}\right)\mathbf{M}$. As $ {\bf 1}_K^{\sf T}\bar{{\bf M}}=0 $, $ \bar{{\bf M}} $ has a maximum of $ K-1 $ non zero eigenvalues which means that at most $ K-1 $ eigenvalues of $ {\bf L}_\alpha $ can be found at macroscopic distance from $ \mathcal{S}^\alpha $ (excluding the case $1+\theta^\alpha(\rho)=0 $). Thus, \emph{at most $ K-1 $ eigenvectors of $ {\bf L}_\alpha $ can be used in the first step of the spectral algorithm described in the introduction}.
\end{remark}
\subsection{Application 1: Optimal $ \alpha $}
\label{subsec:optimal}
In this section, we find the $ \alpha $ for which the community detectability threshold is maximally achieved. This, in turn, shall allow to extract some non-trivial information about the classes from the extreme eigenvectors.

From Corollary~\ref{phasetransition}, since $ \bar{{\bf M}} $ does not depend on $ \alpha $, the smaller $ \tau^{\alpha} $ the more likely the detectability condition $ \left|\lambda(\bar{{\bf M}})\right|> \tau^{\alpha} $ is met. We then seek $ \alpha $ for which $ \tau^\alpha $ is minimal. We may thus define $$ \alpha_{\rm opt} \equiv \argmin_{\alpha \in \mathcal{A}} \left\{ \tau^{\alpha} \right\}. $$ 
Retrieving $ \alpha_{\rm opt} $ has a tremendous practical advantage as it optimizes the detection of barely detectable communities (and, as shall be seen through simulations, greatly improves the performance). The estimation of $ \alpha_{\rm opt} $ however requires the knowledge of $ E_{2}^{\alpha}(x) $ for each $ \alpha \in \mathcal{A}$. The estimation of $ E_{2}^{\alpha}(x) $ can be done numerically by solving the fixed point equation defined in Theorem~\ref{determinst2} provided $ \mu $ is known. Thanks to Equation~(\ref{eq:estimator}), $ \mu $ can be consistently estimated from the empirical degrees $d_i$'s. We thus have all the ingredients to estimate $ \alpha_{\rm opt} $.
\begin{lemma} \label{mu_estimation}
	Define $\hat{\mu}\equiv\frac1n\sum_{i=1}^n \delta_{\hat{q}_i}$ with $\hat{q}_i=\frac{d_i}{\sqrt{{\bf d}^{\sf T} {\bm 1}_n}}$ and $\hat{\mathcal S}^\alpha$, $\hat{E}_i^\alpha(z)$, $i\in\{1,2\}$, as in Theorem~\ref{determinst2} but for $\mu$ replaced by $\hat{\mu}$. Then, as $n\to\infty$,
    $$ \hat{\alpha}_{\rm opt} \to \alpha_{\rm opt}$$
    almost surely, where $\hat{\alpha}_{\rm opt}\equiv \argmin_{\alpha\in\mathcal{A}} \{\hat{\tau}^\alpha\}$ with
    $$\hat{\tau}_\alpha \equiv -\frac1{\lim_{x\downarrow \hat{S}^\alpha} \hat{E}_2^\alpha(x)}.$$
\end{lemma}
\begin{remark}[Estimation of $S^{\alpha}$]
To estimate numerically the right edge $S^{\alpha}$ of the support $\mathcal{S}^{\alpha}=[-S^{\alpha},S^{\alpha}]$, we use the fact that $ x \mapsto E_2^\alpha(x) $ is not defined  in $ \mathcal{S}^{\alpha} $. To this end, we evaluate $ S^{\alpha} $ by an iterative dichotomic search in intervals of the type $ [l,r] $ for which $ E_2^\alpha(l) $ is undefined (and thus the system of fixed point equations defining $ E_2^\alpha(x)$ in Theorem~\ref{determinst2} does not converge for $x=l$) and $E_2^\alpha(r)$ is defined (the system of fixed point equations defining $ E_2^\alpha(x)$ in Theorem~\ref{determinst2} converges for $x=r$), starting from e.g., $ l=0 $ and $ r $ quite large. Since for the points $ x $ where $E_2^\alpha(x)$ is not defined, the fixed point algorithm solving the equations defining $E_2^\alpha(x)$ run indefinitely, one must fix a given number of iterations after which the algorithm is stopped and $ x $ is then considered to be in $ \mathcal{S}^{\alpha}.$
\end{remark}

The aforementionned importance of choosing $ \alpha=\hat{\alpha}_{\rm opt} $ along with the need to pre-multiply the dominant eigenvectors of $ {\bf L}_{\alpha} $ by $ {\bf D}^{\alpha-1} $ before classification, as discussed after exposing Theorem~\ref{approx}, naturally bring us to a \emph{novel heterogeneous-graph improved community detection method}. Algorithm~\ref{alg:algorithm} below summarizes the main steps of our improved spectral algorithm. Note that the same algorithm can be applied for any $ \alpha $ by replacing $ \hat{\alpha}_{\rm opt} $ by the corresponding value of $ \alpha $ and skipping step~\ref{one}.
\begin{algorithm}[h!]
\label{alg:algorithm}
\caption{Improved spectral algorithm}
\begin{algorithmic}[1]
\STATE Evaluate $ \alpha=\hat{\alpha}_{\rm opt}=\argmin_{\alpha \in (0,1)} \lim_{x\downarrow \hat{S}^\alpha}\hat{E}_2^{\alpha}(x)$ as per Lemma~\ref{mu_estimation}. \label{one} 
\STATE Retrieve the $ \ell $ eigenvectors corresponding to the $ \ell $ largest eigenvalues (which are found away from the bulks in the spectrum of $ {\bf L}_\alpha $) of $ {\bf L}_\alpha=(2m)^\alpha\frac{1}{\sqrt{n}}\mathbf{D}^{-\alpha} \left[ {\bf A} - \frac{{\bf d}{\bf d}^{\sf T}}{2m} \right] \mathbf{D}^{-\alpha} $. Denote $ {\bf u}_1^\alpha,\dots,{\bf u}_\ell^\alpha $ those eigenvectors.\footnotemark \label{two}
\STATE Letting $ {\bf v}_i^\alpha={\bf D}^{\alpha-1}{\bf u}_i^\alpha $ and $ \bar{{\bf v}}_i^\alpha=\frac{{\bf v}_i^\alpha}{\left\|{\bf v}_i^\alpha\right\|} $, stack the vectors $ \bar{{\bf v}}_i^\alpha $'s columnwise in a matrix $ {\bf W}=\left[\bar{{\bf v}}_1^\alpha,\dots,\bar{{\bf v}}_\ell^\alpha\right] \in \mathbb{R}^{n\times \ell}$. \label{three}
\STATE Let $ {\bf r}_1,\dots,{\bf r}_n \in \mathbb{R}^{\ell}$ be the rows of ${\bf W}$. Cluster ${\bf r}_i \in \mathbb{R}^{\ell}$, $ 1\leq i \leq n $ in one of the $ K $ groups using any low-dimensional classification algorithm (e.g., k-means or EM). The label assigned to $ {\bf r}_i $ then corresponds to the label of node $ i $. \label{four}
\end{algorithmic}
\end{algorithm}
\footnotetext{$ \ell $ corresponds in practice to the number of eigenvalues of $ {\bf L}_{\alpha} $ which isolate from the bulks. From Remark~\ref{rankG}, we can have up to $ K-1 $ of them.} 

In the following, we will restrict ourselves to $ \alpha \in \mathcal{A}=[0,1] $ for the numerical simulations. To illustrate the importance of the choice of $ \alpha_{\rm opt} $, Figure~\ref{fig:phase_transition_2masses} presents the theoretical (asymptotic) ratio between the largest eigenvalue of $ {\bf L}_\alpha $ and the right edge of the limiting support $ \mathcal{S}^{\alpha} $ ($S^{\alpha}_{+}$) with respect to the amplitude of the eigenvalues of $\bar{{\bf M}}$. Intuitively, the larger this ratio the better the clustering performance as the eigenvector associated to this largest eigenvalue contains less noise. Although $ \alpha_{\rm opt} $ only ensures in theory to have the best isolation of the eigenvalues only in ``worst cases scenarios"(i.e., when $\lambda(\bar{{\bf M}})$ is only slighty larger than $ \tau^{\alpha_{\rm opt}} $), Figure~\ref{fig:phase_transition_2masses} shows that taking $ \alpha=\alpha_{\rm opt} $ provides the largest gap $ \frac{\lambda_i({\bf L}_\alpha)}{S^\alpha_+} $ for all values of $\lambda(\bar{{\bf M}})$. This suggests (again, without any theoretical support) better performances with $ \alpha=\alpha_{\rm opt} $ in all cases (for any value of $ {\bf M} $).
\begin{figure}
  \centering
  \begin{tikzpicture}[font=\footnotesize]
    \renewcommand{\axisdefaulttryminticks}{4} 
    \tikzstyle{every major grid}+=[style=densely dashed]       
    \tikzstyle{every axis y label}+=[yshift=-10pt] 
    \tikzstyle{every axis x label}+=[yshift=5pt]
    \tikzstyle{every axis legend}+=[cells={anchor=west},fill=white,
        at={(0.98,0.98)}, anchor=north east, font=\scriptsize ]
    \begin{axis}[
      xmin=2,
      ymin=0.98,
      xmax=16,
      ymax=1.1,
      grid=major,
      ymajorgrids=false,
      scaled ticks=true,
      xlabel={Eigenvalue $\ell$ ($\ell=-1/e_2^\alpha(\lambda)$ beyond phase transition)},
      ylabel={Normalized spike $\frac{\lambda}{S^\alpha_+}$ }
      ]
      \addplot[blue,line width=0.5pt] plot coordinates{
          (0,1)(3.688746,1.000000)(3.698122,1.000204)(3.707287,1.000408)(3.716256,1.000612)(3.725042,1.000816)(3.733658,1.001020)(3.742114,1.001224)(3.750419,1.001429)(3.758582,1.001633)(3.766612,1.001837)(3.774514,1.002041)(3.782296,1.002245)(3.789963,1.002449)(3.797521,1.002653)(3.804976,1.002857)(3.812331,1.003061)(3.819592,1.003265)(3.826762,1.003469)(3.833845,1.003673)(3.840845,1.003878)(3.847765,1.004082)(3.854607,1.004286)(3.861376,1.004490)(3.868073,1.004694)(3.874701,1.004898)(3.881262,1.005102)(3.887759,1.005306)(3.894195,1.005510)(3.900570,1.005714)(3.906886,1.005918)(3.913147,1.006122)(3.919353,1.006327)(3.925506,1.006531)(3.931607,1.006735)(3.937659,1.006939)(3.943662,1.007143)(3.949618,1.007347)(3.955528,1.007551)(3.961393,1.007755)(3.967215,1.007959)(3.972994,1.008163)(3.978733,1.008367)(3.984430,1.008571)(3.990089,1.008776)(3.995709,1.008980)(4.001292,1.009184)(4.006839,1.009388)(4.012350,1.009592)(4.017826,1.009796)(4.023268,1.010000)(4.023268,1.010000)(4.070827,1.011837)(4 .116139,1.013673)(4.159552,1.015510)(4.201332,1.017347)(4.241688,1.019184)(4.280788,1.021020)(4.318772,1.022857)(4.355752,1.024694)(4.391826,1.026531)(4.427076,1.028367)(4.461571,1.030204)(4.495372,1.032041)(4.528534,1.033878)(4.561102,1.035714)(4.593119,1.037551)(4.624620,1.039388)(4.655639,1.041224)(4.686206,1.043061)(4.716347,1.044898)(4.746087,1.046735)(4.775447,1.048571)(4.804448,1.050408)(4.833108,1.052245)(4.861445,1.054082)(4.889473,1.055918)(4.917207,1.057755)(4.944661,1.059592)(4.971846,1.061429)(4.998775,1.063265)(5.025458,1.065102)(5.051904,1.066939)(5.078125,1.068776)(5.104127,1.070612)(5.129920,1.072449)(5.155510,1.074286)(5.180906,1.076122)(5.206115,1.077959)(5.231142,1.079796)(5.255993,1.081633)(5.280676,1.083469)(5.305194,1.085306)(5.329554,1.087143)(5.353760,1.088980)(5.377816,1.090816)(5.401728,1.092653)(5.425500,1.094490)(5.449135,1.096327)(5.472638,1.098163)(5.496012,1.100000)          
      };
      \addplot[blue,line width=0.5pt] plot coordinates{
          (0,1)(3.820149,1.000000)(3.830396,1.000204)(3.840419,1.000408)(3.850235,1.000612)(3.859857,1.000816)(3.869297,1.001020)(3.878566,1.001224)(3.887674,1.001429)(3.896631,1.001633)(3.905443,1.001837)(3.914119,1.002041)(3.922666,1.002245)(3.931089,1.002449)(3.939395,1.002653)(3.947589,1.002857)(3.955676,1.003061)(3.963661,1.003265)(3.971547,1.003469)(3.979339,1.003673)(3.987040,1.003878)(3.994655,1.004082)(4.002186,1.004286)(4.009636,1.004490)(4.017008,1.004694)(4.024306,1.004898)(4.031531,1.005102)(4.038686,1.005306)(4.045773,1.005510)(4.052795,1.005714)(4.059753,1.005918)(4.066649,1.006122)(4.073486,1.006327)(4.080265,1.006531)(4.086988,1.006735)(4.093656,1.006939)(4.100271,1.007143)(4.106834,1.007347)(4.113347,1.007551)(4.119811,1.007755)(4.126227,1.007959)(4.132597,1.008163)(4.138922,1.008367)(4.145202,1.008571)(4.151439,1.008776)(4.157634,1.008980)(4.163788,1.009184)(4.169902,1.009388)(4.175977,1.009592)(4.182013,1.009796)(4.188012,1.010000)(4.188012,1.010000)(4.240440,1.011837)(4 .290391,1.013673)(4.338242,1.015510)(4.384286,1.017347)(4.428753,1.019184)(4.471828,1.021020)(4.513662,1.022857)(4.554381,1.024694)(4.594093,1.026531)(4.632886,1.028367)(4.670840,1.030204)(4.708021,1.032041)(4.744488,1.033878)(4.780294,1.035714)(4.815483,1.037551)(4.850098,1.039388)(4.884174,1.041224)(4.917745,1.043061)(4.950840,1.044898)(4.983486,1.046735)(5.015707,1.048571)(5.047527,1.050408)(5.078966,1.052245)(5.110041,1.054082)(5.140772,1.055918)(5.171173,1.057755)(5.201260,1.059592)(5.231047,1.061429)(5.260546,1.063265)(5.289769,1.065102)(5.318727,1.066939)(5.347431,1.068776)(5.375891,1.070612)(5.404116,1.072449)(5.432114,1.074286)(5.459894,1.076122)(5.487463,1.077959)(5.514828,1.079796)(5.541997,1.081633)(5.568975,1.083469)(5.595770,1.085306)(5.622386,1.087143)(5.648829,1.088980)(5.675105,1.090816)(5.701218,1.092653)(5.727173,1.094490)(5.752976,1.096327)(5.778629,1.098163)(5.804137,1.100000)          
      };
      \addplot[blue,line width=0.5pt] plot coordinates{          
          (0,1)(5.569717,1.000000)(5.580348,1.000204)(5.590796,1.000408)(5.601073,1.000612)(5.611187,1.000816)(5.621146,1.001020)(5.630959,1.001224)(5.640632,1.001429)(5.650173,1.001633)(5.659586,1.001837)(5.668878,1.002041)(5.678053,1.002245)(5.687118,1.002449)(5.696075,1.002653)(5.704930,1.002857)(5.713686,1.003061)(5.722347,1.003265)(5.730916,1.003469)(5.739397,1.003673)(5.747794,1.003878)(5.756107,1.004082)(5.764342,1.004286)(5.772499,1.004490)(5.780582,1.004694)(5.788593,1.004898)(5.796533,1.005102)(5.804406,1.005306)(5.812214,1.005510)(5.819957,1.005714)(5.827638,1.005918)(5.835259,1.006122)(5.842821,1.006327)(5.850326,1.006531)(5.857775,1.006735)(5.865170,1.006939)(5.872512,1.007143)(5.879803,1.007347)(5.887043,1.007551)(5.894235,1.007755)(5.901378,1.007959)(5.908475,1.008163)(5.915527,1.008367)(5.922533,1.008571)(5.929497,1.008776)(5.936417,1.008980)(5.943296,1.009184)(5.950134,1.009388)(5.956932,1.009592)(5.963691,1.009796)(5.970412,1.010000)(5.970412,1.010000)(6.029297,1.011837)(6 .085620,1.013673)(6.139752,1.015510)(6.191984,1.017347)(6.242546,1.019184)(6.291628,1.021020)(6.339384,1.022857)(6.385946,1.024694)(6.431424,1.026531)(6.475912,1.028367)(6.519492,1.030204)(6.562235,1.032041)(6.604205,1.033878)(6.645455,1.035714)(6.686035,1.037551)(6.725989,1.039388)(6.765355,1.041224)(6.804170,1.043061)(6.842465,1.044898)(6.880270,1.046735)(6.917610,1.048571)(6.954510,1.050408)(6.990993,1.052245)(7.027079,1.054082)(7.062786,1.055918)(7.098132,1.057755)(7.133133,1.059592)(7.167805,1.061429)(7.202161,1.063265)(7.236215,1.065102)(7.269978,1.066939)(7.303462,1.068776)(7.336678,1.070612)(7.369635,1.072449)(7.402344,1.074286)(7.434813,1.076122)(7.467051,1.077959)(7.499065,1.079796)(7.530863,1.081633)(7.562452,1.083469)(7.593838,1.085306)(7.625030,1.087143)(7.656031,1.088980)(7.686848,1.090816)(7.717487,1.092653)(7.747953,1.094490)(7.778251,1.096327)(7.808385,1.098163)(7.838361,1.100000)          
      };
      \addplot[blue,line width=0.5pt] plot coordinates{
          (0,1)(8.608852,1.000000)(8.623277,1.000204)(8.637443,1.000408)(8.651363,1.000612)(8.665052,1.000816)(8.678522,1.001020)(8.691784,1.001224)(8.704850,1.001429)(8.717728,1.001633)(8.730427,1.001837)(8.742956,1.002041)(8.755322,1.002245)(8.767532,1.002449)(8.779592,1.002653)(8.791510,1.002857)(8.803290,1.003061)(8.814937,1.003265)(8.826458,1.003469)(8.837856,1.003673)(8.849135,1.003878)(8.860301,1.004082)(8.871357,1.004286)(8.882306,1.004490)(8.893152,1.004694)(8.903899,1.004898)(8.914549,1.005102)(8.925106,1.005306)(8.935573,1.005510)(8.945951,1.005714)(8.956244,1.005918)(8.966454,1.006122)(8.976583,1.006327)(8.986634,1.006531)(8.996608,1.006735)(9.006509,1.006939)(9.016337,1.007143)(9.026094,1.007347)(9.035783,1.007551)(9.045405,1.007755)(9.054961,1.007959)(9.064454,1.008163)(9.073885,1.008367)(9.083254,1.008571)(9.092564,1.008776)(9.101816,1.008980)(9.111012,1.009184)(9.120152,1.009388)(9.129237,1.009592)(9.138269,1.009796)(9.147250,1.010000)(9.147250,1.010000)(9.225896,1.011837)(9 .301071,1.013673)(9.373286,1.015510)(9.442938,1.017347)(9.510346,1.019184)(9.575766,1.021020)(9.639409,1.022857)(9.701454,1.024694)(9.762050,1.026531)(9.821326,1.028367)(9.879392,1.030204)(9.936343,1.032041)(9.992266,1.033878)(10.047232,1.035714)(10.101310,1.037551)(10.154557,1.039388)(10.207027,1.041224)(10.258766,1.043061)(10.309819,1.044898)(10.360223,1.046735)(10.410014,1.048571)(10.459225,1.050408)(10.507885,1.052245)(10.556023,1.054082)(10.603662,1.055918)(10.650827,1.057755)(10.697539,1.059592)(10.743818,1.061429)(10.789683,1.063265)(10.835151,1.065102)(10.880239,1.066939)(10.924961,1.068776)(10.969332,1.070612)(11.013365,1.072449)(11.057073,1.074286)(11.100467,1.076122)(11.143560,1.077959)(11.186360,1.079796)(11.228880,1.081633)(11.271126,1.083469)(11.313110,1.085306)(11.354839,1.087143)(11.396321,1.088980)(11.437563,1.090816)(11.478574,1.092653)(11.519360,1.094490)(11.559928,1.096327)(11.600284,1.098163)(11.640433,1.100000)          
      };
      \addplot[blue,line width=0.5pt] plot coordinates{
          (0,1)(10.627925,1.000000)(10.647100,1.000204)(10.665931,1.000408)(10.684437,1.000612)(10.702636,1.000816)(10.720545,1.001020)(10.738179,1.001224)(10.755551,1.001429)(10.772673,1.001633)(10.789558,1.001837)(10.806217,1.002041)(10.822658,1.002245)(10.838892,1.002449)(10.854927,1.002653)(10.870772,1.002857)(10.886433,1.003061)(10.901918,1.003265)(10.917233,1.003469)(10.932384,1.003673)(10.947378,1.003878)(10.962220,1.004082)(10.976915,1.004286)(10.991468,1.004490)(11.005884,1.004694)(11.020166,1.004898)(11.034320,1.005102)(11.048348,1.005306)(11.062256,1.005510)(11.076046,1.005714)(11.089721,1.005918)(11.103286,1.006122)(11.116743,1.006327)(11.130095,1.006531)(11.143345,1.006735)(11.156495,1.006939)(11.169549,1.007143)(11.182508,1.007347)(11.195375,1.007551)(11.208153,1.007755)(11.220843,1.007959)(11.233448,1.008163)(11.245969,1.008367)(11.258408,1.008571)(11.270768,1.008776)(11.283050,1.008980)(11.295256,1.009184)(11.307388,1.009388)(11.319447,1.009592)(11.331434,1.009796)(11.343352 ,1.010000)(11.343352,1.010000)(11.447691,1.011837)(11.547370,1.013673)(11.643071,1.015510)(11.735328,1.017347)(11.824564,1.019184)(11.911125,1.021020)(11.995294,1.022857)(12.077309,1.024694)(12.157372,1.026531)(12.235653,1.028367)(12.312303,1.030204)(12.387450,1.032041)(12.461207,1.033878)(12.533674,1.035714)(12.604941,1.037551)(12.675086,1.039388)(12.744180,1.041224)(12.812287,1.043061)(12.879465,1.044898)(12.945766,1.046735)(13.011239,1.048571)(13.075927,1.050408)(13.139870,1.052245)(13.203105,1.054082)(13.265667,1.055918)(13.327585,1.057755)(13.388891,1.059592)(13.449610,1.061429)(13.509768,1.063265)(13.569389,1.065102)(13.628494,1.066939)(13.687104,1.068776)(13.745238,1.070612)(13.802914,1.072449)(13.860149,1.074286)(13.916960,1.076122)(13.973360,1.077959)(14.029366,1.079796)(14.084989,1.081633)(14.140243,1.083469)(14.195140,1.085306)(14.249690,1.087143)(14.303907,1.088980)(14.357798,1.090816)(14.411375,1.092653)(14.464647,1.094490)(14.517622,1.096327)(14.570309,1.098163)(14.62271 7,1.100000)          
      };
      \addplot[red,line width=0.5pt] plot coordinates{
          (0,1)(3.687812,1.000000)(3.697442,1.000204)(3.706851,1.000408)(3.716057,1.000612)(3.725073,1.000816)(3.733912,1.001020)(3.742584,1.001224)(3.751100,1.001429)(3.759469,1.001633)(3.767698,1.001837)(3.775797,1.002041)(3.783770,1.002245)(3.791625,1.002449)(3.799368,1.002653)(3.807003,1.002857)(3.814536,1.003061)(3.821971,1.003265)(3.829312,1.003469)(3.836563,1.003673)(3.843728,1.003878)(3.850811,1.004082)(3.857814,1.004286)(3.864741,1.004490)(3.871594,1.004694)(3.878376,1.004898)(3.885089,1.005102)(3.891736,1.005306)(3.898319,1.005510)(3.904841,1.005714)(3.911302,1.005918)(3.917705,1.006122)(3.924053,1.006327)(3.930345,1.006531)(3.936585,1.006735)(3.942773,1.006939)(3.948911,1.007143)(3.955001,1.007347)(3.961044,1.007551)(3.967041,1.007755)(3.972993,1.007959)(3.978901,1.008163)(3.984767,1.008367)(3.990592,1.008571)(3.996376,1.008776)(4.002121,1.008980)(4.007827,1.009184)(4.013496,1.009388)(4.019128,1.009592)(4.024725,1.009796)(4.030286,1.010000)(4.030286,1.010000)(4.078884,1.011837)(4 .125176,1.013673)(4.169518,1.015510)(4.212185,1.017347)(4.253392,1.019184)(4.293313,1.021020)(4.332089,1.022857)(4.369837,1.024694)(4.406656,1.026531)(4.442630,1.028367)(4.477832,1.030204)(4.512323,1.032041)(4.546159,1.033878)(4.579387,1.035714)(4.612050,1.037551)(4.644186,1.039388)(4.675827,1.041224)(4.707006,1.043061)(4.737748,1.044898)(4.768080,1.046735)(4.798022,1.048571)(4.827597,1.050408)(4.856823,1.052245)(4.885718,1.054082)(4.914296,1.055918)(4.942574,1.057755)(4.970564,1.059592)(4.998279,1.061429)(5.025732,1.063265)(5.052932,1.065102)(5.079891,1.066939)(5.106618,1.068776)(5.133122,1.070612)(5.159410,1.072449)(5.185492,1.074286)(5.211375,1.076122)(5.237065,1.077959)(5.262570,1.079796)(5.287895,1.081633)(5.313046,1.083469)(5.338029,1.085306)(5.362850,1.087143)(5.387513,1.088980)(5.412024,1.090816)(5.436386,1.092653)(5.460604,1.094490)(5.484683,1.096327)(5.508625,1.098163)(5.532436,1.100000)          
      };
      \addplot[blue,mark=o,only marks,line width=0.5pt] plot coordinates{(3.688746,1)(3.820149,1)(5.569717,1)(8.608852,1)(10.627925,1)};
      \addplot[red,mark=o,only marks,line width=0.5pt] plot coordinates{(3.687812,1)};
      \draw[->,thick] (axis cs:3.906886,1.022) -- (axis cs:4.201332,1.017347) node [below,pos=0,font=\footnotesize] {$0$};
      \draw[->,thick] (axis cs:4.850098,1.015) -- (axis cs:4.182013,1.009796) node [below,pos=0,font=\footnotesize] {$\frac14$};
      \draw[->,thick] (axis cs:7.236215,1.042) -- (axis cs:6.686035,1.037551) node [below,pos=0,font=\footnotesize] {$\frac12$};
      \draw[->,thick] (axis cs:11.600284,1.08) -- (axis cs:11.100467,1.076122) node [below,pos=0,font=\footnotesize] {$\frac34$};
      \draw[->,thick] (axis cs:13.802914,1.06) -- (axis cs:13.139870,1.052245) node [below,pos=0,font=\footnotesize] {$1$};
      \draw[->,thick] (axis cs:4.030286,1.05) -- (axis cs:4.768080,1.046735) node [above,pos=0,font=\footnotesize] {$\alpha_{\rm opt}$};
    \end{axis}
  \end{tikzpicture}
  \caption{Ratio between the largest eigenvalue $\lambda$ of ${\bf L}_\alpha$ and the right edge of the support $ \mathcal{S}^\alpha $, as a function of the largest eigenvalue $\ell$ of $\bar{\bf M}$, ${\bf M}=\Delta {\bf I}_3$, $c_i=\frac13$, for $\Delta\in[10,150]$, $\mu=\frac34\delta_{q_{(1)}}+\frac14\delta_{q_{(2)}}$ with $q_{(1)}=0.1$ and $q_{(2)}=0.5$, for $\alpha\in\{0,\frac14,\frac12,\frac34,1,\alpha_{\rm opt}\}$ (indicated on the curves of the graph). Here, $\alpha_{\rm opt}=0.07$. Circles indicate phase transition.}
  \label{fig:phase_transition_2masses}
\end{figure}
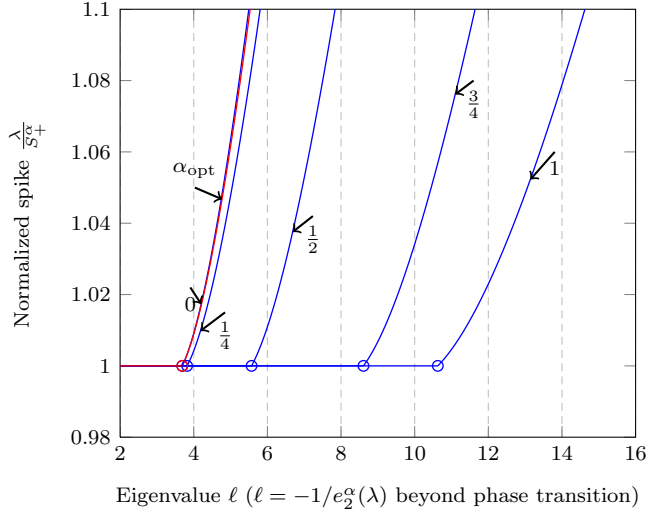

In the sequel, to compare the different algorithms, we will use the performance evaluation measure known as the \emph{overlap to ground truth communities}, defined in \cite{krzakala2013spectral} as 
$$ {\rm Overlap} \equiv \frac{\frac1n\sum_{i=1}^{n}\delta(g_{i}\hat{g}_{i})-\frac1K}{1-\frac1K}, $$ where $ g_{i} $ and $ \hat{g}_{i} $ are the true and estimated labels of node $ i $, respectively. Note that this definition implicitly suggests that all communities are of equal proportions and is therefore not fully compatible with our more general present setting which allows for unbalanced classes. Figure~\ref{fig:overlap_2masses} subsequently shows the overlap performance under the setting of Figure~\ref{fig:phase_transition_2masses}. It is worth mentioning that the empirically observed phase transitions closely match the theoretical ones (drawn in circles and the same as in Figure~\ref{fig:phase_transition_2masses}).
\begin{figure}[h!]
  \centering
  \begin{tikzpicture}[font=\footnotesize]
    \renewcommand{\axisdefaulttryminticks}{4} 
    \tikzstyle{every major grid}+=[style=densely dashed]       
    \tikzstyle{every axis y label}+=[yshift=-10pt] 
    \tikzstyle{every axis x label}+=[yshift=5pt]
    \tikzstyle{every axis legend}+=[cells={anchor=west},fill=white,
        at={(0.98,0.02)}, anchor=south east, font=\scriptsize ]
    \begin{axis}[
      xmin=5,
      ymin=0,
      xmax=50,
      ymax=1,
      grid=major,
      ymajorgrids=false,
      scaled ticks=true,
      xlabel={$\Delta$},
      ylabel={Overlap},
      legend style={at={(axis cs:51,0.23)},anchor=west}
      ]
      {
      \addplot[blue,line width=0.5pt] plot coordinates{
          (0.000000,0.018010)(1.250000,0.016490)(2.500000,0.018420)(3.750000,0.017010)(5.000000,0.016640)(6.250000,0.017250)(7.500000,0.020310)(8.750000,0.026895)(10.000000,0.044070)(11.250000,0.095735)(12.500000,0.172230)(13.750000,0.331570)(15.000000,0.444300)(16.250000,0.500335)(17.500000,0.547830)(18.750000,0.589950)(20.000000,0.625675)(21.250000,0.655970)(22.500000,0.687785)(23.750000,0.714360)(25.000000,0.737215)(26.250000,0.759905)(27.500000,0.783395)(28.750000,0.799705)(30.000000,0.818235)(31.250000,0.835130)(32.500000,0.848315)(33.750000,0.863455)(35.000000,0.876185)(36.250000,0.887530)(37.500000,0.900655)(38.750000,0.909590)(40.000000,0.917825)(41.250000,0.926435)(42.500000,0.934060)(43.750000,0.941475)(45.000000,0.947505)(46.250000,0.953180)(47.500000,0.958105)(48.750000,0.963215)(50.000000,0.967270)
      };
      \addlegendentry{$\alpha=0$}
      \addplot[blue,densely dashed,line width=0.5pt] plot coordinates{
          (0.000000,0.018245)(1.250000,0.017545)(2.500000,0.017170)(3.750000,0.017100)(5.000000,0.017635)(6.250000,0.018840)(7.500000,0.019470)(8.750000,0.020560)(10.000000,0.020770)(11.250000,0.026310)(12.500000,0.033420)(13.750000,0.044180)(15.000000,0.068715)(16.250000,0.113250)(17.500000,0.219740)(18.750000,0.402385)(20.000000,0.528185)(21.250000,0.595990)(22.500000,0.647980)(23.750000,0.692655)(25.000000,0.727540)(26.250000,0.760235)(27.500000,0.787660)(28.750000,0.809930)(30.000000,0.833005)(31.250000,0.851990)(32.500000,0.868470)(33.750000,0.883455)(35.000000,0.897795)(36.250000,0.910910)(37.500000,0.921360)(38.750000,0.932830)(40.000000,0.939990)(41.250000,0.947940)(42.500000,0.954040)(43.750000,0.960625)(45.000000,0.965430)(46.250000,0.969625)(47.500000,0.974040)(48.750000,0.978670)(50.000000,0.980995)
};
      \addlegendentry{$\alpha=\frac12$}
}
      \addplot[blue,densely dotted,line width=0.5pt] plot coordinates{
          (0.000000,0.015615)(1.250000,0.015390)(2.500000,0.016740)(3.750000,0.016030)(5.000000,0.016025)(6.250000,0.016485)(7.500000,0.016210)(8.750000,0.017455)(10.000000,0.016765)(11.250000,0.016960)(12.500000,0.019225)(13.750000,0.019240)(15.000000,0.021075)(16.250000,0.020520)(17.500000,0.019990)(18.750000,0.021140)(20.000000,0.023670)(21.250000,0.023290)(22.500000,0.026535)(23.750000,0.027700)(25.000000,0.032405)(26.250000,0.032670)(27.500000,0.038260)(28.750000,0.043065)(30.000000,0.057275)(31.250000,0.063080)(32.500000,0.088855)(33.750000,0.117325)(35.000000,0.149375)(36.250000,0.209790)(37.500000,0.299270)(38.750000,0.359705)(40.000000,0.499970)(41.250000,0.600195)(42.500000,0.682285)(43.750000,0.718695)(45.000000,0.749200)(46.250000,0.782620)(47.500000,0.802675)(48.750000,0.825470)(50.000000,0.841235)
      };
      \addlegendentry{$\alpha=1$}
     {
      \addplot[red,line width=1pt] plot coordinates{
          (0.000000,0.018130)(1.250000,0.017950)(2.500000,0.016620)(3.750000,0.017280)(5.000000,0.017975)(6.250000,0.016600)(7.500000,0.019970)(8.750000,0.024855)(10.000000,0.029965)(11.250000,0.054840)(12.500000,0.106060)(13.750000,0.242680)(15.000000,0.392755)(16.250000,0.484455)(17.500000,0.544925)(18.750000,0.594330)(20.000000,0.639380)(21.250000,0.674630)(22.500000,0.709545)(23.750000,0.739220)(25.000000,0.764105)(26.250000,0.789365)(27.500000,0.812145)(28.750000,0.828930)(30.000000,0.848770)(31.250000,0.865240)(32.500000,0.878000)(33.750000,0.892540)(35.000000,0.904290)(36.250000,0.915590)(37.500000,0.925855)(38.750000,0.934940)(40.000000,0.941690)(41.250000,0.949205)(42.500000,0.955710)(43.750000,0.961480)(45.000000,0.966295)(46.250000,0.970040)(47.500000,0.974050)(48.750000,0.978535)(50.000000,0.980710)
      };
      \addlegendentry{$\alpha=\alpha_{\rm opt}$}
      }
      {
      \addplot[green!70!black,dashdotted,line width=0.5pt] plot coordinates{
          (0.000000,0.015195)(1.250000,0.015130)(2.500000,0.015490)(3.750000,0.014745)(5.000000,0.015275)(6.250000,0.013920)(7.500000,0.015750)(8.750000,0.015770)(10.000000,0.016030)(11.250000,0.015165)(12.500000,0.015895)(13.750000,0.016050)(15.000000,0.017110)(16.250000,0.021015)(17.500000,0.158695)(18.750000,0.468200)(20.000000,0.639750)(21.250000,0.669635)(22.500000,0.701285)(23.750000,0.727705)(25.000000,0.749290)(26.250000,0.772960)(27.500000,0.793735)(28.750000,0.809760)(30.000000,0.828355)(31.250000,0.845480)(32.500000,0.859065)(33.750000,0.873905)(35.000000,0.885300)(36.250000,0.897120)(37.500000,0.909570)(38.750000,0.918315)(40.000000,0.925600)(41.250000,0.933695)(42.500000,0.940770)(43.750000,0.947320)(45.000000,0.953390)(46.250000,0.957960)(47.500000,0.962855)(48.750000,0.967320)(50.000000,0.971315)
      };
      \addlegendentry{Bethe Hessian}
      }
      {
      \addplot[blue,mark=o,only marks,line width=0.5pt] plot coordinates{(11.0662,0)};
      \addplot[blue,mark=o,only marks,densely dashed,line width=0.5pt] plot coordinates{(16.7092,0)};
      }
      \addplot[blue,mark=o,only marks,densely dotted,line width=0.5pt] plot coordinates{(31.8838,0)};
      \addlegendentry{Phase transition}
      {
      \addplot[red,mark=o,only marks,line width=1pt] plot coordinates{(11.0634,0)};
      }
    \end{axis}
  \end{tikzpicture}
  \caption{Overlap performance for $n=3000$, $K=3$, $c_i=\frac13$, $\mu=\frac34\delta_{q_{(1)}}+\frac14\delta_{q_{(2)}}$ with $q_{(1)}=0.1$ and $q_{(2)}=0.5$, ${\bf M}=\Delta {\bf I}_3$, for $\Delta\in[5,50]$. Here $\alpha_{\rm opt}=0.07$.}
  \label{fig:overlap_2masses}
\end{figure}
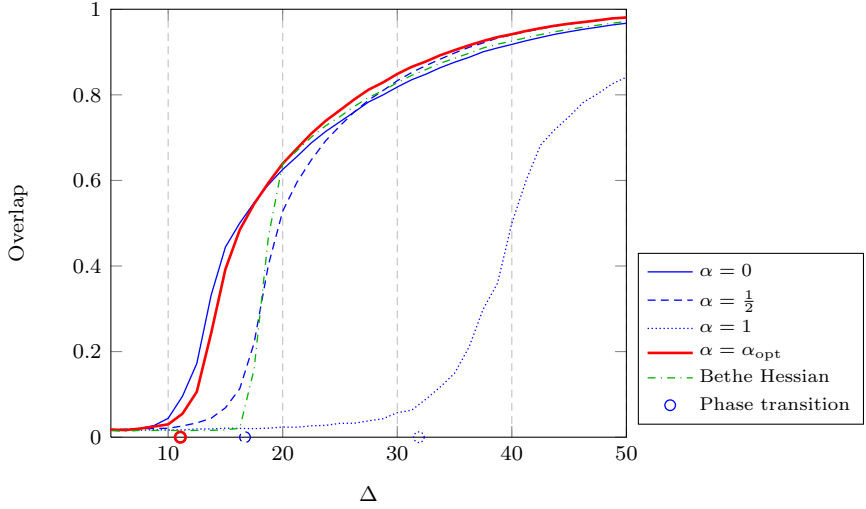
We then present in Figure~\ref{fig:overlap_2masses_q} an example where the BH algorithm fails due to strongly heterogeneous node degrees. Assuming nodes connect with either low $q_{(1)}=0.1$ or high $q_{(2)}>q_{(1)}$ intrinsic weights, we observe a sudden drop of the BH overlap for large $q_{(2)}-q_{(1)}$. This phenomenon is consistent with the fact, observed earlier in Figure~\ref{fig:2Dplot}, that BH creates artificial communities out of nodes with the same $q_i$ parameter. This is a practical demonstration of the need for a proper eigenvector normalization to avoid degree biases. This observation has recently led~\cite{newman2013spectral} to consider a regularization for the non-backtracking operator on which the BH method is based.

\begin{figure}[h!]
  \centering
  \begin{tikzpicture}[font=\footnotesize]
    \renewcommand{\axisdefaulttryminticks}{4}
    \tikzstyle{every major grid}+=[style=densely dashed]      
    \tikzstyle{every axis y label}+=[yshift=-10pt]
    \tikzstyle{every axis x label}+=[yshift=5pt]
    \tikzstyle{every axis legend}+=[cells={anchor=west},fill=white,
        at={(0.02,0.98)}, anchor=north west, font=\scriptsize ]
    \begin{axis}[
      xmin=0.1,
      ymin=0,
      xmax=0.9,
      ymax=1,
      grid=major,
      ymajorgrids=false,
      scaled ticks=true,
      ylabel={Overlap},
      xlabel={$q_{(2)}$ ($q_{(1)}=0.1$)},
      legend style={at={(1.02,0.2)},anchor=west}
      ]
      \addplot[blue,line width=0.5pt] plot coordinates{
          (0.100000,0.026515)(0.125000,0.030275)(0.150000,0.034090)(0.175000,0.042005)(0.200000,0.051600)(0.225000,0.078450)(0.250000,0.122930)(0.275000,0.180675)(0.300000,0.295255)(0.325000,0.414490)(0.350000,0.486430)(0.375000,0.529995)(0.400000,0.568070)(0.425000,0.599035)(0.450000,0.623780)(0.475000,0.648795)(0.500000,0.668385)(0.525000,0.686080)(0.550000,0.699165)(0.575000,0.712895)(0.600000,0.722685)(0.625000,0.736160)(0.650000,0.745200)(0.675000,0.756255)(0.700000,0.762940)(0.725000,0.773540)(0.750000,0.782290)(0.775000,0.789625)(0.800000,0.797420)(0.825000,0.805280)(0.850000,0.809605)(0.875000,0.819305)(0.900000,0.823135)
      };
      \addplot[blue,densely dashed,line width=0.5pt] plot coordinates{
          (0.100000,0.027535)(0.125000,0.031095)(0.150000,0.031085)(0.175000,0.040055)(0.200000,0.044480)(0.225000,0.057015)(0.250000,0.078485)(0.275000,0.097720)(0.300000,0.140605)(0.325000,0.190050)(0.350000,0.272445)(0.375000,0.355910)(0.400000,0.454770)(0.425000,0.526080)(0.450000,0.565195)(0.475000,0.601065)(0.500000,0.628560)(0.525000,0.652025)(0.550000,0.669990)(0.575000,0.688230)(0.600000,0.702800)(0.625000,0.718000)(0.650000,0.730865)(0.675000,0.743465)(0.700000,0.752195)(0.725000,0.764185)(0.750000,0.773360)(0.775000,0.782330)(0.800000,0.790315)(0.825000,0.798290)(0.850000,0.809420)(0.875000,0.815825)(0.900000,0.821235)
      };
      \addplot[blue,densely dotted,line width=0.5pt] plot coordinates{
          (0.100000,0.016640)(0.125000,0.019700)(0.150000,0.017925)(0.175000,0.020840)(0.200000,0.021535)(0.225000,0.021625)(0.250000,0.020025)(0.275000,0.017915)(0.300000,0.021420)(0.325000,0.020460)(0.350000,0.020835)(0.375000,0.020230)(0.400000,0.022450)(0.425000,0.021570)(0.450000,0.022490)(0.475000,0.021310)(0.500000,0.024020)(0.525000,0.023240)(0.550000,0.021910)(0.575000,0.024255)(0.600000,0.023740)(0.625000,0.023235)(0.650000,0.024155)(0.675000,0.023535)(0.700000,0.023605)(0.725000,0.022860)(0.750000,0.022425)(0.775000,0.024290)(0.800000,0.022035)(0.825000,0.025845)(0.850000,0.021800)(0.875000,0.023100)(0.900000,0.025580)
      };
      \addplot[red,line width=1pt] plot coordinates{
          (0.100000,0.027935)(0.125000,0.029875)(0.150000,0.035060)(0.175000,0.048325)(0.200000,0.055240)(0.225000,0.080305)(0.250000,0.144460)(0.275000,0.204785)(0.300000,0.330655)(0.325000,0.446430)(0.350000,0.514640)(0.375000,0.553030)(0.400000,0.591900)(0.425000,0.621690)(0.450000,0.646185)(0.475000,0.672040)(0.500000,0.691995)(0.525000,0.707820)(0.550000,0.721965)(0.575000,0.734635)(0.600000,0.744635)(0.625000,0.756875)(0.650000,0.766080)(0.675000,0.777305)(0.700000,0.783780)(0.725000,0.792980)(0.750000,0.802170)(0.775000,0.809860)(0.800000,0.816790)(0.825000,0.823535)(0.850000,0.829795)(0.875000,0.837975)(0.900000,0.841065)
      };
      \addplot[green!70!black,dashdotted,line width=0.5pt] plot coordinates{
          (0.100000,0.019515)(0.125000,0.019660)(0.150000,0.019800)(0.175000,0.018230)(0.200000,0.018680)(0.225000,0.018885)(0.250000,0.023655)(0.275000,0.036860)(0.300000,0.083705)(0.325000,0.200035)(0.350000,0.325885)(0.375000,0.467710)(0.400000,0.584180)(0.425000,0.614115)(0.450000,0.645105)(0.475000,0.665860)(0.500000,0.683465)(0.525000,0.695440)(0.550000,0.705495)(0.575000,0.704280)(0.600000,0.681170)(0.625000,0.620700)(0.650000,0.545395)(0.675000,0.438580)(0.700000,0.394540)(0.725000,0.341275)(0.750000,0.319325)(0.775000,0.280340)(0.800000,0.280515)(0.825000,0.269530)(0.850000,0.269685)(0.875000,0.260785)(0.900000,0.279285)
      };
      \legend{ {$\alpha=0$},{$\alpha=0.5$},{$\alpha=1$},{$\alpha=\hat{\alpha}_{\rm opt}$},{Bethe Hessian},{Phase transition} }
    \end{axis}
  \end{tikzpicture}
  \caption{Overlap for $n=3000$, $K=3$, $\mu=\frac34\delta_{q_{(1)}}+\frac14\delta_{q_{(2)}}$ with $q_{(1)}=0.1$ and $q_{(2)}\in[0.1,0.9]$, ${\bf M}$ defined by $M_{ii}=10$, $M_{ij}=-10, i\neq j$, $c_i=\frac13$.}
  \label{fig:overlap_2masses_q}
\end{figure}
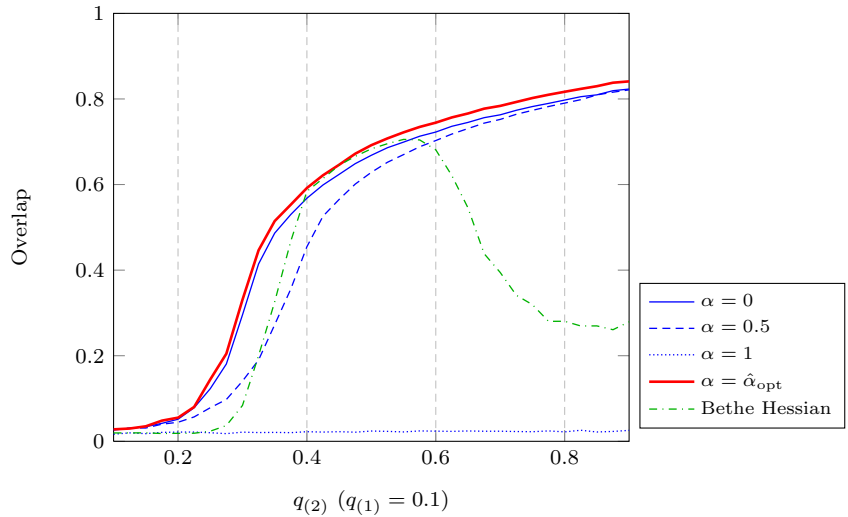


In Figure~\ref{fig:overlap_powerlaw}, we consider a more realistic synthetic graph where the $ q_i $'s assume a power law of support $[0.05,0.3]$ which simulates a sparse graph characteristic of real world networks. Although this is not the regime we study in this article, our method for $ \alpha=\hat{\alpha}_{\rm opt} $ still competes with the BH method which was developped for sparse homogeneous graphs. However, it is seen that the theoretical phase transitions do not closely match the empirical ones espectially for the case $ \alpha=1 $. This mismatch is likely due to the fact that our theoretical results in this article require $ P_{ij}=\mathcal{O}(1) $ which is not always the case in this scenario.
\begin{figure}[h!]
  \centering
  \begin{tikzpicture}[font=\footnotesize]
    \renewcommand{\axisdefaulttryminticks}{4}
    \tikzstyle{every major grid}+=[style=densely dashed]      
    \tikzstyle{every axis y label}+=[yshift=-10pt]
    \tikzstyle{every axis x label}+=[yshift=5pt]
    \tikzstyle{every axis legend}+=[cells={anchor=west},fill=white,
        at={(0.02,0.98)}, anchor=north west, font=\scriptsize ]
    \begin{axis}[
      xmin=10,
      ymin=0,
      xmax=150,
      ymax=1,
      grid=major,
      ymajorgrids=false,
      scaled ticks=true,
      xlabel={$\Delta$},
      ylabel={Overlap},
      legend style={at={(axis cs:152,0.23)},anchor=west}
      ]
      \addplot[blue,line width=0.5pt] plot coordinates{
          (0.000000,0.015460)(5.000000,0.015443)(10.000000,0.016147)(15.000000,0.019610)(20.000000,0.020640)(25.000000,0.021007)(30.000000,0.026520)(35.000000,0.032000)(40.000000,0.046640)(45.000000,0.072210)(50.000000,0.110260)(55.000000,0.158290)(60.000000,0.266130)(65.000000,0.366310)(70.000000,0.487840)(75.000000,0.585040)(80.000000,0.650180)(85.000000,0.692520)(90.000000,0.731620)(95.000000,0.765930)(100.000000,0.793740)(105.000000,0.818660)(110.000000,0.843550)(115.000000,0.863600)(120.000000,0.881000)(125.000000,0.895510)(130.000000,0.907460)(135.000000,0.921480)(140.000000,0.930890)(145.000000,0.940750)(150.000000,0.948680)
      };
      \addplot[blue,densely dashed,line width=0.5pt] plot coordinates{
          (0.000000,0.015150)(5.000000,0.017603)(10.000000,0.019057)(15.000000,0.017350)(20.000000,0.021840)(25.000000,0.021427)(30.000000,0.021370)(35.000000,0.031020)(40.000000,0.038690)(45.000000,0.065300)(50.000000,0.120340)(55.000000,0.266940)(60.000000,0.464100)(65.000000,0.588380)(70.000000,0.663740)(75.000000,0.727500)(80.000000,0.773900)(85.000000,0.811160)(90.000000,0.842900)(95.000000,0.869670)(100.000000,0.891640)(105.000000,0.908840)(110.000000,0.925310)(115.000000,0.937480)(120.000000,0.947280)(125.000000,0.955420)(130.000000,0.962990)(135.000000,0.970460)(140.000000,0.975330)(145.000000,0.979380)(150.000000,0.982710)
};
      \addplot[blue,densely dotted,line width=0.5pt] plot coordinates{
          (0.000000,0.001360)(5.000000,0.001023)(10.000000,0.001157)(15.000000,0.001050)(20.000000,0.001190)(25.000000,0.001287)(30.000000,0.001350)(35.000000,0.001980)(40.000000,0.001670)(45.000000,0.001600)(50.000000,0.001330)(55.000000,0.001850)(60.000000,0.001570)(65.000000,0.002460)(70.000000,0.001890)(75.000000,0.002900)(80.000000,0.002310)(85.000000,0.002100)(90.000000,0.006280)(95.000000,0.003480)(100.000000,0.005840)(105.000000,0.008540)
          (110,0.015)
          (115.000000,0.019840)(120.000000,0.040330)(125.000000,0.119270)(130.000000,0.272510)
          (135,0.38)
          (140.000000,0.483990)(145.000000,0.574640)(150.000000,0.677590)
      };
      \addplot[red,line width=1pt] plot coordinates{
          (0.000000,0.018060)(5.000000,0.019723)(10.000000,0.017287)(15.000000,0.021000)(20.000000,0.022240)(25.000000,0.021957)(30.000000,0.034430)(35.000000,0.049100)(40.000000,0.086860)(45.000000,0.163500)(50.000000,0.321470)(55.000000,0.474480)(60.000000,0.559940)(65.000000,0.622330)(70.000000,0.678080)(75.000000,0.731290)(80.000000,0.772350)(85.000000,0.806510)(90.000000,0.834780)(95.000000,0.861010)(100.000000,0.883250)(105.000000,0.899880)(110.000000,0.915660)(115.000000,0.928510)(120.000000,0.938630)(125.000000,0.948820)(130.000000,0.956690)(135.000000,0.964010)(140.000000,0.969330)(145.000000,0.974260)(150.000000,0.978390)
      };
      \addplot[green!70!black,dashdotted,line width=0.5pt] plot coordinates{
          (0.000000,0.016340)(5.000000,0.015313)(10.000000,0.015417)(15.000000,0.015030)(20.000000,0.016980)(25.000000,0.015547)(30.000000,0.016690)(35.000000,0.018380)(40.000000,0.019290)(45.000000,0.069520)(50.000000,0.215160)(55.000000,0.441200)(60.000000,0.558730)(65.000000,0.623240)(70.000000,0.676100)(75.000000,0.727370)(80.000000,0.766990)(85.000000,0.801060)(90.000000,0.827210)(95.000000,0.853140)(100.000000,0.874710)(105.000000,0.892110)(110.000000,0.908060)(115.000000,0.921030)(120.000000,0.931320)(125.000000,0.941800)(130.000000,0.948970)(135.000000,0.958060)(140.000000,0.963010)(145.000000,0.968840)(150.000000,0.973170)
      };
      \addplot[blue,mark=o,only marks,line width=0.5pt] plot coordinates{(48.8813,0)};
      \addplot[blue,mark=o,only marks,densely dashed,line width=0.5pt] plot coordinates{(47.9614,0)};
      \addplot[blue,mark=o,only marks,densely dotted,line width=0.5pt] plot coordinates{(62.0640,0)};
      \addplot[red,mark=o,only marks,line width=1pt] plot coordinates{(42.2605,0)};
      \legend{ {$\alpha=0$},{$\alpha=0.5$},{$\alpha=1$},{$\alpha=\hat{\alpha}_{\rm opt}$},{Bethe Hessian},{Phase transition} }
    \end{axis}
  \end{tikzpicture}
  \caption{Overlap for $n=3000$, $K=3$, $c_i=\frac13$, $\mu$ a power law with exponent $3$ and support $[0.05,0.3]$, ${\bf M}=\Delta {\bf I}_3$, for $\Delta\in[10,150]$. Here $\hat{\alpha}_{\rm opt}=0.28$.}
  \label{fig:overlap_powerlaw}
\end{figure}
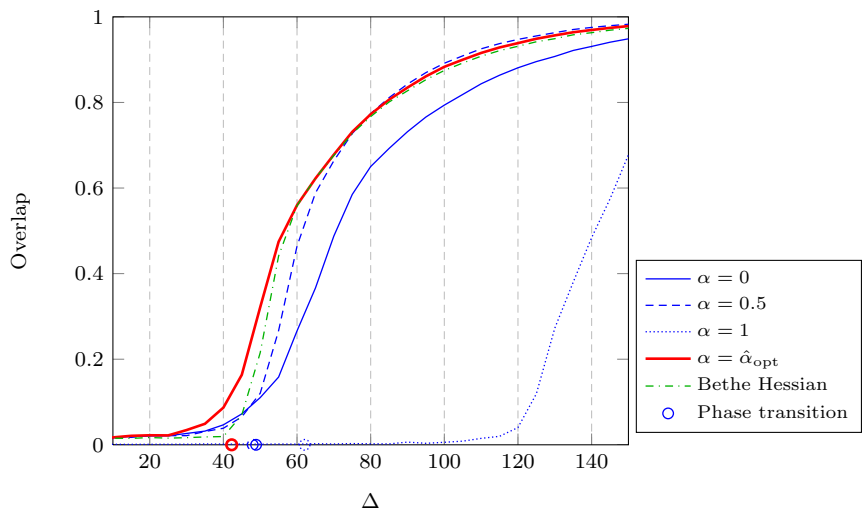

We finally confront the overlap performance on real world benchmarks in Table~\ref{tbl:benchmarks}. The best overlap score for each benchmark is set in boldface and quasi-equal scores are shown in italic. Our approach largely outperforms the BH method on some benchmarks and has competitive performances on others. However note that, for so small network sizes, the performance achieved by ${\bf L}_{\hat{\alpha}_{\rm opt}}$ may be quite unsatisfactory.\footnote{We should note here that the scores for the BH are different from the ones found in the article \cite{saade2014spectral} since here we are running a consistent algorithm (EM in the last step of the spectral algorithm) for all $ K $ while the authors of \cite{saade2014spectral} have instead used the signs of the eigenvector for networks with two communities and k-means algorithm for those with more than two communities.} 

    \begin{table}
    \centering
        \begin{tabular}{l|rrrr||r}
            Graph ($n$, $K$) [reference] & $\alpha=0$ & $\alpha=\frac12$ & $\alpha=1$ & $\hat{\alpha}_{\rm opt}$ & BH \\
            \hline
            Polbooks ($105$, $3$)\cite{newman2006finding}& $\it 0.743$ &  $\bf 0.757$  & $0.214$ &    $\it 0.743$ &  $\bf 0.757$ \\
            Adjnoun ($112$, $2$)\cite{newman2006finding}& $0.571$ &  $\bf 0.714$  & $0.000$ &    $0.571$ &  $0.661$ \\
            Karate ($34$, $2$)\cite{zachary1977information}& $0.176$ &  $\it 0.941$      & $0.353$ &    $0.176$ &  $\bf 1.000$ \\
            Dolphins ($62$, $2$)\cite{lusseau2003bottlenose}& $\bf 0.968$ &  $\bf 0.968$  & $0.387$ &    $\bf 0.968$ & $\it 0.935$ \\
            Polblogs ($1221$, $2$)\cite{adamic2005political}& $\bf 0.897$ &  $0.035$      & $0.040$ &    $\bf 0.897$ &  $0.304$ \\
            Football ($115$, $12$)\cite{newman2006finding}& $0.858$ &  $\it 0.905$      & $\it 0.905$ &    $\it 0.905$& $\bf 0.924$
        \end{tabular}
        \medskip
        \caption{Overlap performance on benchmark graphs.}
        \label{tbl:benchmarks}
    \end{table}
    
\bigskip

In order to assess the performance of Algorithm~\ref{alg:algorithm}, we now need to investigate closely the content of the eigenvectors of $ {\bf L}_{\alpha}$ used for clustering (and of their pre-multiplied by $ {\bf D}^{\alpha-1} $ versions). These regularized eigenvectors happen to be shapped like noisy ``plateaus" (step functions), each plateau characterizing a class. The properties of those noisy plateaus are significant to assess the performance of spectral clustering. The objective of the next section is to characterize exactly those quantities.

\subsection{Eigenvectors}
\label{subsec:eigenvect}
In this section, in order to fully characterize the performances of Algorithm~\ref{alg:algorithm}, we study in depth the normalized eigenvectors $ \bar{{\bf v}}_i^\alpha $ used for the classification in the algorithm (step~\ref{three} of Algorithm~\ref{alg:algorithm}). The eigenvectors corresponding to the eigenvalues for which $ 1+\theta^{\alpha}(\rho)=0 $ are still not considered since we recall they are of no interest for the classification. For technical reasons, we restrict ourselves here to those eigenpairs $ (\lambda_i,\bar{{\bf v}}_i^\alpha) $'s for which there exists no $ \lambda_j \neq \lambda_i $ such that, if $ \lambda_i \to \rho $, $ \lambda_j \to \rho.$

As one can see in Figure~\ref{fig:correct}, the different clusters of points (rows of $ {\bf W} $ in Algorithm~\ref{alg:algorithm}) have different dispersions (variances) in the DCSBM model under consideration. The most appropriate algorithm to use in step~\ref{four} of Algorithm~\ref{alg:algorithm} is the expectation maximization (EM) method. EM considers each point $ {\bf r}_i \in \mathbb{R}^{\ell} $ arising from $[\bar{{\bf v}}_1^\alpha,\dots,\bar{{\bf v}}_\ell^\alpha] $ as a mixture of $ K $ Gaussian random vectors with means $ {\bm \nu}_{EM}^a$ and covariances $ {\bm \Sigma}_{EM}^a \in \mathbb{R}^{\ell \times \ell} $, $a \in \{1,\dots,K\}$. Starting from initial means and covariances, they are sequentially updated until convergence. To identify ${\bm \nu}_{EM}^a $, ${\bm \Sigma}_{EM}^a$ and thus understand the performance of Algorithm~\ref{alg:algorithm}, we may write $ \bar{{\bf v}}_i^\alpha $ \footnote{Recall that the graph nodes were assumed labeled by class, and thus the entries of $\bar{{\bf v}}_i^\alpha$ are similarly sorted by class.} as the ``noisy plateaus" vector 
\begin{equation}
\label{eigenvectorStruct}
\bar{{\bf v}}_i^\alpha = \sum_{a=1}^{K} \nu_{i}^{a}{\bf j}_{a} + \sqrt{\sigma_{ii}^{a}}{\bf w}_{i}^{a}
\end{equation}
where $ {\bf w}_{i}^{a} \in \mathbb{R}^{n} $ is a random vector orthogonal to $ {\bf j}_{a}$, of norm $ \sqrt{n_a} $ and supported on the indices of $ \mathcal{C}_a $ and 
\begin{align} \label{classMean}
\nu_{i}^{a}&=\frac{1}{n_a}\left(\bar{{\bf v}}_i^\alpha\right)^{\sf T}{\bf j}_a =\frac{1}{n_a}\frac{({\bf u}_i^{\alpha})^{\sf T}{\bf D}^{\alpha-1}{\bf j}_a}{\sqrt{({\bf u}_i^{\alpha})^{\sf T}{\bf D}^{2(\alpha-1)}{\bf u}_i^{\alpha}}} \\
\label{classCov}
\sigma_{ij}^{a}&=\frac{1}{n_a}\left[\frac{({\bf u}_i^{\alpha})^{\sf T}{\bf D}^{\alpha-1}\mathcal{D}_a{\bf D}^{\alpha-1}{\bf u}_j^{\alpha}}{\sqrt{({\bf u}_i^{\alpha})^{\sf T}{\bf D}^{2(\alpha-1)}{\bf u}_i}\sqrt{({\bf u}_j^{\alpha})^{\sf T}{\bf D}^{2(\alpha-1)}{\bf u}_j^{\alpha}}}\right] - \nu_i^a \nu_j^a
\end{align}
with $ \mathcal{D}_a=\mathcal{D}({\bf j}_a) $. The vector $ {\bm \nu}^{a}=(\nu_i^a)_{i=1}^{\ell} \in \mathbb{R}^{\ell} $ and the matrix $ {\bm \Sigma}^{a}=(\sigma_{ij}^a)_{i,j=1}^{\ell} \in \mathbb{R}^{\ell \times \ell}$ represent respectively the empirical means and empirical covariances of the points $ {\bf r}_i $ (defined in Algorithm~\ref{alg:algorithm}) belonging to class $ \mathcal{C}_a $. Thus, provided that EM converges to the correct solution, $({\bm \nu}_{EM}^a)_i $ and $({\bm \Sigma}_{EM}^a)_{ij}$ shall converge asymptotically to the limiting values of $ \nu_{i}^{a} \in \mathbb{R} $ and $ \sigma_{ij}^{a}$ respectively.

Clearly, for small values of $ {\bm \Sigma}^a $ compared to $ {\bm \nu}^a $, clustering the vectors $ \bar{{\bf v}}_i^\alpha $ shall lead to good performances. We shall investigate here the links between $ {\bm \Sigma}^a $, $ {\bm \nu}^a $ and the model parameters $ {\bf M} $, $ \mu $ and $ \alpha $. 

Technically, the standard tools used in spiked random matrix analysis do not allow for an immediate assessment of the quantities $ \nu_{i}^{a} $ and $ \sigma_{ij}^{a} $. As a workaround, we follow the approach used in \cite{couillet2016kernel} which relies on the possibility to estimate bilinear forms of the type $ {\bf a}^{\sf T}{\bf u}_i^{\alpha}({\bf u}_i^{\alpha})^{\sf T}{\bf b} $ for given vectors $ {\bf a}, {\bf b} \in \mathbb{R}^n $ and unit multiplicity eigenvectors $ {\bf u}_i^{\alpha} $ of $ {\bf L}_\alpha $ as we have from Cauchy formula, as $ n \to \infty $ almost surely, (since $ \lambda_i({\bf L}_\alpha) \to \rho $)
$$ {\bf a}^{\sf T}{\bf u}_i^{\alpha}({\bf u}_i^{\alpha})^{\sf T}{\bf b}=-\frac{1}{2\pi i}\oint_{\Gamma_\rho} {\bf a}^{\sf T}\left({\bf L}_\alpha-z{\bf I}_n\right)^{-1}{\bf b}\mathrm{d}z$$
and for a given matrix $ {\bf D} $ $$ ({\bf u}_i^{\alpha})^{\sf T}{\bf D}{\bf u}_i^{\alpha}=\tr {\bf u}_i^{\alpha}({\bf u}_i^{\alpha})^{\sf T}{\bf D}=  -\frac{1}{2\pi i}\oint_{\Gamma_\rho} \tr \left({\bf L}_\alpha-z{\bf I}_n\right)^{-1}{\bf D}\mathrm{d}z$$
where $ \Gamma_\rho $ is a positively oriented contour circling around the limiting eigenvalue $ \rho $ of $ \lambda_i({\bf L}_{\alpha}) $ associated to the eigenvector $ {\bf u}_i^{\alpha} $ of $ {\bf L}_\alpha $.
Similar to the eigenvalue estimation step in previous sections, the estimation of those quantities then consists in relating $ {\bf L}_\alpha $ to $ \tilde{{\bf L}}_\alpha $ and then to $ \underline{{\bf G}}_z^\alpha $. Precisely,
\begin{itemize}
\item we will estimate the $ \nu_{i}^{a} $'s by obtaining an estimator of the $ K \times K $ matrix $$ \frac{1}{n}\frac{{\bf J}^{\sf T}{\bf D}^{\alpha-1}{\bf u}_i^\alpha({\bf u}_i^\alpha)^{\sf T}{\bf D}^{\alpha-1}{\bf J}}{({\bf u}_i^\alpha)^{\sf T}{\bf D}^{2(\alpha-1)}{\bf u}_i^\alpha}, $$ the diagonal entries of which allow to estimate $ \left|\nu_{i}^{a}\right| $ while the off-diagonal entries are used to decide on the signs of the $ \nu_{i}^{a} $'s (up to a convention in the sign of $ {\bf u}_i^{\alpha} $).
\item Similarly, we may first estimate the more involved object $$ \frac{1}{n}\frac{{\bf J}^{\sf T}{\bf D}^{\alpha-1}{\bf u}_i^\alpha({\bf u}_i^\alpha)^{\sf T}{\bf D}^{\alpha-1}\mathcal{D}_a{\bf D}^{\alpha-1}{\bf u}_j^\alpha({\bf u}_j^\alpha)^{\sf T}{\bf D}^{\alpha-1}{\bf J}}{\left(({\bf u}_i^\alpha)^{\sf T}{\bf D}^{2(\alpha-1)}{\bf u}_i^\alpha\right)\left(({\bf u}_j^\alpha)^{\sf T}{\bf D}^{2(\alpha-1)}{\bf u}_j^\alpha\right)} $$ from which $ \frac{({\bf u}_i^{\alpha})^{\sf T}{\bf D}^{\alpha-1}\mathcal{D}_a{\bf D}^{\alpha-1}{\bf u}_j^{\alpha}}{\sqrt{({\bf u}_i^{\alpha})^{\sf T}{\bf D}^{2(\alpha-1)}{\bf u}_i^{\alpha}}\sqrt{({\bf u}_j^{\alpha})^{\sf T}{\bf D}^{2(\alpha-1)}{\bf u}_j^{\alpha}}} $ can be retrieved by dividing any entry $ e,f $ of the former quantity by non-vanishing quantities $ \nu_{i}^{e}\nu_{i}^{f} $. For the eigenvectors $ {\bf u}_i^{\alpha} $ used for clustering, there is always at least one index $ f $ such that $ \nu_i^f $ is non zero (otherwise, this eigenvector is of no use for clustering).
\end{itemize}
All calculus done, we obtain the following limit for the empirical means $ \nu_{i}^{a} $'s.
\begin{theorem}[Means]
\label{cor:means}
For each eigenpair $ (\lambda(\bar{{\bf M}}),{\bf v})$ of $ \mathcal{D}({\bf c})^{\frac12}\left({\bf I}_{K} -\mathbf{1}_K{\bf c}^{\sf T}\right)\mathbf{M}\left(\mathbf{I}_{K}-\mathbf{c}{\bf 1}_K^{T}\right)\mathcal{D}({\bf c})^{\frac12} $ of unit multiplicity, mapped to eigenpair $ (\rho,{\bf u}_i^\alpha) $ of $ {\bf L}_\alpha $ as defined in Corollary~\ref{phasetransition}, under the conditions of Assumption \ref{as1} and for $ \nu_{i}^{a} $ defined in~\eqref{classMean}, we have almost surely as $ n \to \infty $, $ \left|(\nu_{i}^{a})^{2}-(\nu_i^{a,\infty})^2\right| \rightarrow 0 $ where
\[(\nu_i^{a,\infty})^2 \equiv \frac{1}{n_a}\frac{\left[E_0^{\alpha}(\rho)\right]^{2}}{e_{00;2}^{\alpha}(\rho,\rho)+\chi^{\alpha}(\rho)}(v_{a})^{2}\]
with \\
$ \chi^{\alpha}(\rho)=\frac{\left[\left(1+e_{42;2}^{\alpha}(\rho)\right)e_{-1,0}^{\alpha}(\rho)-e_{32;2}^{\alpha}(\rho)e_{00}^{\alpha}(\rho)\right]e_{32;3}^{\alpha}(\rho)-\left[e_{22;2}^{\alpha}(\rho)e_{-1,0}^{\alpha}(\rho)+\left(1-e_{22;2}^{\alpha}(\rho)\right)e_{00}^{\alpha}(\rho)\right]e_{42;3}^{\alpha}(\rho)}{\left(1+e_{42;2}^{\alpha}(\rho)\right)\left(1-e_{22;2}^{\alpha}(\rho)\right)+\left[e_{32;2}^{\alpha}(\rho)\right]^{2}} $ and $ v_{a} $ is the component $ a $ of $ {\bf v}$.
\end{theorem}

And for the empirical covariances $ \sigma_{ij}^{a} $'s, we have the following limit
\begin{theorem}[Covariances]
\label{cor:cov}
For two unit multiplicity eigenpairs $ (\lambda_1(\bar{{\bf M}}),{\bf v}^{1})$ and $ (\lambda_{2}(\bar{{\bf M}}),{\bf v}^{2})$ of $ \mathcal{D}({\bf c})^{\frac12}\left({\bf I}_{K} -\mathbf{1}_K{\bf c}^{\sf T}\right)\mathbf{M}\left(\mathbf{I}_{K}-\mathbf{c}{\bf 1}_K^{T}\right)\mathcal{D}({\bf c})^{\frac12} $ mapped respectively to $ (\rho_1,{\bf u}_{i}^\alpha) $ and $ (\rho_2,{\bf u}_{j}^\alpha) $ eigenpairs of $ {\bf L}_{\alpha} $ and for $ \sigma_{ij}^{a} $ defined in~\eqref{classCov}, we have almost surely as $ n \to \infty $, $ \left|\sigma_{ij}^{a}-\sigma_{ij}^{a,\infty}\right| \rightarrow 0 $ where
\[\sigma_{ij}^{a,\infty} \equiv \frac{1}{n_a}\frac{\left[\left(e_{00;2}^{\alpha}(\rho_1,\rho_2)-e_{00}^{\alpha}(\rho_1)e_{00}^{\alpha}(\rho_2)\right)v_{a}^{\rho_1}v_{a}^{\rho_2}+\delta_{\rho_1}^{\rho_2}c_{a}\chi^{\alpha}(\rho_1)\right]}{\sqrt{e_{00;2}^{\alpha}(\rho_1)+\chi^{\alpha}(\rho_1)}\sqrt{e_{00;2}^{\alpha}(\rho_2)+\chi^{\alpha}(\rho_2)}}\]
where $ \chi^{\alpha}(\rho)$ is defined in Theorem~\ref{cor:means}.
\end{theorem}

From Theorems~\ref{cor:means} and~\ref{cor:cov}, $ \nu_i^{a,\infty} $ and $ \sigma_{ij}^{a,\infty} $ depend on the $ e_{ij} $'s (defined in Theorem~\ref{determinst2}), the normalized eigenvectors $ {\bf v} $ of $ \mathcal{D}({\bf c})^{\frac12}\left({\bf I}_{K} -\mathbf{1}_K{\bf c}^{\sf T}\right)\mathbf{M}\left(\mathbf{I}_{K}-\mathbf{c}{\bf 1}_K^{T}\right)\mathcal{D}({\bf c})^{\frac12} $ and the proportions $ c_a $'s of classes. Thanks to Lemma~\ref{lm:consistentEstimate}, the $ e_{ij} $'s can consistently be estimated similarly to what was described in Lemma~\ref{mu_estimation}. However, the eigenvectors $ {\bf v} $ and the class proportions are not directly accessible in practice. Nevertheless, in the particular case of $ K=2 $ classes, we know exactly $ {\bf v} $.

\begin{remark}[$ K=2 $ classes]
\label{2Classes}
Here, only one isolated eigenvector is used for the classification. Since $ {\bf v}_r $ (right eigenvector of $ \bar{{\bf M}} $) is orthogonal to $ {\bf 1}_2 $, $ {\bf v}_r  $ is necessarily the vector $ \left[1,-1\right]^{\sf T} $. Hence, the normalized eigenvector $ {\bf v}=\frac{\mathcal{D}({\bf c})^{-\frac12}{\bf v}_r}{\|\mathcal{D}({\bf c})^{-\frac12}{\bf v}_r\|} $ is  $\frac{1}{\sqrt{1/c_1+1/c_2}}\left[\frac{1}{\sqrt{c_{1}}},-\frac{1}{\sqrt{c_{2}}}\right]^{\sf T}$. 
\end{remark}
We thus obtain from Theorems~\ref{cor:means} and~\ref{cor:cov} along with Remark~\ref{2Classes},
\begin{corollary}[Means and covariances for $ K=2 $ classes] \label{cor:2class}
For $ a=1,2 $
\begin{align*}
(\nu^{a,\infty})^2&=\frac{n}{n_a^2}\frac{\left[e_{00}^{\alpha}(\rho)\right]^{2}}{e_{00;2}^{\alpha}(\rho,\rho)+\chi^{\alpha}(\rho)}\frac{1}{\left(1/c_{1}+1/c_{2}\right)} \\
(\sigma^{a,\infty})^{2}&=\frac{\left[\frac{n}{n_a^2}\frac{\left(e_{00;2}^{\alpha}(\rho,\rho)-e_{00}^{\alpha}(\rho)^{2}\right)}{\left(1/c_{1}+1/c_{2}\right)}+\frac{1}{n}\chi^{\alpha}(\rho)\right]}{e_{00;2}^{\alpha}(\rho,\rho)+\chi^{\alpha}(\rho)}
\end{align*}
for $ \rho $ the unique isolated eigenvalue of $ {\bf L}_\alpha $ (if it exists).
\end{corollary}
\begin{remark}[Case $ q_i=q_0 $]
\label{QiQ0}
Here, since $ \forall i,$ $ q_i=q_0$, all the $ e_{ij}(\rho) $'s are completely explicit and only depend on $ q_0 $, $ \rho $ and $ \alpha $ while for $ q_i \neq q_0 $, the $ e_{ij}(\rho) $'s can only be found by solving a fixed point equation numerically. Furthermore, $ e_{00;2}(\rho_1,\rho_2)=e_{00}(\rho_1)e_{00}(\rho_2) $ and thus 
\begin{align*}
\sigma_{ii}^{a,\infty} &= \frac{1}{n}\frac{\chi^{\alpha}(\rho)}{e_{00;2}^{\alpha}(\rho)+\chi^{\alpha}(\rho)} \\
\sigma_{ij}^{a,\infty}&=0 \quad i \neq j
\end{align*}
with $ \chi^{\alpha}(\rho)=\frac{q_0^{2-4\alpha}(1-q_0^2)}{(-\rho-q_0^{1-2\alpha}(1-q_0^2)E_1^{\alpha}(\rho))^{2}((-\rho-q_0^{1-2\alpha}(1-q_0^2)E_1^{\alpha}(\rho))^{2}-q_0^{2-4\alpha}(1-q_0^2))} $, \\ $ e_{00;2}^{\alpha}(\rho)=\frac{1}{(-\rho-q_0^{1-2\alpha}(1-q_0^2)E_1^{\alpha}(\rho))^{2}} $ and $ E_1^{\alpha}(\rho)=-\frac{\rho}{2q_0^{1-2\alpha}(1-q_0^2)}-\sqrt{(\frac{\rho}{2q_0^{1-2\alpha}(1-q_0^2)})^{2}-1}.$
This indicates that the normalized eigenvectors $ \bar{{\bf v}}_i^\alpha \in \mathbb{R}^{n}$, $ 1\leq i \leq \ell $ used for the classification (Step~\ref{four2} of Algorithm~\ref{alg:algorithm}) are asymptotically uncorrelated. Thus, the classification can be performed by treating each eigenvector independently, rather than jointly. In addition, we have $ \sigma_{ii}^{a,\infty}=\sigma_{ii}^{b,\infty} $ for $ 1\leq a \neq b \leq K $ meaning that the nodes can only be discriminated based on the means $ \nu_i^a $'s.
\end{remark}

\subsection{Application 2: EM improvement and Performance analysis}
\label{subsec:perf}
\subsubsection{EM improvement}
The performances of EM highly depend on the chosen starting parameters; a first natural choice is to set them randomly, which as we shall see leads to poor performances especially in cases where the clusters are not easily separable.  
 
Since the theoretical limiting means $ {\bm \nu}^{a,\infty} $ and covariances $ {\bm \Sigma}^{a,\infty} $ are respectively the limiting values of $ {\bm \nu}^{a}_{EM} $ and covariances $ {\bm \Sigma}^{a}_{EM} $ provided EM converges to the correct solution, we may set as initial parameters of EM our findings ${\bm \nu}^{a,\infty}$ (Theorem~\ref{cor:means}) and $ {\bm \Sigma}^{a,\infty} $ (Theorem~\ref{cor:cov}) for $ a \in \{1,\dots,K\} $ if those can be estimated. It will be confirmed later via simulations that this new setting of initial parameters is much better than a random initial setting of EM, and is close to the theoretical performances evaluated in Section~\ref{Perf}. Unfortunately, in most scenarios, the many unknowns prevent such an estimation. Nonetheless, from Corollary~\ref{cor:2class}, provided the class proportions (or the sizes of each class) are (more or less) known, we can consistently estimate ${\bm \nu}^{\infty}$ and $ {\bm \Sigma}^{\infty}$ in a $ 2 $-class scenario. To show the effect of our setting of initial parameters of EM based on the findings ${\bm \nu}^{\infty}$ and $ {\bm \Sigma}^{\infty}$, Figure~\ref{fig:perfTheo2} compares the empirical performances of our new spectral algorithm based on the regularized eigenvector of $ {\bf L}_{0.5} $  for different initial settings of the EM parameters $ \textit{i)}  $ random setting (Random EM) $ \textit{ii)}  $ our theoretical setting (by assuming that the class proportions are known) and $ \textit{iii)}  $ the ground truth setting (oracle EM where we set the initial points to the empirically evaluated means and covariances of each cluster based on ground truth) versus the theoretical ones (evaluated in Section~\ref{Perf}). Below the transition point, no consistent clustering can be achieved for large $ n $ using the eigenvectors associated to highest eigenvalues since the clusters are not separable and our theoretical limiting means and covariances are not defined since there is no isolated eigenvalues in that case. We have thus initialized EM at random in this non interesting regime (as for Random EM). The EM algorithm may in that regime set all the nodes to the same cluster, which will then result to a classification rate close to the proportion of the nodes in the cluster of largest size. This explains the important mismatch between theory and practice (the retrieval of the theoretical plot is discussed in the subsequent section) below the transition threshold, when $ c_1$ and $ c_2 $ are not equal (Figure~\ref{fig:perfTheo2}). In the interesting regime (after the transition point), we see that the performances (in terms of correct classification rate) of the algorithm using our theoretical setting of EM closely match the performances of an ideal setting with ground truth (oracle EM). The performances of the algorithm using a random initialization (Random EM) are completely degraded especially around critical cases (small values of $\Delta $). Random EM becomes reliable only for very large values of $\Delta $ where clustering is somewhat trivial.
\begin{figure}[h!]
  \centering
  \begin{tikzpicture}[font=\footnotesize]
    \renewcommand{\axisdefaulttryminticks}{4} 
    \tikzstyle{every major grid}+=[style=densely dashed]       
    \tikzstyle{every axis y label}+=[yshift=-10pt] 
    \tikzstyle{every axis x label}+=[yshift=5pt]
    \tikzstyle{every axis legend}+=[cells={anchor=west},fill=white,
        at={(0.98,0.02)}, anchor=south east, font=\scriptsize ]
    \begin{axis}[
      xmin=0,
      ymin=0.5,
      xmax=20,
      ymax=1,
      grid=major,
      ymajorgrids=false,
      scaled ticks=true,
      xlabel={$\Delta$},
      ylabel={Correct classificate rate},
      legend style={at={(axis cs:20.2,0.7)},anchor=west}
      ]
      
     \addplot[red,densely dashed,line width=1pt] plot coordinates{
(0.500000,0.500000)(1.000000,0.500000)(1.500000,0.500000)(2.000000,0.500000)(2.500000,0.500000)(3.000000,0.500000)(3.500000,0.500000)(4.000000,0.500000)(4.500000,0.500000)(5.000000,0.500000)(5.500000,0.500000)(6.000000,0.500000)(6.500000,0.500000)(7.000000,0.500000)(7.500000,0.500000)(8.000000,0.500000)(8.500000,0.500000)(9.000000,0.500000)(9.500000,0.640903)(10.000000,0.708689)(10.500000,0.755696)(11.000000,0.792304)(11.500000,0.822213)(12.000000,0.847278)(12.500000,0.868599)(13.000000,0.886900)(13.500000,0.902697)(14.000000,0.916377)(14.500000,0.928242)(15.000000,0.938538)(15.500000,0.947468)(16.000000,0.955206)(16.500000,0.961901)(17.000000,0.967681)(17.500000,0.972662)(18.000000,0.976941)(18.500000,0.980609)(19.000000,0.983744)(19.500000,0.986415)(20.000000,0.988683)
      };
      \addlegendentry{Theory}
      \addplot[blue,densely dashed,line width=1pt] plot coordinates{
     (0.500000,0.705365)(1.500000,0.684637)(2.500000,0.704698)(3.500000,0.698325)(4.500000,0.693848)(5.500000,0.671170)(6.500000,0.679725)(7.500000,0.675175)(8.500000,0.661105)(9.500000,0.671185)(10.500000,0.772962)(11.500000,0.849042)(12.500000,0.891620)(13.500000,0.918230)(14.500000,0.932115)(15.500000,0.944705)(16.500000,0.956935)(17.500000,0.964592)(18.500000,0.972215)(19.500000,0.978268)
      };
      \addlegendentry{oracle EM}
      
   
      \addplot[red,densely dotted,line width=1pt] plot coordinates{
(0.500000,0.720050)(1.500000,0.721550)(2.500000,0.742050)(3.500000,0.673775)(4.500000,0.710138)(5.500000,0.717100)(6.500000,0.657675)(7.500000,0.679513)(8.500000,0.685450)(9.500000,0.677825)(10.500000,0.795400)(11.500000,0.852350)(12.500000,0.892625)(13.500000,0.918688)(14.500000,0.930675)(15.500000,0.945025)(16.500000,0.954788)(17.500000,0.967138)(18.500000,0.972837)(19.500000,0.977762)
      };
      \addlegendentry{Our initialization}

\addplot[blue,densely dotted,line width=1pt] plot coordinates{
(0.500000,0.663713)(1.500000,0.625888)(2.500000,0.649962)(3.500000,0.662163)(4.500000,0.675950)(5.500000,0.684113)(6.500000,0.650975)(7.500000,0.670100)(8.500000,0.656450)(9.500000,0.676800)(10.500000,0.744113)(11.500000,0.722438)(12.500000,0.588825)(13.500000,0.647562)(14.500000,0.763013)(15.500000,0.891088)(16.500000,0.943137)(17.500000,0.966713)(18.500000,0.972575)(19.500000,0.977575)

      };
      \addlegendentry{Random EM}
      
    \end{axis}
  \end{tikzpicture}
  \caption{Probability of correct recovery for $ \alpha=0.5 $, $n=4000$, $K=2$, $c_1=0.8$, $c_2=0.2$, $\mu=\frac{3}{4}\delta_{q_{(1)}}+\frac{1}{4}\delta_{q_{(2)}}$ with $ q_{(1)}=0.2 $ and $ q_{(2)}=0.8 $, ${\bf M}=\Delta {\bf I}_2$, for $\Delta\in[0,20]$.}
  \label{fig:perfTheo2}
\end{figure}
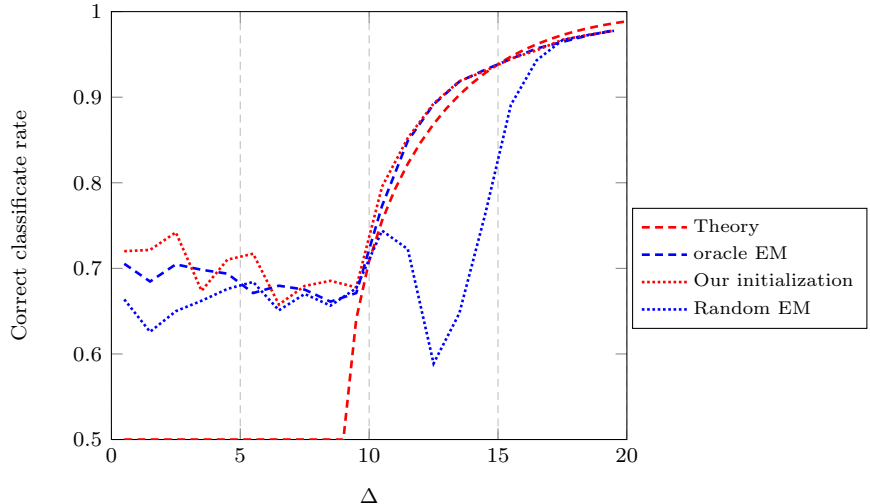

\bigskip
\subsubsection{Performance analysis}
\label{Perf}
In this section, we analyze the performance of our new spectral community detection algorithm in terms of probability of misclassification. 

From our theoretical class means and class covariance findings, we may compute the theoretical misclassification error probability by defining the proper decision regions which separate the different clusters in $\ell$-dimensional space. For simplicity, we focus here on $ K=2 $ so that $ \ell=1 $, but the generalization to $ K>2 $ is straightforward.

The univariate decision boundary is found by solving for $ x $ $ \log \frac{p\left(x|\mathcal{C}_1\right)c_{1}}{p\left(x|\mathcal{C}_2\right)c_2}=0 $ where $ p\left(x|\mathcal{C}_a\right)=\frac{1}{\sqrt{2\pi \sigma^{a,\infty}}}\exp(-\frac{(x-\nu^{a,\infty})^{2}}{2\sigma^{a,\infty}}) $. 
When $\sigma^{1,\infty}=\sigma^{2,\infty} $ (this is the case for example when $ q_i=q_0$ as per Remark~\ref{QiQ0}) with $ \nu^{1,\infty} \leq \nu^{2,\infty} $, there is only one decision threshold $ x_T=\frac{\sigma^{1,\infty}\log c_1/c_2}{2(\nu^{1,\infty}-\nu^{2,\infty})}+\frac{\nu^{1,\infty}+\nu^{2,\infty}}{2}$ and the asymptotic misclassification error probability achievable by EM is $$ c_1\left(1-Q\left(\frac{x_T-\nu^{1,\infty}}{\sqrt{\sigma^{1,\infty}}}\right)\right)+c_2Q\left(\frac{x_T-\nu^{2,\infty}}{\sqrt{\sigma^{2,\infty}}}\right), $$ where $ Q(x)=\int_{-\infty}^{x}\frac{1}{\sqrt{2\pi}}\exp(-\frac{x^2}{2})\mathrm{d}x. $
Otherwise, when the clusters have different variances ($ \sigma^{1,\infty}< \sigma^{2,\infty} $), there are two thresholds $ x_{1T} $, $ x_{2T} $ to $ \log \frac{p\left(x|\mathcal{C}_1\right)c_{1}}{p\left(x|\mathcal{C}_2\right)c_2}=0 $ (with $ x_{1T} < x_{2T} $) and the asymptotic misclassification error probability achievable by EM is 
$$ c_1+c_1\left[Q\left(\frac{x_{1T}-\nu^{1,\infty}}{\sqrt{\sigma^{1,\infty}}}\right)-Q\left(\frac{x_{2T}-\nu^{1,\infty}}{\sqrt{\sigma^{1,\infty}}}\right)\right]+c_2\left[Q\left(\frac{x_{1T}-\nu^{2,\infty}}{\sqrt{\sigma^{2,\infty}}}\right)-Q\left(\frac{x_{2T}-\nu^{2,\infty}}{\sqrt{\sigma^{2,\infty}}}\right)\right].$$

Figure~\ref{fig:perfTheo} displays the theoretical probability of the correct recovery of the classes (which is the complement of the misclassification error probablity evaluated above) for a graph generated using the DCSBM when the clustering is performed on the properly normalized eigenvectors of $ {\bf L}_\alpha $ for $ \alpha \in \{0,0.25,0.5,0.75,1,\alpha_{opt}\} $. While it has been theoretically designed to be optimal in worst case scenarios (small values of eigenvalues of $ {\bf M} $), the algorithm using $\alpha_{\rm opt} $ seemingly outperforms those using the other values of $ \alpha $, in the whole range of $ {\bf M} $ values (driven by $\Delta$). 
\begin{figure}[h!]
  \centering
  \begin{tikzpicture}[font=\footnotesize]
    \renewcommand{\axisdefaulttryminticks}{4} 
    \tikzstyle{every major grid}+=[style=densely dashed]       
    \tikzstyle{every axis y label}+=[yshift=-10pt] 
    \tikzstyle{every axis x label}+=[yshift=5pt]
    \tikzstyle{every axis legend}+=[cells={anchor=west},fill=white,
        at={(0.98,0.02)}, anchor=south east, font=\scriptsize ]
    \begin{axis}[
      xmin=2,
      ymin=0.5,
      xmax=20,
      ymax=1,
      grid=major,
      ymajorgrids=false,
      scaled ticks=true,
      xlabel={$\Delta$},
      ylabel={Correct classificate rate},
      ]
      
      \addplot[blue,line width=0.5pt] plot coordinates{
(0.500000,0.500000)(1.000000,0.500000)(1.500000,0.500000)(2.000000,0.500000)(2.500000,0.500000)(3.000000,0.500000)(3.500000,0.500000)(4.000000,0.500000)(4.500000,0.500000)(5.000000,0.656352)(5.500000,0.731560)(6.000000,0.785379)(6.500000,0.827627)(7.000000,0.861807)(7.500000,0.889740)(8.000000,0.912594)(8.500000,0.931227)(9.000000,0.946329)(9.500000,0.958475)(10.000000,0.968161)(10.500000,0.975813)(11.000000,0.981800)(11.500000,0.986436)(12.000000,0.989991)(12.500000,0.992686)(13.000000,0.994710)(13.500000,0.996212)(14.000000,0.997315)(14.500000,0.998116)(15.000000,0.998692)(15.500000,0.999101)(16.000000,0.999389)(16.500000,0.999589)(17.000000,0.999726)(17.500000,0.999820)(18.000000,0.999882)(18.500000,0.999924)(19.000000,0.999952)(19.500000,0.999969)(20.000000,0.999981)
      };
      \addlegendentry{$\alpha=0$}
      \addplot[blue,densely dashed,line width=0.5pt] plot coordinates{
      (0.500000,0.500000)(1.000000,0.500000)(1.500000,0.500000)(2.000000,0.500000)(2.500000,0.500000)(3.000000,0.500000)(3.500000,0.500000)(4.000000,0.500000)(4.500000,0.500000)(5.000000,0.560065)(5.500000,0.694282)(6.000000,0.763741)(6.500000,0.813995)(7.000000,0.852944)(7.500000,0.883903)(8.000000,0.908736)(8.500000,0.928684)(9.000000,0.944663)(9.500000,0.957396)(10.000000,0.967471)(10.500000,0.975380)(11.000000,0.981534)(11.500000,0.986278)(12.000000,0.989900)(12.500000,0.992638)(13.000000,0.994686)(13.500000,0.996202)(14.000000,0.997313)(14.500000,0.998118)(15.000000,0.998696)(15.500000,0.999105)(16.000000,0.999392)(16.500000,0.999592)(17.000000,0.999729)(17.500000,0.999821)(18.000000,0.999884)(18.500000,0.999925)(19.000000,0.999952)(19.500000,0.999970)(20.000000,0.999981)
      };
      \addlegendentry{$\alpha=0.25$}
      
      \addplot[red,densely dotted,line width=0.5pt] plot coordinates{
(0.500000,0.500000)(1.000000,0.500000)(1.500000,0.500000)(2.000000,0.500000)(2.500000,0.500000)(3.000000,0.500000)(3.500000,0.500000)(4.000000,0.500000)(4.500000,0.500000)(5.000000,0.500000)(5.500000,0.500000)(6.000000,0.619377)(6.500000,0.733050)(7.000000,0.798738)(7.500000,0.845491)(8.000000,0.880791)(8.500000,0.908084)(9.000000,0.929385)(9.500000,0.946044)(10.000000,0.959045)(10.500000,0.969145)(11.000000,0.976940)(11.500000,0.982912)(12.000000,0.987450)(12.500000,0.990866)(13.000000,0.993415)(13.500000,0.995298)(14.000000,0.996675)(14.500000,0.997671)(15.000000,0.998386)(15.500000,0.998892)(16.000000,0.999247)(16.500000,0.999494)(17.000000,0.999663)(17.500000,0.999778)(18.000000,0.999855)(18.500000,0.999907)(19.000000,0.999940)(19.500000,0.999962)(20.000000,0.999977)
      };
      \addlegendentry{$\alpha=0.5$}
   
      \addplot[red,line width=1pt] plot coordinates{
(0.500000,0.500000)(1.000000,0.500000)(1.500000,0.500000)(2.000000,0.500000)(2.500000,0.500000)(3.000000,0.500000)(3.500000,0.500000)(4.000000,0.500000)(4.500000,0.500000)(5.000000,0.500000)(5.500000,0.500000)(6.000000,0.500000)(6.500000,0.500000)(7.000000,0.500000)(7.500000,0.500000)(8.000000,0.688181)(8.500000,0.778437)(9.000000,0.835581)(9.500000,0.876369)(10.000000,0.906741)(10.500000,0.929753)(11.000000,0.947299)(11.500000,0.960683)(12.000000,0.970863)(12.500000,0.978565)(13.000000,0.984355)(13.500000,0.988675)(14.000000,0.991872)(14.500000,0.994218)(15.000000,0.995924)(15.500000,0.997153)(16.000000,0.998030)(16.500000,0.998649)(17.000000,0.999083)(17.500000,0.999383)(18.000000,0.999589)(18.500000,0.999729)(19.000000,0.999823)(19.500000,0.999886)(20.000000,0.999927)
      };
      \addlegendentry{$\alpha=0.75$}

\addplot[black,line width=1pt] plot coordinates{
(0.500000,0.500000)(1.000000,0.500000)(1.500000,0.500000)(2.000000,0.500000)(2.500000,0.500000)(3.000000,0.500000)(3.500000,0.500000)(4.000000,0.500000)(4.500000,0.500000)(5.000000,0.500000)(5.500000,0.500000)(6.000000,0.500000)(6.500000,0.500000)(7.000000,0.500000)(7.500000,0.500000)(8.000000,0.500000)(8.500000,0.500000)(9.000000,0.500000)(9.500000,0.500000)(10.000000,0.500000)(10.500000,0.684419)(11.000000,0.773311)(11.500000,0.830544)(12.000000,0.871838)(12.500000,0.902836)(13.000000,0.926474)(13.500000,0.944597)(14.000000,0.958487)(14.500000,0.969098)(15.000000,0.977160)(15.500000,0.983246)(16.000000,0.987807)(16.500000,0.991198)(17.000000,0.993698)(17.500000,0.995525)(18.000000,0.996850)(18.500000,0.997801)(19.000000,0.998478)(19.500000,0.998956)(20.000000,0.999290)
      };
      \addlegendentry{$\alpha=1$}
      \addplot[cyan,line width=1pt] plot coordinates{
(0.500000,0.500000)(1.000000,0.500000)(1.500000,0.500000)(2.000000,0.500000)(2.500000,0.500000)(3.000000,0.500000)(3.500000,0.500000)(4.000000,0.500000)(4.500000,0.500000)(5.000000,0.623667)(5.500000,0.716234)(6.000000,0.776370)(6.500000,0.821961)(7.000000,0.858275)(7.500000,0.887350)(8.000000,0.910944)(8.500000,0.930067)(9.000000,0.945567)(9.500000,0.957869)(10.000000,0.967711)(10.500000,0.975474)(11.000000,0.981477)(11.500000,0.986238)(12.000000,0.989837)(12.500000,0.992532)(13.000000,0.994589)(13.500000,0.996141)(14.000000,0.997261)(14.500000,0.998063)(15.000000,0.998662)(15.500000,0.999071)(16.000000,0.999367)(16.500000,0.999573)(17.000000,0.999712)(17.500000,0.999810)(18.000000,0.999877)(18.500000,0.999919)(19.000000,0.999948)(19.500000,0.999967)(20.000000,0.999979)
      };
      \addlegendentry{$\alpha=\alpha_{\rm opt}$}

    \end{axis}
  \end{tikzpicture}
  \caption{Theoretical probability of correct recovery for $n=4000$, $K=2$, $c_1=0.8$, $c_2=0.2$, $\mu$ uniformly distributed between $ 0.2 $ and $ 0.8 $, ${\bf M}=\Delta {\bf I}_2$, for $\Delta\in[0,20]$.}
  \label{fig:perfTheo}
\end{figure}
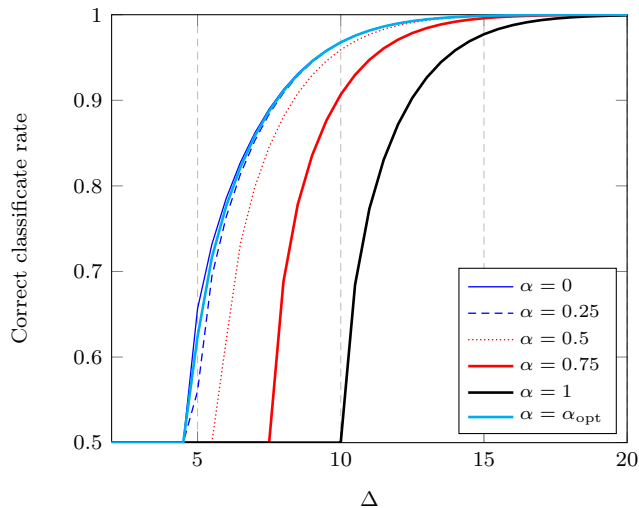


\section{Concluding Remarks}
\label{sec:conclusion}
In this article, we have studied a family of graph affinity matrices $ {\bf L}_\alpha \propto \frac{1}{\sqrt{n}}\mathbf{D}^{-\alpha} \left[ {\bf A} - \frac{{\bf d}{\bf d}^{\sf T}}{2m} \right] \mathbf{D}^{-\alpha}$ which generalize the matrices (modularity, Laplacian) used for spectral community detection in dense networks. The main difficulty for the study of those random matrices comes from the dependency between their entries. We tackle this difficulty by establishing the approximation $ \|{\bf L}_\alpha-\tilde{{\bf L}}_\alpha\| \to 0 $ using similar techniques as in \cite{couillet2015understanding} where $ \tilde{{\bf L}}_\alpha $ belongs to the class of so called ``spiked" random matrices for which the study of eigenvalues and eigenvectors is classical. The study of eigenvalues and eigenvectors of $ {\bf L}_\alpha $ used for the classification is thus performed using $ \tilde{{\bf L}}_\alpha.$

We go further than the observation of \cite{gulikers2015spectral} and \cite{newman2013spectral} which state that it is important to use the eigenvectors of $ {\bf L}_1 $ rather than the classically used $ {\bf L}_0 $ for the classification when the network has heterogeneous degree distribution to avoid some important misclassifications induced by degree biases. We saw in Figure \ref{fig:correct} for example that the eigenvectors of ${\bf L}_1$ correct the degree biases but we show that it is better to use instead the eigenvectors of $ {\bf L}_{0} $ premultiplied by $ {\bf D}^{-1} $ (see Figures \ref{fig:phase_transition_2masses}-\ref{fig:overlap_powerlaw} for example). Better still, we show that there exists an optimal $ \alpha $ called $ \alpha_{\rm opt} $ for which taking the eigenvectors of $ {\bf L}_{\alpha_{\rm opt}} $ pre-multiplied by $ {\bf D}^{\alpha_{\rm opt}-1} $ ensures best performance (or to be more precise best asymptotic cluster detectability).

We generalize the study in \cite{nadakuditi2012graph} concerning the evaluation of per-class means of the entries of the unique eigenvector used for the classification, which was limited to the symmetric stochastic block model with two classes of the same average size. Here we consider a more general model (a non necessarily symmetric stochastic block model with heterogeneous degree distribution (DCSBM), arbitrary number of classes, arbitrary class sizes) and we introduce new techniques to evaluate theoretically the limiting per-class means and covariances of the eigenvectors components which are used for the low-dimensional classification. This in turn allows us to establish exactly the theoretical performance of the spectral methods under study in terms of correct classification rate. Those aforementioned limiting per-class means and per-class covariances are the limiting values of the corresponding per-class means and per-class covariances that the Expectation Maximization (EM) algorithm shall find empirically (in the last step of the spectral method) provided it converges (Step~\ref{four2} of Algorithm~\ref{alg:algorithm}). One can then initialize the EM parameters with our theoretical findings instead of initializing them at random as classically done. However, those theoretical limiting quantities depend on the model parameters such as the class proportions, the eigenvectors of the affinity matrix $\bar{{\bf M}}$ which, except for particular cases ($ K=2 $ classes for example), are not directly accessible from real world graphs. We may empirically estimate those parameters by applying Algorithm~\ref{alg:algorithm} for a fixed $ \alpha $, computing the empirical class-wise means $ \nu_i^a $'s/covariances $ \sigma_{ij}^a $'s and using their theoretical formula (Theorems~\ref{cor:means} and ~\ref{cor:cov}) to deduce the unknown parameters $ {\bf v} $'s (eigenvectors of $\bar{{\bf M}}$) and $ {\bf c} $ (class proportions vector) associated to the graph.

The results and methods in this article are all based on the strong assumption that the class-wise correction factors $ C_{g_ig_j} $ differ by $ \mathcal{O}(n^{-\frac12}) $ since $ \forall i,j \in \{1,\dots,n\} $, $ C_{g_ig_j}=1+\frac{M_{g_ig_j}}{\sqrt{n}}.$ When this condition is not fulfilled, our results may not be applicable. However, this is a more interesting assumption, challenging in practice since when the $C_{ab} $'s differ by $ \mathcal{O}(1) $, one can easily achieve a vanishing misclassification error rate asymptotically and most algorithms perform similarly well in detecting the communities. Besides, our analyses assume very large networks ($ n \to \infty $ with convergence rates of the order $ n^{-\frac12} $) and thus the performances of our algorithm on small size graphs can be quite poor as observed in Table~\ref{tbl:benchmarks}.

Our broad study of spectral methods for community detection is so far limited to dense networks. Real world networks being sparse in general, it would be interesting to extend our study to sparse graph models. The Non Backtacking and Bethe Hessian (BH) methods developped under homogeneous graph models, are currently state-of-the-art spectral methods which allow for non-trivial performances of community detection in sparse graphs. For heterogeneous graphs, the ``flow" matrix (the Non Backtracking matrix normalized by the degrees) proposed in \cite{newman2013spectral} is shown to have a much better behavior than the two aforementionned matrices in sparse heterogeneous graphs. This follows the same spirit as what is done in this article where we normalize the modularity matrix by the degrees for dense heterogeneous graphs. The use of the symmetric BH matrix (of size $ n $ the number of nodes) is highly preferred in terms of computational time and memory than the non-symmetric, high-dimensional Non Backtracking matrix (of size $ 2m $ with $ m $ the number of edges). Following then~\cite{newman2013spectral} and our current work, one may wonder about the existence of an equivalent $\alpha$-normalized BH matrix and its associated performances. This is, we believe, a promising avenue of future investigation.

Pushing further the applicability reach of the present study, note that one of the remaining issues of spectral clustering methods especially for large graphs is the expensive computational complexity of the eigenvectors of the large random matrices representing those graphs. This computation burden is reduced when using power methods and is thus less expensive for sparse graphs (with many zeros in the corresponding matrix). This suggests potential computational gains incurred by smartly removing some edges in the graph to make it sparse prior to eigenvectors computation, which will of course reduce the performance to the benefit of the computational cost. Our mathematical framework may allow us to study the tradeoff between complexity/cost and performances of spectral methods for community detection on such subsampled large dense graphs.

\section{Proofs}
\label{proofs}
\subsection*{Preliminaries}
The random matrix under study $ {\bf L}_\alpha=(2m)^{\alpha}\frac{1}{\sqrt{n}}\mathbf{D}^{-\alpha} \left[ {\bf A} - \frac{{\bf d}{\bf d}^{\sf T}}{2m} \right] \mathbf{D}^{-\alpha} $ is not a classically studied matrix in random matrix theory. We will thus first find an approximate tractable random matrix $ \tilde{{\bf L}}_\alpha$ which asymptotically preserves the eigenvalue distribution and the extreme eigenvectors of $ {\bf L}_\alpha $ (Section~\ref{proofApprox}). Then, the empirical distribution of the eigenvalues of $ {\bf L}_\alpha $ is studied in Section~\ref{sec:determinist1} along with the exact localization of those eigenvalues in Section~\ref{app:th2}. Finally, a thorough study of the eigenvectors associated to the aforementioned eigenvalues is investigated in Section~\ref{app:th3}.

We follow here the proof technique of \cite{couillet2016kernel}. In the sequel, we will make some approximations of random variables in the asymptotic regime where $ n \to \infty $. For the sake of random variables comparisons, we give the following stochastic definitions. For $ x\equiv x_n $ a random variable and $ u_n \geq 0 $, we write $ x=\mathcal{O}(u_n) $ if for any $ \eta>0 $ and $ D>0 $, $ n^D\mathbb{P}(x\geq n^{\eta}u_n)\rightarrow 0 $ as $ n\rightarrow \infty $. For $ {\bf v} $ a vector or a diagonal matrix with random entries, $ {\bf v}=\mathcal{O}(u_n) $ means that the maximal entry of $ {\bf v} $ in absolute value is $ \mathcal{O}(u_n) $ in the sense defined previously. When $ {\bf M} $ is a square matrix, $ {\bf M}=\mathcal{O}(u_n) $ means that the operator norm of $ {\bf M} $ is $ \mathcal{O}(u_n) $. For $ {\bf x} $ a vector or a matrix with random entries, $ {\bf x}=o(u_n) $ means that there is $ \kappa>0 $ such that $ {\bf x}=\mathcal{O}(n^{-\kappa}u_n) $.

Most of the proofs here are classical in random matrix theory (see e.g., \cite{baik2006eigenvalues}) but require certain controls inherent to our model. The goal of the article not being an exhaustive development of the proofs techniques, we will admit a number of technical results already studied in the literature. However, we will exhaustively develop the calculus to obtain our final results which are not trivial.

\subsection{Random equivalent for $ {\bf L}_\alpha $}
\label{proofApprox}
The matrix $ {\bf L}_\alpha=({\bf d}^{\sf T}{\bf 1}_n)^{\alpha}\frac{1}{\sqrt{n}}\mathbf{D}^{-\alpha} \left[ {\bf A} - \frac{{\bf d}{\bf d}^{\sf T}}{{\bf d}^{\sf T}{\bf 1}_n} \right] \mathbf{D}^{-\alpha}$ has non independent entries and is not a classical random matrix model. The idea is thus to approximate $ {\bf L}_\alpha $ by a more tractable random matrix model $ \tilde{{\bf L}}_\alpha $ in such a way that they share asymptotically the same set of outlying eigenvalues/eigenvectors which are of interest in our clustering scenario. We recall that the entries $ A_{ij} $ of the adjacency matrix were defined from the DCSBM model as independent Bernoulli random variables with parameter $ q_iq_j\left(1+\frac{M_{g_ig_j}}{\sqrt{n}}\right) $; one may thus write
\[A_{ij}=q_{i}q_{j}+q_{i}q_{j}\frac{M_{g_ig_j}}{\sqrt{n}}+X_{ij}\]
where $ X_{ij} $, $ 1\leq i,j \leq n $, are independent (up to the symmetry) zero mean random variables of variance $ q_{i}q_{j}(1-q_{i}q_{j}) +\mathcal{O}(n^{-\frac12})$, since $ A_{ij} $ has mean $ q_{i}q_{j}+q_{i}q_{j}\frac{M_{g_ig_j}}{\sqrt{n}} $ and variance $ q_{i}q_{j}(1-q_{i}q_{j}) +\mathcal{O}(n^{-\frac12}) $. We can then write the normalized adjacency matrix as follows 
\begin{align}
\label{eq:A2}
	\frac1{\sqrt{n}} \mathbf{A} &= \frac{1}{\sqrt{n}}\mathbf{q}\mathbf{q}^{\sf T}+\frac{1}{n} \left\{ \mathbf{q}_{(a)}\mathbf{q}_{(b)}^{\sf T}M_{ab} \right\}_{a,b=1}^{K}+ \frac{1}{\sqrt{n}}\mathbf{X} \\
\label{eq:A3} &=\underbrace{\frac{{\bf qq}^{\sf T}}{\sqrt{n}}}_{\mathbf{A}_{d,\sqrt{n}}}+\underbrace{\frac{1}{n}{\bf D}_q{\bf JMJ}^{\sf T}{\bf D}_q}_{\mathbf{A}_{d,1}}+\underbrace{\frac{{\bf X}}{\sqrt{n}}}_{\mathbf{A}_{r,1}},
 \end{align}
 where\footnote{We recall that subscript `$d,n^k$' stands for deterministic term whose operator norm is of order $n^k$ and `$r,n^k$' for random term with operator norm of order $n^k$.} $ \mathbf{q}_{(i)}=[ q_{n_{1}+\ldots+n_{i-1}+1},\ldots,q_{n_{1}+\ldots+n_{i}}]^{\sf T}\in \mathbb{R}^{n_i}$ ($n_{0}=0$) , $\mathbf{X}=\left\{X_{ij}\right\}_{i,j=1}^n$ and $ {\bf D}_q=\mathcal{D}({\bf q}) $. The idea of the proof is to write all the terms of $ {\bf L}_\alpha $ based on Equation~\eqref{eq:A3}, since all those terms depend on $ {\bf A} $. To this end, we will evaluate successively $ {\bf d}={\bf A1}_n $, $ {\bf D}=\mathcal{D}({\bf d}) $, ${\bf d}{\bf d}^{\sf T} $ and $2m={\bf d}^{\sf T}{\bf 1}_n$. It will appear that $ {\bf D} $ and $ {\bf d}^{\sf T}{\bf 1}_n $ are composed of dominant terms (with higher operator norm) and vanishing terms (with smaller operator norm); we may then proceed to writing a Taylor expansion of $ {\bf D}^{-\alpha} $ and $ (2m)^{\alpha}=({\bf d}^{\sf T}{\bf 1}_n)^{\alpha} $ for any $ \alpha $ around their dominant terms to finally retrieve a Taylor expansion of ${\bf L}_\alpha$.
 
 Let us start by developing the degree vector $ {\bf d}={\bf A1}_n $. We have 
 \begin{equation} \label{eq:d} {\bf d}={\bf qq}^{\sf T}{\bf 1}_n + \frac{1}{\sqrt{n}}{\bf D}_q{\bf JMJ}^{\sf T}{\bf D}_q{\bf 1}_n+{\bf X1}_n={\bf q}^{\sf T}{\bf 1}_n\Big(\underbrace{{\bf q}}_{\mathcal{O}(n^{\frac12})}+\underbrace{\frac{1}{\sqrt{n}}\frac{{\bf D}_q{\bf JMJ}^{\sf T}{\bf D}_q{\bf 1}_n}{{\bf q}^{\sf T}{\bf 1}_n}}_{\mathcal{O}(n^{-\frac12})}+\underbrace{\frac{{\bf X1}_n}{{\bf q}^{\sf T}{\bf 1}_n}}_{\mathcal{O}(n^{-\frac12})}\Big). 
 \end{equation}
 Let us then write the expansions of $ {\bf d}^{\sf T}{\bf 1}_n $, $ ({\bf d}^{\sf T}{\bf 1}_n)^{\alpha} $, $ {\bf d}{\bf d}^{\sf T} $ and $ \frac{{\bf d}{\bf d}^{\sf T}}{({\bf d}^{\sf T}{\bf 1}_n)} $ respectively. From~\eqref{eq:d}, we obtain
 \begin{equation} \label{eq:dTransposeOne} {\bf d}^{\sf T}{\bf 1}_n= ({\bf q}^{\sf T}{\bf 1}_n)^{2}\Big[1+\underbrace{\frac{1}{\sqrt{n}}\frac{{\bf 1}_n^{\sf T}{\bf D}_q{\bf JMJ}^{\sf T}{\bf D}_q{\bf 1}_n}{({\bf q}^{\sf T}{\bf 1}_n)^{2}}}_{\mathcal{O}(n^{-\frac12})}+\underbrace{\frac{{\bf 1}_n^{\sf T}{\bf X1}_n}{({\bf q}^{\sf T}{\bf 1}_n)^{2}}}_{\mathcal{O}(n^{-\frac12})}\Big].\end{equation}
 Thus for any $ \alpha $, proceeding to a $ 1^{st} $ order Taylor expansion, we may write \begin{equation}
 \label{eq:dTranspose1Alpha}
 ({\bf d}^{\sf T}{\bf 1}_n)^{\alpha}=({\bf q}^{\sf T}{\bf 1}_n)^{2\alpha}\Big[1+\frac{\alpha}{\sqrt{n}} \underbrace{\frac{{\bf 1}_n^{\sf T}{\bf D}_q{\bf JMJ}^{\sf T}{\bf D}_q{\bf 1}_n}{({\bf q}^{\sf T}{\bf 1}_n)^{2}}}_{\mathcal{O}(n^{-\frac12})}+\alpha \underbrace{\frac{{\bf 1}_n^{\sf T}{\bf X1}_n}{({\bf q}^{\sf T}{\bf 1}_n)^{2}}}_{\mathcal{O}(n^{-\frac12})}+o(n^{-\frac12})\Big].
\end{equation}
Besides, from~\eqref{eq:d} we have 
 \begin{align}
 \label{eq:ddTranspose}
 {\bf d}{\bf d}^{\sf T}&=({\bf q}^{\sf T}{\bf 1}_n)^{2} \Big[\underbrace{{\bf qq}^{\sf T}}_{\mathcal{O}(n)}+\underbrace{\frac{1}{\sqrt{n}}\frac{{\bf q}{\bf 1}_n^{\sf T}{\bf D}_q{\bf JMJ}^{\sf T}{\bf D}_q}{{\bf q}^{\sf T}{\bf 1}_n}}_{\mathcal{O}(\sqrt{n})}+\underbrace{\frac{1}{\sqrt{n}}\frac{{\bf D}_q{\bf JMJ}^{\sf T}{\bf D}_q{\bf 1}_n{\bf q}^{\sf T}}{{\bf q}^{\sf T}{\bf 1}_n}}_{\mathcal{O}(\sqrt{n})}+\underbrace{\frac{{\bf q1}_n^{\sf T}{\bf X}}{{\bf q}^{\sf T}{\bf 1}_n}}_{\mathcal{O}(\sqrt{n})}+\underbrace{\frac{{\bf X1}_n{\bf q}^{\sf T}}{{\bf q}^{\sf T}{\bf 1}_n}}_{\mathcal{O}(\sqrt{n})} \nonumber \\
 &+ \underbrace{\frac{1}{n}\frac{{\bf D}_q{\bf JMJ}^{\sf T}{\bf D}_q{\bf 1}_n{\bf 1}_n^{\sf T}{\bf D}_q{\bf JMJ}^{\sf T}{\bf D}_q}{({\bf q}^{\sf T}{\bf 1}_n)^{2}}}_{\mathcal{O}(1)}+\underbrace{\frac{1}{\sqrt{n}}\frac{{\bf D}_q{\bf JMJ}^{\sf T}{\bf D}_q{\bf 1}_n{\bf 1}_n^{\sf T}{\bf X}}{({\bf q}^{\sf T}{\bf 1}_n)^{2}}}_{\mathcal{O}(1)} +\underbrace{\frac{1}{\sqrt{n}}\frac{{\bf X}{\bf 1}_n{\bf 1}_n^{\sf T}{\bf D}_q{\bf JMJ}^{\sf T}{\bf D}_q}{({\bf q}^{\sf T}{\bf 1}_n)^{2}}}_{\mathcal{O}(1)} \nonumber  \\ 
 &+\underbrace{\frac{{\bf X}{\bf 1}_n{\bf 1}_n^{\sf T}{\bf X}}{({\bf q}^{\sf T}{\bf 1}_n)^{2}}}_{\mathcal{O}(1)} +o(1)\Big].
 \end{align}
Keeping in mind that we shall only need terms with non vanishing operator norms asymptotically, we will require $\frac{1}{\sqrt{n}}\left[ {\bf A} - \frac{{\bf d}{\bf d}^{\sf T}}{{\bf d}^{\sf T}{\bf 1}_n} \right]$ to have terms with spectral norms of order at least $ \mathcal{O}(1) $. We get from multiplying \eqref{eq:ddTranspose} and \eqref{eq:dTranspose1Alpha} (with $ \alpha=-1 $)
 \begin{align}
 \label{eq:16}
 \frac{1}{\sqrt{n}}\frac{{\bf d}{\bf d}^{\sf T}}{{\bf d}^{\sf T}{\bf 1}_n}&=\frac{{\bf qq}^{\sf T}}{\sqrt{n}}+\frac{1}{n}\frac{{\bf q}{\bf 1}_n^{\sf T}{\bf D}_q{\bf JMJ}^{\sf T}{\bf D}_q}{{\bf q}^{\sf T}{\bf 1}_n}+\frac{1}{n}\frac{{\bf D}_q{\bf JMJ}^{\sf T}{\bf D}_q{\bf 1}_n{\bf q}^{\sf T}}{{\bf q}^{\sf T}{\bf 1}_n}+\frac{1}{\sqrt{n}}\frac{{\bf q1}_n^{\sf T}{\bf X}}{{\bf q}^{\sf T}{\bf 1}_n}+\frac{1}{\sqrt{n}}\frac{{\bf X1}_n{\bf q}^{\sf T}}{{\bf q}^{\sf T}{\bf 1}_n} \nonumber \\
 &-\frac{1}{n}\frac{{\bf 1}_n^{\sf T}{\bf D}_q{\bf JMJ}^{\sf T}{\bf D}_q{\bf 1}_n}{({\bf q}^{\sf T}{\bf 1}_n)^{2}}{\bf qq}^{\sf T}-\frac{1}{\sqrt{n}}\frac{{\bf 1}_n^{\sf T}{\bf X1}_n}{({\bf q}^{\sf T}{\bf 1}_n)^{2}}{\bf qq}^{\sf T}+\mathcal{O}(n^{-\frac12}).
 \end{align}
 By subtracting \eqref{eq:16} from \eqref{eq:A3}, we obtain 
  \begin{align}
   \label{eq:substA}
\frac{1}{\sqrt{n}} \left({\bf A}-\frac{{\bf d}{\bf d}^{\sf T}}{{\bf d}^{\sf T}{\bf 1}_n}\right)&=\frac{1}{n}{\bf D}_q{\bf JMJ}^{\sf T}{\bf D}_q-\frac{1}{n}\frac{{\bf q}{\bf 1}_n^{\sf T}{\bf D}_q{\bf JMJ}^{\sf T}{\bf D}_q}{{\bf q}^{\sf T}{\bf 1}_n}-\frac{1}{n}\frac{{\bf D}_q{\bf JMJ}^{\sf T}{\bf D}_q{\bf 1}_n{\bf q}^{\sf T}}{{\bf q}^{\sf T}{\bf 1}_n} \nonumber \\
&+\frac{1}{n}\frac{{\bf 1}_n^{\sf T}{\bf D}_q{\bf JMJ}^{\sf T}{\bf D}_q{\bf 1}_n}{({\bf q}^{\sf T}{\bf 1}_n)^{2}}{\bf qq}^{\sf T}+\frac{{\bf X}}{\sqrt{n}}-\frac{1}{\sqrt{n}}\frac{{\bf q1}_n^{\sf T}{\bf X}}{{\bf q}^{\sf T}{\bf 1}_n}-\frac{1}{\sqrt{n}}\frac{{\bf X1}_n{\bf q}^{\sf T}}{{\bf q}^{\sf T}{\bf 1}_n} \nonumber \\
&+\frac{1}{\sqrt{n}}\frac{{\bf 1}_n^{\sf T}{\bf X1}_n}{({\bf q}^{\sf T}{\bf 1}_n)^{2}}{\bf qq}^{\sf T}+\mathcal{O}(n^{-\frac12}).
 \end{align}
 It then remains to evaluate $ {\bf D}^{-\alpha} $. From \eqref{eq:d}, we may write $ {\bf D}=\mathcal{D}({\bf d}) $ as 
 \[{\bf D}={\bf q}^{\sf T}{\bf 1}_n\Big(\underbrace{{\bf D}_q}_{\mathcal{O}(1)}+\underbrace{\mathcal{D}\left(\frac{1}{\sqrt{n}}\frac{{\bf D}_q{\bf JMJ}^{\sf T}{\bf D}_q{\bf 1}_n}{{\bf q}^{\sf T}{\bf 1}_n}\right)}_{\mathcal{O}(n^{-\frac12})}+\underbrace{\mathcal{D}\left(\frac{{\bf X1}_n}{{\bf q}^{\sf T}{\bf 1}_n}\right)}_{\mathcal{O}(n^{-\frac12})}\Big).\]
The right hand side of ${\bf D}$ (in brackets) having a leading term in $ \mathcal{O}(1) $ and residual terms in $ \mathcal{O}(n^{-\frac12})$, the Taylor expansion of the $ (-\alpha) $-power of $ {\bf D}$ is then retrieved
 \begin{equation}\label{eq:DAlpha} {\bf D}^{-\alpha}=\left({\bf q}^{\sf T}{\bf 1}_n\right)^{-\alpha} \Big(\underbrace{{\bf D}_q}_{\mathcal{O}(1)}-\alpha \underbrace{\mathcal{D}\left(\frac{1}{\sqrt{n}}\frac{{\bf D}_q{\bf JMJ}^{\sf T}{\bf D}_q{\bf 1}_n}{{\bf q}^{\sf T}{\bf 1}_n}\right)}_{\mathcal{O}(n^{-\frac12})}-\alpha \underbrace{\mathcal{D}\left(\frac{{\bf X1}_n}{{\bf q}^{\sf T}{\bf 1}_n}\right)}_{\mathcal{O}(n^{-\frac12})}+\mathcal{O}(n^{-1})\Big). \end{equation}

 By combining the expressions \eqref{eq:dTranspose1Alpha}, \eqref{eq:substA} and \eqref{eq:DAlpha}, we obtain a Taylor approximation of $ {\bf L}_\alpha $ as follows
 \begin{align*} 
 {\bf L}_\alpha&={\bf D}_q^{-\alpha}\frac{{\bf X}}{\sqrt{n}}{\bf D}_q^{-\alpha}+\frac{1}{n}{\bf D}_q^{1-\alpha}{\bf JMJ}^{\sf T}{\bf D}_q^{1-\alpha}-\frac{1}{n}\frac{{\bf D}_q^{1-\alpha}{\bf 1}_n{\bf 1}_n^{\sf T}{\bf D}_q{\bf JMJ}^{\sf T}{\bf D}_q^{1-\alpha}}{{\bf q}^{\sf T}{\bf 1}_n} \nonumber \\ 
 &-\frac{1}{n}\frac{{\bf D}_q^{1-\alpha}{\bf JMJ}^{\sf T}{\bf D}_q{\bf 1}_n{\bf 1}_n^{\sf T}{\bf D}_q^{1-\alpha}}{{\bf q}^{\sf T}{\bf 1}_n} +\frac{1}{n}\frac{{\bf 1}_n^{\sf T}{\bf D}_q{\bf JMJ}^{\sf T}{\bf D}_q{\bf 1}_n}{({\bf q}^{\sf T}{\bf 1}_n)^{2}}{\bf D}_q^{1-\alpha}{\bf 1}_n{\bf 1}_n^{\sf T}{\bf D}_q^{1-\alpha}-\frac{1}{\sqrt{n}}\frac{{\bf D}_q^{1-\alpha}{\bf 1}_n{\bf 1}_n^{\sf T}{\bf X}{\bf D}_q^{-\alpha}}{{\bf q}^{\sf T}{\bf 1}_n} \nonumber \\ 
 &-\frac{1}{\sqrt{n}}\frac{{\bf D}_q^{-\alpha}{\bf X1}_n{\bf 1}_n^{\sf T}{\bf D}_q^{1-\alpha}}{{\bf q}^{\sf T}{\bf 1}_n}+\frac{1}{\sqrt{n}}\frac{{\bf 1}_n^{\sf T}{\bf X1}_n}{({\bf q}^{\sf T}{\bf 1}_n)^{2}}{\bf D}_q^{1-\alpha}{\bf 1}_n{\bf 1}_n^{\sf T}{\bf D}_q^{1-\alpha}+\mathcal{O}(n^{-\frac12}).
 \end{align*}
 The three following arguments allow to complete the proof
 \begin{itemize}
 \item $ {\bf 1}_n={\bf J1}_K $ and ${\bf D}_{q}{\bf 1}_n={\bf q}. $
 \item We may write $ (\frac{1}{n}{\bf J}^{\sf T}{\bf q})_i=\frac{n_i}{n} \left(\frac{1}{n_i}\sum_{a \in \mathcal{C}_i}q_a\right) $. For classes of large sizes $ n_i $, from the law of large numbers,  $\left(\frac{1}{n_i}\sum_{a \in \mathcal{C}_i}q_a\right) \asto m_{\mu}$ and so, $\frac{1}{n}{\bf J}^{\sf T}{\bf q} \asto m_{\mu}{\bf c}$ where we recall that $ m_\mu=\int t\mu(\mathrm{d}t) $.
 \item As $ {\bf X} $ is a symmetric random matrix having independent entries of zero mean and finite variance, from the law of large numbers, we have $\frac{1}{n}\frac{{\bf 1}_n^{\sf T}{\bf X1}_n}{\sqrt{n}} \asto 0$.
 \end{itemize}
 Using those three arguments, $ {\bf L}_\alpha $ may be further rewritten
 \begin{align}
 \label{eq:LAlpha}
  {\bf L}_\alpha&={\bf D}_q^{-\alpha}\frac{{\bf X}}{\sqrt{n}}{\bf D}_q^{-\alpha}+\frac{1}{n}{\bf D}_q^{1-\alpha}{\bf JMJ}^{\sf T}{\bf D}_q^{1-\alpha}-\frac{1}{n}{\bf D}_q^{1-\alpha}{\bf J1}_K{\bf c}^{\sf T}{\bf MJ}^{\sf T}{\bf D}_q^{1-\alpha} \nonumber \\ 
  &-\frac{1}{n}{\bf D}_q^{1-\alpha}{\bf JMc1}_K^{\sf T}{\bf J}^{\sf T}{\bf D}_q^{1-\alpha}+\frac{1}{n}{\bf D}_q^{1-\alpha}{\bf J1}_K{\bf c}^{\sf T}{\bf Mc1}_K^{\sf T}{\bf J}^{\sf T}{\bf D}_q^{1-\alpha} \nonumber \\ 
  &-\frac{1}{\sqrt{n}{\bf q}^{\sf T}{\bf 1}_n}{\bf D}_q^{1-\alpha}{\bf J1}_K{\bf 1}_n^{\sf T}{\bf X}{\bf D}_q^{-\alpha} -\frac{1}{\sqrt{n}{\bf q}^{\sf T}{\bf 1}_n}{\bf D}_q^{-\alpha}{\bf X1}_n{\bf 1}_K^{\sf T}{\bf J}^{\sf T}{\bf D}_q^{1-\alpha}+\mathcal{O}(n^{-\frac12}).
  \end{align}
  By rearranging the terms of \eqref{eq:LAlpha}, we obtain the expected result
  \begin{align*}
  {\bf L}_\alpha&={\bf D}_q^{-\alpha}\frac{{\bf X}}{\sqrt{n}}{\bf D}_q^{-\alpha} \\
  &+\begin{bmatrix}
	\frac{\mathbf{D}_{q}^{1-\alpha}{\bf J}}{\sqrt{n}} & \frac{\mathbf{D}_{q}^{-\alpha}{\bf X}{\bf 1}_{n}}{{\bf q}^{\sf T}{\bf 1}_{n}}
	\end{bmatrix} \begin{bmatrix}
	\left({\bf I}_{K}-{\bf 1}_{K}{\bf c}^{T}\right){{\bf M}}\left({\bf I}_{K}-{\bf c}{\bf 1}_{K}^{T}\right) & -{\bf 1}_{K} \\
	-{\bf 1}_{K}^{T} & 0
	\end{bmatrix} \begin{bmatrix}
	\frac{{\bf J}^{\sf T}\mathbf{D}_{q}^{1-\alpha}}{\sqrt{n}} \\ \frac{{\bf 1}_{n}^{\sf T}{\bf X}\mathbf{D}_{q}^{-\alpha}}{{\bf q}^{\sf T}{\bf 1}_{n}}
	\end{bmatrix} +\mathcal{O}(n^{-\frac12}).
  \end{align*}
  This proves Theorem~\ref{approx}.
  \subsection{Limiting spectral distribution of $ {\bf L}_\alpha $}
  \label{sec:determinist1}
  It follows from Theorem~\ref{approx} that $ \tilde{{\bf L}}_\alpha={\bf D}_q^{-\alpha}\frac{{\bf X}}{\sqrt{n}}{\bf D}_q^{-\alpha}+{\bf U}{\bm \Lambda}{\bf U}^{\sf T}$ is equivalent to an additive spiked random matrix \cite{chapon2012outliers} where
\begin{align*}
\mathbf{U} &= \begin{bmatrix}
	\frac{\mathbf{D}_{q}^{1-\alpha}{\bf J}}{\sqrt{n}} & \frac{\mathbf{D}_{q}^{-\alpha}{\bf X}{\bf 1}_{n}}{{\bf q}^{\sf T}{\bf 1}_{n}}
	\end{bmatrix}, \\
	{\bm \Lambda} &= \begin{bmatrix}
	\left({\bf I}_{K}-{\bf 1}_{K}{\bf c}^{T}\right){{\bf M}}\left({\bf I}_{K}-{\bf c}{\bf 1}_{K}^{T}\right) & -{\bf 1}_{K} \\
	-{\bf 1}_{K}^{T} & 0
	\end{bmatrix},
\end{align*} with the difference that the deterministic part $ {\bf U}{\bm \Lambda}{\bf U}^{\sf T} $ is not independent of the random part $ {\bf D}_q^{-\alpha}\frac{{\bf X}}{\sqrt{n}}{\bf D}_q^{-\alpha} $ (an issue that we solve here) and $ {\bf U} $ is not composed of orthonormal vectors. Let us then study $\bar{{\bf X}}={\bf D}_q^{-\alpha}\frac{{\bf X}}{\sqrt{n}}{\bf D}_q^{-\alpha}$ (having entries $ \bar{X}_{ij} $ with zero mean and variance $ \sigma_{ij}^{2}/n $ with $\sigma_{ij}^{2}=q_iq_j(1-q_iq_j)+\mathcal{O}(n^{-\frac12}) $) and show that its e.s.d.\@ $ \tilde{\pi}^{\alpha} $ converges weakly to $ \bar{\pi}^{\alpha} $ with Stieljes transform $ E_0^\alpha(z)=\int \left(t-z\right)^{-1}\mathrm{d}\bar{\pi}^{\alpha}(t) $ for $ z \in \mathbb{C}^{+} $ defined in Theorem~\ref{determinst2}. This will imply (By Weyl interlacing formula) that the empirical spectral measure $ \pi^{\alpha}\equiv \frac{1}{n}\sum_{i=1}^{n}\delta_{\lambda_{i}(\tilde{{\bf L}}_\alpha)} $ (with $ \lambda_{i}(\tilde{{\bf L}}_\alpha) $ eigenvalues of $ \tilde{{\bf L}}_\alpha $) will also converge to $ \bar{\pi}^{\alpha}.$ 

The matrix $ \bar{{\bf X}} $ is a classical random matrix model in RMT already studied in similar cases. It is well known for those random matrix models (having entries with given means, variances and bounded first order moments) that the law of the $ \bar{X}_{ij} $'s does not change the results on the limiting law of the e.s.d.\@ $ \tilde{\pi}^{\alpha} $: this property is kwown as \emph{universality} (e.g., \cite{silverstein1995empirical}). For technical reasons, we can thus assume that the $ \bar{X}_{ij} $'s are Gaussian random variables with the same means and variances in order to use standard Gaussian calculus, introduced in \cite{pastur2011eigenvalue}. The objective of the proof is to find the deterministic limit $ E_0^\alpha(z)$ for the random quantity $\frac{1}{n}\tr \left(\bar{{\bf X}}-z{\bf I}_n\right)^{-1}$ which is the Stieljes transform of the e.s.d.\@ $ \tilde{\pi}^{\alpha} $. Deterministic equivalents for the Stieljes transform of empirical spectral measures associated with centered and symmetric random matrix models with a variance profile have already been studied in for example \cite{ajanki2015quadratic,hachem2007deterministic}. We give in Appendix~\ref{apendA} an exhaustive development of the Gaussian calculus to obtain $ E_0^\alpha(z)$. The final result is as follows.
\begin{lemma}[A first deterministic equivalent]
\label{lemma3}
Let $ \mathbf{Q}=(\bar{{\bf X}}-z\mathbf{I}_n)^{-1}$. Then, for all $ z \in \mathbb{C}^{+}$, 
\begin{equation}
\label{AlgoEi}
{\bf Q} \leftrightarrow \bar{{\bf Q}}=\left(-z{\bf I}_n-\mathcal{D}\left(e_i(z)\right)_{i=1}^{n}\right)^{-1}
\end{equation}
where $ e_i(z)$ the unique solution of $e_i(z) =\frac{1}{n}\tr \mathcal{D}\left(\sigma_{ij}^{2}\right)_{j=1}^{n}\left(-z{\bf I}_n-\mathcal{D}\left(e_j(z)\right)_{j=1}^{n}\right)^{-1}$ and the notation $\mathbf{A}\leftrightarrow \mathbf{B}$ stands for $\frac{1}{n} \tr \mathbf{C}\mathbf{A}- \frac{1}{n} \tr \mathbf{CB} \to 0$ and $ \mathbf{d}_1^{\sf T}(\mathbf{A}-\mathbf{B})\mathbf{d}_{2} \to 0 $ almost surely, for all deterministic Hermitian matrix $ \mathbf{C} $ and deterministic vectors $\mathbf{d}_{i}$ of bounded norms (spectral norm for matrices and Euclidian norm for vectors).  
\end{lemma}
From Lemma~\ref{lemma3}, we get directly $ \frac{1}{n}\tr{\bf Q}-E_0^\alpha(z) \asto 0$ with $E_0^\alpha(z)=\frac{1}{n}\sum_{i=1}^n \frac{1}{-z-e_i(z)}.$
Observe now that
\begin{align}
\label{Ei}
e_i(z)&=\frac{1}{n} \sum_{j=1}^{n} \frac{q_i^{1-2\alpha}q_j^{1-2\alpha}-q_i^{2-2\alpha}q_j^{2-2\alpha}}{-z-e_j(z)} \nonumber\\
&=q_i^{1-2\alpha}e_{11}^{\alpha}(z)-q_i^{2-2\alpha}e_{21}^{\alpha}(z)
\end{align}
where 
\begin{align}
\label{eq:systStielj}
e_{11}^{\alpha}(z)&=\frac{1}{n}\sum_{j=1}^{n}\frac{q_j^{1-2\alpha}}{-z-q_j^{1-2\alpha}e_{11}^{\alpha}(z)+q_j^{2-2\alpha}e_{21}^{\alpha}(z)} \nonumber \\
e_{21}^{\alpha}(z)&=\frac{1}{n}\sum_{j=1}^{n}\frac{q_j^{2-2\alpha}}{-z-q_j^{1-2\alpha}e_{11}^{\alpha}(z)+q_j^{2-2\alpha}e_{21}^{\alpha}(z)}
\end{align}
from which we get $ E_0^\alpha(z)=e_{00}^{\alpha}(z) $ where
\begin{equation*}
\label{eq:StieltjesE1}
e_{00}^{\alpha}(z)=\int \frac{1}{-z-e_{11}^\alpha(z)q^{1-2\alpha}+e_{21}^\alpha(z)q^{2-2\alpha}}\mu(dq).
\end{equation*}
From this, we have that $ E_0^\alpha(z) $ does not depend on $ n $, so that $ \frac{1}{n}\tr{\bf Q} \asto E_0^\alpha(z)$, $ \tilde{\pi}^{\alpha} \to \bar{\pi}^{\alpha} $, and thus $ \pi^{\alpha} \to \bar{\pi}^{\alpha} $ since $ \tilde{{\bf L}}_\alpha $ and $ \bar{{\bf X}} $ only differ by a finite rank matrix.
\begin{remark}[Convergence of the $ e_i $'s.]
Similar results to Lemma~\ref{lemma3} have been derived for example in \cite{hachem2007deterministic} and the fixed point algorithm~\eqref{AlgoEi} which consists of iterating the $ e_i $'s is shown to converge. Since the calcultation of the $ e_{ab} $'s is an intermediary step of \eqref{AlgoEi} from~\eqref{Ei}, the fixed point algorithm~\eqref{eq:systStielj} also converges. Similarly to \cite{hachem2007deterministic}, when none of the $ ({\bf D}_q^{-\alpha})_{ii} $'s is isolated, the random matrix $\bar{{\bf X}}$ does not produce isolated eigenvalues outside the support $ \mathcal{S}^{\alpha} $ of $ \bar{\pi}^{\alpha} $. Here, for large $ n $, this property is verified since from Assumption~\ref{as1}, the $ q_i $'s are i.i.d.\@ arising from a law with compact support (the probability that a $ ({\bf D}_q^{-\alpha})_{ii} $ gets isolated tends to $ 0 $ asymptotically). This gives Proposition~\ref{prop1} which we will not prove here; similar proofs are provided for example in~\cite{bai1998no}. From the analyticity of the Stieljes transform outside its support, Lemma~\ref{lemma3} extends naturally to $ \mathbb{C} \setminus \mathcal{S}^{\alpha}$. This proves Theorem~\ref{determinst2}.
\end{remark}
\subsection{Isolated eigenvalues of $ {\bf L}_\alpha $}
\label{app:th2}
In the previous section, we have shown that the e.s.d.\@ of $ {\bf L}_\alpha $ converges weakly to the limiting law of the eigenvalues of $\bar{{\bf X}}$ since they only differ by a finite rank matrix. We shall have in addition isolated eigenvalues of $ {\bf L}_\alpha $ induced by the aforementionned low rank matrix. We are interested here in the localization of eigenvalues of $ {\bf L}_\alpha $ isolated from the support $ \mathcal{S}^\alpha $ of the limiting law of its e.s.d.\@ According to Proposition~\ref{prop1}, there is almost surely no eigenvalue of $\bar{{\bf X}}$ at non-vanishing distance from $  \mathcal{S}^{\alpha} $ asymptotically as $ n\to \infty $ and hence the plausible isolated eigenvalues of $ {\bf L}_\alpha $ are only due to the matrix $ {\bf U}{\bm \Lambda}{\bf U}^{\sf T} $. We follow classical random matrix approaches used for the study of the spectrum of spiked random matrices \cite{benaych2012singular,chapon2012outliers}. From Theorem~\ref{approx}, the eigenvalues $ \lambda $ of $ {\bf L}_\alpha $ falling at non-vanishing distance from the limiting support $ \mathcal{S}^\alpha $ solve for large $ n $, $0=\det({\bf L}_\alpha-\lambda{\bf I}_n)$ almost surely for $ \lambda \notin \mathcal{S}^{\alpha} $. Since $ \|{\bf L}_\alpha-\tilde{{\bf L}}_\alpha\| \asto 0 $, $ \lambda_i({\bf L}_\alpha)-\lambda_i(\tilde{{\bf L}}_\alpha) \asto 0 $ for all eigenvalues $ \lambda_i({\bf L}_\alpha) $. We may then just solve $0=\det({\bf D}_q^{-\alpha}\frac{{\bf X}}{\sqrt{n}}{\bf D}_q^{-\alpha}+{\bf U}{\bm \Lambda}{\bf U}^{\sf T}-\lambda{\bf I}_n)$. Now, as from Proposition~\ref{prop1}, the random matrix $\bar{{\bf X}}$ does not have eigenvalues at non-vanishing distance from $ \mathcal{S}^\alpha $ asymptotically, for $ \lambda \notin \mathcal{S}^{\alpha} $, we can thus factor and cancel out $ \det(\bar{{\bf X}}-\lambda{\bf I}_n) $ from the previous determinant equation, so that we are left to solve \[ 0=\det({\bf I}_n+{\bf Q}_\lambda^{\alpha}{\bf U}{\bm \Lambda}{\bf U}^{\sf T})=\det({\bf I}_{K+1}+{\bf U}^{\sf T}{\bf Q}_\lambda^{\alpha}{\bf U}{\bm \Lambda}) \] where $ {\bf Q}_\lambda^{\alpha}=(\bar{{\bf X}}-\lambda{\bf I}_n)^{-1} $. As we will show next, the matrix $ {\bf I}_{K+1}+{\bf U}^{\sf T}{\bf Q}_\lambda^{\alpha}{\bf U}{\bm \Lambda} $ converges to a deterministic matrix, almost surely for large $ n $. By the argument principle (similar to e.g., \cite{chapon2012outliers}), the roots of $ {\bf I}_{K+1}+{\bf U}^{\sf T}{\bf Q}_\lambda^{\alpha}{\bf U}{\bm \Lambda} $ are asymptotically those of the limiting matrix, with same multiplicity and it suffices to study the latter.

	We then proceed to retrieving a limit for $ {\bf I}_{K+1}+{\bf U}^{\sf T}{\bf Q}_\lambda^{\alpha}{\bf U}{\bm \Lambda} $. From Theorem~\ref{approx}, we have 
\[ 
{\bf U}^{\sf T}{\bf Q}_\lambda^{\alpha}{\bf U} = \begin{pmatrix}
\frac{1}{n}{\bf J}^{\sf T}{\bf D}_q^{1-\alpha}{\bf Q}_\lambda^{\alpha}{\bf D}_q^{1-\alpha}{\bf J} & \frac{1}{\sqrt{n}({\bf q}^{\sf T}{\bf 1}_n)}{\bf J}^{\sf T}{\bf D}_q^{1-\alpha}{\bf Q}_\lambda^{\alpha}{\bf D}_q^{-\alpha}{\bf X1}_n \\
\frac{1}{\sqrt{n}({\bf q}^{\sf T}{\bf 1}_n)}{\bf 1}_n^{\sf T}{\bf X}{\bf D}_q^{-\alpha}{\bf Q}_\lambda^{\alpha}{\bf D}_q^{1-\alpha}{\bf J} & \frac{1}{({\bf q}^{\sf T}{\bf 1}_n)^{2}}{\bf 1}_n^{\sf T}{\bf X}{\bf D}_q^{-\alpha}{\bf Q}_\lambda^{\alpha}{\bf D}_q^{-\alpha}{\bf X1}_n
\end{pmatrix}.
\]
The entries $ (1,2) $, $ (2,1) $ and $ (2,2) $ of $ {\bf U}^{\sf T}{\bf Q}_\lambda^{\alpha}{\bf U} $ are random as they contain the random matrix $ {\bf X} $ but tend to be deterministic in the limit. In fact, using the resolvent identity, we have that $ {\bf Q}_\lambda^{\alpha}{\bf D}_q^{-\alpha}\frac{{\bf X}}{\sqrt{n}}{\bf D}_q^{-\alpha}={\bf I}_n+\lambda{\bf Q}_\lambda^{\alpha} $, the entry $ (1,2) $ becomes $ \frac{1}{({\bf q}^{\sf T}{\bf 1}_n)}{\bf J}^{\sf T}{\bf D}_q{\bf 1}_n+\lambda\frac{1}{\sqrt{n}({\bf q}^{\sf T}{\bf 1}_n)}{\bf J}^{\sf T}{\bf D}_q^{1-\alpha}{\bf Q}_\lambda^{\alpha}{\bf D}_q^{\alpha}{\bf 1}_n$ and the entry $ (2,2) $ is equal to $ \frac{n}{({\bf q}^{\sf T}{\bf 1}_n)^{2}}\Big({\bf 1}_n^{\sf T}{\bf X}{\bf 1}_n+\lambda{\bf 1}_n^{\sf T}{\bf D}_q^{2\alpha}{\bf 1}_n+\lambda^2{\bf 1}_n^{\sf T}{\bf D}_q^{\alpha}{\bf Q}_\lambda^{\alpha}{\bf D}_q^{\alpha}{\bf 1}_n\Big) $. Now, we can freely use Lemma~\ref{lemma3} to evaluate the limits of the entries of $ {\bf U}^{\sf T}{\bf Q}_\lambda^{\alpha}{\bf U} $ since all the terms are of the form $ {\bf a}^{\sf T}{\bf Q}_\lambda^{\alpha}{\bf b} $ with $ {\bf a} $ and $ {\bf b} $ deterministic vectors. From Lemma~\ref{lemma3}, the entries $ (1,1) $, $ (1,2) $ and $ (2,2) $ converge almost surely respectively to $ \frac{1}{n}{\bf J}^{\sf T}{\bf D}_q^{1-\alpha}\bar{{\bf Q}}_\lambda^{\alpha}{\bf D}_q^{1-\alpha}{\bf J} $, $ \frac{1}{({\bf q}^{\sf T}{\bf 1}_n)}{\bf J}^{\sf T}{\bf D}_q{\bf 1}_n+\lambda\frac{1}{({\bf q}^{\sf T}{\bf 1}_n)}{\bf J}^{\sf T}{\bf D}_q^{1-\alpha}\bar{{\bf Q}}_\lambda^{\alpha}{\bf D}_q^{\alpha}{\bf 1}_n $ and $ \frac{n}{({\bf q}^{\sf T}{\bf 1}_n)^{2}}\Big({\bf 1}_n^{\sf T}{\bf X}{\bf 1}_n+\lambda{\bf 1}_n^{\sf T}{\bf D}_q^{2\alpha}{\bf 1}_n+\lambda^2{\bf 1}_n^{\sf T}{\bf D}_q^{\alpha}\bar{{\bf Q}}_\lambda^{\alpha}{\bf D}_q^{\alpha}{\bf 1}_n\Big) $ for large $n.$

Now, using the fact that for any bounded continuous function $ f $, from the law of large numbers, \begin{equation}
\label{argument1}
\frac{1}{n} \sum_{j \in \mathcal{C}_i} f(q_{j})=\frac{n_i}{n}\frac{1}{n_i}\sum_{j \in \mathcal{C}_i} f(q_{j}) \asto c_i \int f(q)\mu(dq).
\end{equation}
After some algebra, we obtain $ \frac{1}{n}{\bf J}^{\sf T}{\bf D}_q^{1-\alpha}\bar{{\bf Q}}_\lambda^{\alpha}{\bf D}_q^{1-\alpha}{\bf J} \asto e_{21}^{\alpha}(\lambda)\mathcal{D}({\bf c}) $ where the $ e_{ij} $'s are given in Theorem~\ref{determinst2}. Similarly for the terms $ (1,2) $ and $ (2,2) $, we obtain respectively
\[\frac{1}{({\bf q}^{\sf T}{\bf 1}_n)}{\bf J}^{\sf T}{\bf D}_q{\bf 1}_n+\lambda\frac{1}{({\bf q}^{\sf T}{\bf 1}_n)}{\bf J}^{\sf T}{\bf D}_q^{1-\alpha}\bar{{\bf Q}}_\lambda^{\alpha}{\bf D}_q^{\alpha}{\bf 1}_n \asto \left(1+\frac{\lambda}{m_{\mu}}e_{10}^{\alpha}(\lambda)\right){\bf c}\] and 
\[ \frac{n}{({\bf q}^{\sf T}{\bf 1}_n)^{2}}\Big({\bf 1}_n^{\sf T}{\bf X}{\bf 1}_n+\lambda{\bf 1}_n^{\sf T}{\bf D}_q^{2\alpha}{\bf 1}_n+\lambda^2{\bf 1}_n^{\sf T}{\bf D}_q^{\alpha}\bar{{\bf Q}}_\lambda^{\alpha}{\bf D}_q^{\alpha}{\bf 1}_n\Big) \asto \frac{1}{m_{\mu}^{2}}\left(\lambda v_{\mu}+\lambda^{2}e_{0;-1}^{\alpha}(\lambda)\right)\]
with $ v_{\mu}=\int q^{2\alpha}\mu(\mathrm{d}q) $ and where we have also used the fact that $ \frac{1}{n}{\bf 1}_n^{\sf T}\frac{{\bf X}}{\sqrt{n}}{\bf 1}_n \asto 0 $ again from the law of large numbers.

The limit of ${\bf I}_{K+1}+{\bf U}^{\sf T}{\bf Q}_\lambda^{\alpha}{\bf U}{\bm \Lambda}$ is then obtained as
\begin{align*}
{\bf I}_{K+1}+{\bf U}^{\sf T}{\bf Q}_\lambda^{\alpha}{\bf U}{\bm \Lambda} &\asto \\
&\begin{pmatrix}
{\bf I}_K+e_{21}^{\alpha}(\lambda)(\mathcal{D}({\bf c})-{\bf cc}^{\sf T}){\bf M}({\bf I}_K-{\bf c}{\bf 1}_K^{\sf T})-\left(1+\frac{\lambda}{m_{\mu}}e_{10}^{\alpha}(\lambda)\right){\bf c}{\bf 1}_K^{\sf T} & -e_{21}^{\alpha}(\lambda){\bf c}\\
\frac{\lambda}{m_{\mu}^{2}}\left(v_{\mu}+\lambda e_{0;-1}^{\alpha}(\lambda)\right){\bf 1}_K^{\sf T} & -\lambda\frac{e_{10}^{\alpha}(\lambda)}{m_{\mu}}
\end{pmatrix}.
\end{align*}

Using the Schur complement formula for the determinant of block matrices, we have that the determinant of the RHS matrix is zero whenever
\begin{align*}
-\lambda\frac{e_{10}^{\alpha}(\lambda)}{m_{\mu}}\det &\Big[{\bf I}_K+e_{21}^{\alpha}(\lambda)(\mathcal{D}({\bf c})-{\bf cc}^{\sf T}){\bf M}({\bf I}_K-{\bf c}{\bf 1}_K^{\sf T}) \\  &-\left(1+\frac{\lambda}{m_{\mu}}e_{10}^{\alpha}(\lambda)\right){\bf c}{\bf 1}_K^{\sf T}+\frac{\left(v_{\mu}+\lambda e_{0;-1}^{\alpha}(\lambda)\right)e_{21}^{\alpha}(\lambda)}{m_{\mu}e_{10}^{\alpha}(\lambda)}{\bf c}{\bf 1}_K^{\sf T}\Big]=0
\end{align*}
or equivalently $\det(\underline{{\bf G}}_\lambda)=0$ with $ \underline{{\bf G}}_\lambda $ defined in Theorem~\ref{isoleigen}. The isolated eigenvalues $ \lambda $ of $ {\bf L}_\alpha $, which are the $ \lambda $ for which $ \det({\bf I}_{K+1}+{\bf U}^{\sf T}{\bf Q}_\lambda^{\alpha}{\bf U}{\bm \Lambda})=0  $, are then asymptotically the $ \rho $ such that $ \det(\underline{{\bf G}}_\rho)=0.$ This proves Theorem~\ref{isoleigen}.
\bigskip

From Theorem~\ref{isoleigen}, we now have all the ingredients to determine the conditions under which we may have eigenvalues of $ {\bf L}_\alpha $ which isolate from $ \mathcal{S}^{\alpha} $. Let $ l $ be a non zero eigenvalue of ${\bf G}_\rho^\alpha=(\mathcal{D}\left( \mathbf{c} \right) -\mathbf{cc}^{\sf T})\mathbf{M}(\mathbf{I}_{K}-\mathbf{c}{\bf 1}_K^{T})$. Since $ \det((\mathcal{D}({\bf c})-{\bf cc}^{\sf T}){\bf M}({\bf I}_K-{\bf c1}_K^{\sf T}))=\det(({\bf I}_K-{\bf c1}_K^{\sf T})(\mathcal{D}({\bf c})-{\bf cc}^{\sf T}){\bf M})=\det((\mathcal{D}({\bf c})-{\bf cc}^{\sf T}){\bf M}) $, $ l $ is also a non zero eigenvalue of $\bar{{\bf M}}= (\mathcal{D}({\bf c})-{\bf cc}^{\sf T}){\bf M} $. From Remark~\ref{Mbar}, for each isolated eigenvalue $ \rho $ of $ {\bf L}_\alpha $ we have a one-to-one mapping with a non zero eigenvalue $ l $ of $\bar{{\bf M}}$ such that $ l=-\frac{1}{E_2^{\alpha}(\rho)} $. Hence, to show the existence of isolated eigenvalues of $ {\bf L}_\alpha $, we need to solve for $ \rho \in \mathbb{R}\setminus \mathcal{S}^{\alpha} $, $ l=-\frac{1}{E_2^{\alpha}(\rho)} $ for each non zero eigenvalue $ l $ of $\bar{{\bf M}}$. Precisely, let us write $\mathcal S^{\alpha}=\bigcup_{m=1}^M [S^{\alpha}_{m,-},S^{\alpha}_{m,+}]$ with $S^{\alpha}_{1,-}\leq S^{\alpha}_{1,+}<S^{\alpha}_{2,-}\leq \ldots < S^{\alpha}_{M,+}$ and define $S_{0,+}=-\infty$ and $S_{M+1,-}=+\infty$. Then, recalling that the Stieltjes transform of a real supported measure is necessarily increasing on $\mathbb{R}$, there exist isolated eigenvalues of ${\bf L}^{\alpha}$ in $(S^{\alpha}_{m,+},S^{\alpha}_{m+1,-})$, $m\in\{0,\ldots,M\}$, for all large $n$ almost surely, if and only if there exists eigenvalues $\ell$ of $\bar{{\bf M}}$ such that
\begin{align}
	\label{condition}
	\lim_{x\downarrow S^{\alpha}_{m,+}} E_2^{\alpha}(x) < -\ell^{-1} < \lim_{x\uparrow S^{\alpha}_{m+1,-}} E_2^{\alpha}(x).
\end{align}

In particular, when $\mathcal S^{\alpha}=[S^{\alpha}_-,S^{\alpha}_+]$ is composed of a single connected component (as when $\mathcal S^{\alpha}$ is the support of the semi-circle law as well as most cases met in practice), then isolated eigenvalues of ${\bf L}^{\alpha}$ may only be found beyond $S^{\alpha}_+$ if $\ell>\lim_{x\downarrow S^{\alpha}_+} -\frac{1}{E_2^{\alpha}(x)}$ ($ l>0 $) or below $S^{\alpha}_-$ if $\ell<\lim_{x\uparrow S^{\alpha}_-} - \frac{1}{E_2^{\alpha}(x)}$ ($ l<0 $), for some non-zero eigenvalue $\ell$ of $\bar{{\bf M}}$. For the asymptotic spectrum of $ {\bf L}_\alpha $, $ S^{\alpha}_-=-S^{\alpha}_+ $ as one can show that for any $ z \in \mathbb{R} \setminus \mathcal{S}^\alpha $, $E_2^{\alpha}(-z)=-E_2^{\alpha}(z)$ so that both previous conditions reduce to $ \mid \ell \mid>\lim_{x\downarrow S^{\alpha}_+} - \frac{1}{E_2^\alpha(x)} $ . This proves Corollary~\ref{phasetransition}.
\bigskip

Following Remark~\ref{twocases}, we recall that in the case of interest where $ 1+\theta^{\alpha}(\rho)\neq 0 $ (the case $ 1+\theta^{\alpha}(\rho)=0 $ is studied in Appendix~\ref{sec:nonInfoEig}), $ ({\bf V}_l,{\bf V}_r) $ sets of left and right eigenvectors of $ \underline{{\bf G}}_\rho^\alpha $ are also the same set of eigenvectors of the matrix $ {\bf I}_K+(\mathcal{D}\left( \mathbf{c} \right) -\mathbf{cc}^{\sf T})\mathbf{M}(\mathbf{I}_{K}-\mathbf{c}{\bf 1}_K^{T})$ associated both to zero eigenvalues. Furthermore, since it is more convenient to work with symmetric matrices having the same set of left and right eigenvectors, we show next that $ {\bf V}=\mathcal{D}({\bf c})^{-\frac12}{\bf V}_{r}=\mathcal{D}({\bf c})^{\frac12}{\bf V}_{l}$ is a set of eigenvectors of the symmetric matrix $\mathcal{D}({\bf c})^{\frac12}\left({\bf I}_{K} -\mathbf{1}_K{\bf c}^{\sf T}\right)\mathbf{M}\left(\mathbf{I}_{K}-\mathbf{c}{\bf 1}_K^{T}\right)\mathcal{D}({\bf c})^{\frac12} $.  In fact, $ {\bf V}_r $ is right eigenvector of  $ \left(\mathcal{D}\left( \mathbf{c} \right) -\mathbf{cc}^{\sf T}\right)\mathbf{M}\left(\mathbf{I}_{K}-\mathbf{c}{\bf 1}_K^{T}\right)$ associated to eigenvalue $-\frac{1}{e_{21}^{\alpha}(\rho)} $ if and only if
\begin{align}
\label{eq:p1} &\left(\mathcal{D}\left( \mathbf{c} \right) -\mathbf{cc}^{\sf T}\right)\mathbf{M}\left(\mathbf{I}_{K}-\mathbf{c}{\bf 1}_K^{T}\right){\bf V}_r =\left(-\frac{1}{e_{21}^{\alpha}(\rho)}\right){\bf V}_r \\
\label{eq:p2}&\Leftrightarrow \mathcal{D}\left( \mathbf{c} \right)\left({\bf I}_K-\mathbf{1}_K{\bf c}^{\sf T}\right)\mathbf{M}\left(\mathbf{I}_{K}-\mathbf{c}{\bf 1}_K^{T}\right){\bf V}_r =\left(-\frac{1}{e_{21}^{\alpha}(\rho)}\right){\bf V}_r \\
\label{eq:p3}&\Leftrightarrow \mathcal{D}\left( \mathbf{c} \right)^{\frac12}\left({\bf I}_K-\mathbf{1}_K{\bf c}^{\sf T}\right)\mathbf{M}\left(\mathbf{I}_{K}-\mathbf{c}{\bf 1}_K^{T}\right)\mathcal{D}\left( \mathbf{c} \right)^{\frac12}\mathcal{D}\left( \mathbf{c} \right)^{-\frac12}{\bf V}_r =\left(-\frac{1}{e_{21}^{\alpha}(\rho)}\right)\mathcal{D}\left(\mathbf{c} \right)^{-\frac12}{\bf V}_r
\end{align} 
where in \eqref{eq:p2}, we have just factored out $\mathcal{D}\left( \mathbf{c} \right)$ in the left hand side and in \eqref{eq:p3} we have multiplied both sides by $ \mathcal{D}\left(\mathbf{c} \right)^{-\frac12}. $ This implies that ${\bf V}=\mathcal{D}({\bf c})^{-\frac12}{\bf V}_{r}$ is eigenvector of \\ $\mathcal{D}({\bf c})^{\frac12}\left({\bf I}_{K} -\mathbf{1}_K{\bf c}^{\sf T}\right)\mathbf{M}\left(\mathbf{I}_{K}-\mathbf{c}{\bf 1}_K^{T}\right)\mathcal{D}({\bf c})^{\frac12} $ associated to eigenvalue $ -\frac{1}{e_{21}^{\alpha}(\rho)} $. The same reasonning follows to prove that $ {\bf V}=\mathcal{D}({\bf c})^{\frac12}{\bf V}_{l}$ is eigenvector of the symmetric matrix \\ $\mathcal{D}({\bf c})^{\frac12}\left({\bf I}_{K} -\mathbf{1}_K{\bf c}^{\sf T}\right)\mathbf{M}\left(\mathbf{I}_{K}-\mathbf{c}{\bf 1}_K^{T}\right)\mathcal{D}({\bf c})^{\frac12} $. This proves Remark~\ref{Mbar} concerning the eigenvectors of $ \underline{{\bf G}}_\lambda  $ associated to vanishing eigenvalues.

\subsection{Informative eigenvectors}
\label{app:th3}
In this section, we are interested in characterizing the content of the vectors $ \bar{{\bf v}}_i^\alpha=\frac{{\bf D}_{\bf q}^{\alpha-1}{\bf u}_i^{\alpha}}{\left\|{\bf D}_{\bf q}^{\alpha-1}{\bf u}_i^{\alpha}\right\|} $ (which are used for the classification in the Step~\ref{four} of our algorithm) where $ {\bf u}_i^{\alpha} $ is the eigenvector of $ {\bf L}_\alpha $ corresponding to a unit multiplicity isolated eigenvalue converging to $ \rho $ which is such that $ 1+\theta^{\alpha}(\rho) \neq 0 $ (the case $ 1+\theta^{\alpha}(\rho)=0$ is treated in Appendix~\ref{sec:nonInfoEig}). By writing 
\begin{equation}
\label{eigenvectorStruct}
\bar{{\bf v}}_i^\alpha = \sum_{a=1}^{K} \nu_{i}^{a}{\bf j}_{a} + \sqrt{\sigma_{ii}^{a}}{\bf w}_{i}^{a}
\end{equation}
where $ {\bf w}_{i}^{a} \in \mathbb{R}^{n} $ is a random vector orthogonal to $ {\bf j}_{a}$ of norm $ \sqrt{n_a} $, supported on the indices of $ \mathcal{C}_a $ with identically distributed entries, our goal is to estimate the $ \nu_{i}^{a} $'s and $ \sigma_{ij}^{a} $'s the limits of which are asymptotically the limiting class means $ ({\nu}_{EM}^a)_i $'s and class covariances $ ({\sigma}_{EM}^a)_{ij} $'s that the EM algorithm will find after convergence (whenever it converges to the correct solution) in Step~\ref{four} of Algorithm~\ref{alg:algorithm}. 

In Section~\ref{subsec:eigenvect}, we have shown that the estimation of the $ \nu_{i}^{a} $'s requires the evaluation of $ \frac{1}{n}\frac{{\bf J}^{\sf T}{\bf D}^{\alpha-1}{\bf u}_i^\alpha({\bf u}_i^\alpha)^{\sf T}{\bf D^{\alpha-1}}{\bf J}}{({\bf u}_i^\alpha)^{\sf T}{\bf D^{2(\alpha-1)}}{\bf u}_i^\alpha} $ for $ {\bf u}_i^\alpha $ eigenvector associated to a limiting isolated eigenvalue $ \rho $ with unit multiplicity of $ {\bf L}_\alpha $.  
By residue calculus, we have that 
\begin{equation}
\label{eq:Projection}
\frac{1}{n}{\bf J}^{\sf T}{\bf D}^{\alpha-1}{\bf u}_i^\alpha({\bf u}_i^\alpha)^{\sf T}{\bf D}^{\alpha-1}{\bf J}=-\frac{1}{2\pi i}\oint_{\Gamma_\rho} \frac{1}{n}{\bf J}^{\sf T}{\bf D}^{\alpha-1}\left({\bf L}_\alpha-z{\bf I}_n\right)^{-1}{\bf D}^{\alpha-1}{\bf J}\mathrm{d}z 
\end{equation}
for large $ n $ almost surely, where $ \Gamma_\rho $ is a complex (positively oriented) contour circling around the limiting eigenvalue $ \rho $ only. As from Theorem~\ref{approx}, $ {\bf L}_\alpha ={\bf D}_q^{-\alpha}\frac{{\bf X}}{\sqrt{n}}{\bf D}_q^{-\alpha}+{\bf U}{\bm \Lambda}{\bf U}^{\sf T}+o(1)$, we apply the Woodburry identity to the inverse in the previous integrand and we get 
\begin{multline*} 
\frac{1}{n}{\bf J}^{\sf T}{\bf D}^{\alpha-1}\left({\bf L}_\alpha-z{\bf I}_n\right)^{-1}{\bf D}^{\alpha-1}{\bf J}=\frac{1}{n}{\bf J}^{\sf T}{\bf D}^{\alpha-1}{\bf Q}_z^{\alpha}{\bf D}^{\alpha-1}{\bf J} \\ +\frac{1}{n}{\bf J}^{\sf T}{\bf D}^{\alpha-1}{\bf Q}_z^{\alpha}{\bf U}{\bm \Lambda}\left({\bf I}_{K+1}+{\bf U}^{\sf T}{\bf Q}_z^{\alpha}{\bf U}{\bm \Lambda}\right)^{-1}{\bf U}^{\sf T}{\bf Q}_z^{\alpha}{\bf D}^{\alpha-1}{\bf J}+o(1).
\end{multline*}
The first right-hand side has asymptotically no residue when we integrate over the contour $ \Gamma_\rho $ (as per Proposition~\ref{prop1} there is no eigenvalues of $\bar{{\bf X}}$ in $ \Gamma_\rho $ for all large $ n $ almost surely). We are then left with the second right-most term. Using the block structure used in the proof of Section~\ref{app:th2}, we may write
\begin{align*}
	 &\left({\bf I}_{K+1}+{\bf U}^{\sf T}{\bf Q}_z^{\alpha}{\bf U}{\bm \Lambda}\right)^{-1} \asto \\ 
	 &\begin{pmatrix}
{\bf I}_K+e_{21}^{\alpha}(z)(\mathcal{D}({\bf c})-{\bf cc}^{\sf T}){\bf M}({\bf I}_K-{\bf c}{\bf 1}_K^{\sf T})-\left(1+\frac{z}{m_{\mu}}e_{10}^{\alpha}(z)\right){\bf c}{\bf 1}_K^{\sf T} & -e_{21}^{\alpha}(z){\bf c}\\
\frac{z}{m_{\mu}^{2}}\left(v_{\mu}+ze_{0;-1}^{\alpha}(z)\right){\bf 1}_K^{\sf T} & -z\frac{e_{10}^{\alpha}(z)}{m_{\mu}}
\end{pmatrix}^{-1}.
	 \end{align*}
	 Let us write $ \gamma(z)=\frac{z}{m_{\mu}^{2}}\left(v_{\mu}+ze_{0;-1}^{\alpha}(z)\right)$. We can now use a block inversion formula to write
	  \begin{align}
	  \label{inverseG}
	  \left(\mathbf{I}_{K+1}+\mathbf{U}^{T}\mathbf{Q}_{z}^{\alpha}\mathbf{U\Lambda}\right)^{-1} &\asto \begin{pmatrix}
	(\underline{\bf G}_{z}^{\alpha})^{-1} & -\frac{e_{10}^{\alpha}(z)\left[\underline{\bf G}_{z}^\alpha-\frac{\gamma(z)m_\mu e_{21}^{\alpha}(z)}{ze_{10}^{\alpha}(z)}{\bf c1}_K^{\sf T}\right]^{-1}{\bf c}}{-\frac{ze_{21}^{\alpha}(z)}{m_{\mu}}+\gamma(z)e_{10}^{\alpha}(z){\bf 1}_K^{\sf T}\left[\underline{\bf G}_{z}^\alpha-\frac{\gamma(z)m_\mu e_{21}^{\alpha}(z)}{ze_{10}^{\alpha}(z)}{\bf c1}_K^{\sf T}\right]^{-1}{\bf c}} \\
	\frac{\gamma(z)m_{\mu}}{ze_{21}^{\alpha}(z)}{\bf 1}_K^{\sf T}(\underline{\bf G}_{z}^{\alpha})^{-1} & \frac{1}{-\frac{ze_{21}^{\alpha}(z)}{m_{\mu}}+\gamma(z)e_{10}^{\alpha}(z){\bf 1}_K^{\sf T}\left[\underline{\bf G}_{z}^\alpha-\frac{\gamma(z)m_\mu e_{21}^{\alpha}(z)}{ze_{10}^{\alpha}(z)}{\bf c1}_K^{\sf T}\right]^{-1}{\bf c}}
	 \end{pmatrix}
	 \end{align}
	 with $ \underline{\bf G}_{z}^{\alpha}=\mathbf{I}_{K} + e_{21}^{\alpha}(z)\left(\mathcal{D}\left( \mathbf{c} \right) -\mathbf{cc}^{\sf T}\right)\mathbf{M}\left(\mathbf{I}_{K}-\mathbf{c}{\bf 1}_K^{T}\right)+\theta^{\alpha}(z)\mathbf{c}{\bf 1}_K^{T}.$
	 The entries of the previous matrix seem to be cumbersome but as we will see, the residue calculus will greatly simplify. In fact, we have that $ {\bf 1}_{K}^{T}\underline{\bf G}_{z}^\alpha=\left(1+\theta^{\alpha}(z)\right) {\bf 1}_{K}^{T}$ so that $ {\bf 1}_{K}^{T}(\underline{\bf G}_{z}^\alpha)^{-1}=\frac{1}{1+\theta^{\alpha}(z)} {\bf 1}_{K}^{T}$ which is well defined since we are considering the case $ 1+\theta^{\alpha}(z) \neq 0 $. Similarly, we have that
	 \begin{align*}
	 \left[\underline{\bf G}_{z}^\alpha-\frac{\gamma(z)m_\mu e_{10}^{\alpha}(z)}{ze_{21}^{\alpha}(z)}{\bf c1}_K^{\sf T}\right]{\bf c}= \left(-z\frac{e_{10}^{\alpha}(z)}{m_{\mu}}\right){\bf c}
	 \end{align*} 
	 meaning that $ \left[\underline{\bf G}_{z}^\alpha-\frac{\gamma(z)m_\mu e_{21}^{\alpha}(z)}{ze_{10}^{\alpha}(z)}{\bf c1}_K^{\sf T}\right]^{-1}{\bf c} = -\frac{m_{\mu}}{ze_{10}^{\alpha}(z)}{\bf c}$. So finally, the terms $ (1,2) $, $ (2,1) $ and $ (2,2) $ of $  \left(\mathbf{I}_{K+1}+\mathbf{U}^{T}\mathbf{Q}_{z}^{\alpha}\mathbf{U\Lambda}\right)^{-1} $ do no longer depend on $ (\underline{\bf G}_{z}^\alpha)^{-1} $ and thus do not have poles in the contour $\Gamma_{\rho} $. We can then write
	\begin{align*} \label{apen2}
	  \left(\mathbf{I}_{K+1}+\mathbf{U}^{T}\mathbf{Q}_{z}^{\alpha}\mathbf{U\Lambda}\right)^{-1} &= \begin{pmatrix}
	(\underline{\bf G}_{z}^\alpha)^{-1} & 0 \\
	0  & 0
	 \end{pmatrix} + \mathbf{R}_{1}(z)
	 \end{align*} 
	 with $\mathbf{R}_{1}(z)$ having no residue in the contour $ \Gamma_{\rho} $. Thus, to perform the contour integration of the integrand in \eqref{eq:Projection} around $ \Gamma_\rho $, we just need to evaluate the top-left entries of ${\bf J}^{\sf T}{\bf D}^{\alpha-1}{\bf Q}_z^{\alpha}{\bf U}{\bm \Lambda}$ and ${\bf U}^{\sf T}{\bf Q}_z^{\alpha}{\bf D}^{\alpha-1}{\bf J}$. Those are easily retrieved from the proof of Theorem~\ref{isoleigen} in Section~\ref{app:th2}.
	 
We have in particular $(\frac{1}{\sqrt{n}}{\bf J}^{\sf T}{\bf D}^{\alpha-1}{\bf Q}_z^{\alpha}{\bf U}{\bm \Lambda})_{11}\asto e_{00}^{\alpha}(z)(\mathcal{D}({\bf c})-{\bf cc}^{\sf T}){\bf M}({\bf I}_K-{\bf c1}_K^{\sf T})-\beta^{\alpha}(z){\bf c1}_K^{\sf T}$ where $ \beta^{\alpha}(z)=\frac{1}{m_{\mu}}\left[\int t^{2\alpha-1}\mu(\mathrm{d}t)+e_{-1;-1}^{\alpha}(z)\right]$ and similarly $({\bf U}^{\sf T}{\bf Q}_z^{\alpha}{\bf D}^{\alpha-1}{\bf J})_{11} \asto e_{00}^{\alpha}(z)\mathcal{D}({\bf c}),$ so that finally
\begin{align*}
&\frac{1}{n}{\bf J}^{\sf T}{\bf D}^{\alpha-1}{\bf u}_i^\alpha({\bf u}_i^\alpha)^{\sf T}{\bf D}^{\alpha-1}{\bf J} \asto \\ &-\frac{1}{2\pi i}\oint_{\Gamma_\rho} \left[ \left(e_{00}^{\alpha}(z)(\mathcal{D}({\bf c})-{\bf cc}^{\sf T}){\bf M}({\bf I}_K-{\bf c1}_K^{\sf T})-\beta^{\alpha}(z){\bf c1}_K^{\sf T}\right)(\underline{\bf G}_{z}^\alpha)^{-1}\times e_{00}^{\alpha}(z)\mathcal{D}({\bf c})+\mathbf{R}_{2}(z)\mathrm{d}z \right]
\end{align*}
where $  \mathbf{R}_{2}(z) $ is a matrix having no residue in the considered contour. Now, we are ready to compute the integral. From the Cauchy integral formula,
	 \begin{align*}
	 &\frac{1}{n}{\bf J}^{\sf T}{\bf D}^{\alpha-1}{\bf u}_i^\alpha({\bf u}_i^\alpha)^{\sf T}{\bf D}^{\alpha-1}{\bf J} \asto \\ &\lim_{z \to \rho} \left(z-\rho\right) \left[e_{00}^{\alpha}(z)(\mathcal{D}({\bf c})-{\bf cc}^{\sf T}){\bf M}({\bf I}_K-{\bf c1}_K^{\sf T})-\beta^{\alpha}(z){\bf c1}_K^{\sf T}\right](\underline{\bf G}_{z}^\alpha)^{-1}\times e_{00}^{\alpha}(z)\mathcal{D}({\bf c}).
	 \end{align*} 
	 By writing $ \underline{\bf G}_{z}^\alpha=\lambda_z\mathbf{v}_{r,z}\mathbf{v}_{l,z}^{T}+ \tilde{\mathbf{V}}_{r,z}\tilde{\mathbf{\Sigma}}_{z}\tilde{\mathbf{v}}_{l,z}^{T}$ where $ \mathbf{v}_{r,z}$ and  $\mathbf{v}_{l,z}$ are respectively right and left eigenvectors associated with the vanishing eigenvalue $ \lambda_z $ of  $ \underline{\bf G}_{z}^\alpha$ when $ z \to \rho $; $\tilde{\mathbf{V}}_{r,z} \in \mathbb{R}^{n\times \eta_\rho}$ and  $ \tilde{\mathbf{V}}_{l,z}\mathbb{R}^{n\times \eta_\rho} $ are respectively sets of right and left eigenspaces associated with non vanishing eigenvalues, we then have
	 \[\lim_{z \to \rho} (z-\rho) (\underline{\bf G}_{z}^\alpha)^{-1} \stackrel{(1)}{=} \lim_{z \to \rho} (z-\rho)\frac{\mathbf{v}_{r,z}\mathbf{v}_{l,z}^{T}}{\lambda_z^{'}}\]
	where we have used the l'Hopital rule and the fact that the non vanishing eigenvalue part of $  \underline{\bf G}_{z}^\alpha $ will produce zero in the limit $ z \to \rho $. Using $ \lambda_z=\mathbf{v}_{l,z}^{\sf T}\underline{\bf G}_{z}^\alpha\mathbf{v}_{r,z}$, we obtain
	 \begin{align*}
	  &\frac{1}{n}{\bf J}^{\sf T}{\bf D}^{\alpha-1}{\bf u}_i^\alpha({\bf u}_i^\alpha)^{\sf T}{\bf D}^{\alpha-1}{\bf J} \asto \\ &\left[e_{00}^{\alpha}(\rho)(\mathcal{D}({\bf c})-{\bf cc}^{\sf T}){\bf M}({\bf I}_K-{\bf c1}_K^{\sf T})-\beta^{\alpha}(\rho){\bf c1}_K^{\sf T}\right] \frac{\mathbf{v}_{r,\rho}\mathbf{v}_{l,\rho}^{\sf T}}{\left(\mathbf{v}_{l,z}^{\sf T}\underline{\bf G}_{z}^\alpha\mathbf{v}_{r,z}\right)^{'}_{z=\rho}} \times e_{00}^{\alpha}(\rho)\mathcal{D}({\bf c}).
	 \end{align*} 
	Since $ (\mathbf{v}_{l,\rho})^{\sf T}\underline{\bf G}_{\rho}^\alpha=\underline{\bf G}_{\rho}^\alpha\mathbf{v}_{r,\rho}=0$, 
	\begin{align*}
	\left((\mathbf{v}_{l,z})^{T}\underline{\bf G}_{z}^{\alpha}\mathbf{v}_{r,z}\right)^{'}_{z=\rho}&=((\mathbf{v}_{l,z})^{T})^{'}_{z=\rho}\underline{\bf G}_{\rho}^{\alpha}\mathbf{v}_{r,\rho}+(\mathbf{v}_{l,\rho})^{T}\left(\underline{\bf G}_{z}^{\alpha}\right)^{'}_{z=\rho}\mathbf{v}_{r,\rho}+(\mathbf{v}_{l,\rho})^{T}\underline{\bf G}_{\rho}^{\alpha}(\mathbf{v}_{r,z})^{'}_{z=\rho}\\ &=(\mathbf{v}_{l,\rho})^{T}\left(\underline{\bf G}_{z}^{\alpha}\right)^{'}_{z=\rho}\mathbf{v}_{r,\rho} \\ &=\left(e_{21}^{\alpha}(\rho)\right)^{'}(\mathbf{v}_{l,\rho})^{T}\left(\mathcal{D}\left( \mathbf{c} \right) -\mathbf{cc}^{\sf T}\right)\mathbf{M}\left({\bf I}_K-{\bf c1}_K^{T}\right)\mathbf{v}_{r,\rho}
	\end{align*}
	where the subscript $ ' $ denotes the first derivative with respect to $ z $.
	 Using the fact that $ \mathbf{v}_{r,\rho} $ is orthogonal to $ \mathbf{1}_{K}^{T} $, and $ (\mathbf{v}_{r,\rho},\mathbf{v}_{l,\rho}) $ is also a pair of eigenvectors of $ \left(\mathcal{D}\left( \mathbf{c} \right) -\mathbf{cc}^{\sf T}\right)\mathbf{M}\left({\bf I}_K-{\bf c1}_K^{T}\right) $ associated with eigenvalue $ -\frac{1}{e_{21}^{\alpha}(\rho)} $, we get
 \begin{align}
 \label{projector}
	 \frac{1}{n}{\bf J}^{\sf T}{\bf D}^{\alpha-1}{\bf u}_i^\alpha({\bf u}_i\alpha)^{\sf T}{\bf D}^{\alpha-1}{\bf J} &\asto \frac{\left(e_{00}^{\alpha}(\rho)\right)^{2}}{e_{21}^{\alpha}(\rho)^{'}}\frac{\mathbf{v}_{r,\rho}(\mathbf{v}_{l,\rho})^{T}}{\mathbf{v}_{l,\rho}^{T}\mathbf{v}_{r,\rho}}\mathcal{D}\left( \mathbf{c}\right).
	 \end{align}
	 By introducing $ {\bf v}_\rho=\mathcal{D}({\bf c})^{\frac12}{\bf v}_{l,\rho}=\mathcal{D}({\bf c})^{-\frac12}{\bf v}_{r,\rho}$ eigenvector of the symmetric matrix \\ $ \mathcal{D}({\bf c})^{\frac12}\left({\bf I}_{K} -\mathbf{1}_K{\bf c}^{\sf T}\right)\mathbf{M}\left(\mathbf{I}_{K}-\mathbf{c}{\bf 1}_K^{T}\right)\mathcal{D}({\bf c})^{\frac12} $ (as per Remark~\ref{Mbar}), we obtain the final result
	 \begin{align}
	 \label{apendix}
	  \frac{1}{n}{\bf J}^{\sf T}{\bf D}^{\alpha-1}{\bf u}_i^\alpha({\bf u}_i^\alpha)^{\sf T}{\bf D}^{\alpha-1}{\bf J} &\asto \frac{\left(e_{00}^{\alpha}(\rho)\right)^{2}}{e_{21}^{\alpha}(\rho)^{'}}\mathcal{D}\left(\mathbf{c} \right)^{1/2}{\bf v}_\rho({\bf v}_\rho)^{T}\mathcal{D}\left( \mathbf{c} \right)^{1/2}.
	 \end{align} 

	 Next, we need to estimate the denominator term $ ({\bf u}_i^\alpha)^{\sf T}{\bf D}^{2(\alpha-1)}{\bf u}_i^\alpha$ of $ \frac{1}{n}\frac{{\bf J}^{\sf T}{\bf D}^{\alpha-1}{\bf u}_i^\alpha({\bf u}_i^\alpha)^{\sf T}{\bf D}^{\alpha-1}{\bf J}}{({\bf u}_i^\alpha)^{\sf T}{\bf D}^{2(\alpha-1)}{\bf u}_i^\alpha}.$ For $ {\bf u}_i^{\alpha} $ an eigenvector of $ {\bf L}_\alpha $ associated to an isolated eigenvalue converging to $ \rho $ asymptotically, we have
	 \begin{align*}
	 ({\bf u}_i^\alpha)^{\sf T}{\bf D}^{2(\alpha-1)}{\bf u}_i^\alpha &=\tr({\bf u}_i^\alpha({\bf u}_i^\alpha)^{\sf T}{\bf D}^{2(\alpha-1)}) \\
	 &= \tr\left(-\frac{1}{2\pi i}\oint_{\Gamma_\rho} \left({\bf L}_{\alpha}-z{\bf I}_n\right){\bf D}^{2(\alpha-1)}\mathrm{d}z\right).
	 \end{align*}
	 As in the previous section, by applying Woodburry idendity, this is equivalent to evaluating
	 \[\tr\left(-\frac{1}{2\pi i}\oint_{\Gamma_\rho} \left[{\bf U}^{\sf T}{\bf Q}_z^{\alpha}{\bf D}^{2(\alpha-1)}{\bf Q}_z^{\alpha}{\bf U}{\bm \Lambda}\begin{pmatrix}(\underline{\bf G}_{z}^\alpha)^{-1} & 0 \\ 0 & 0 \end{pmatrix}+{\bf R}_{3}(z)\right]\mathrm{d}z \right), \]
	 where $ {\bf R}_{3}(z) $ is a matrix having no residue in the considered contour.
	 
	 Again here, we just need the top left entry of $ {\bf U}^{\sf T}{\bf Q}_z^{\alpha}{\bf D}^{2(\alpha-1)}{\bf Q}_z^{\alpha}{\bf U}{\bm \Lambda} $ which is given from Theorem~\ref{approx} by 
	 \begin{align}
	 \label{eq10}
	 ({\bf U}^{\sf T}{\bf Q}_z^{\alpha}{\bf D}^{2(\alpha-1)}{\bf Q}_z^{\alpha}{\bf U}{\bm \Lambda})_{11}&=\underbrace{\frac{1}{n}{\bf J}^{\sf T}{\bf D}^{1-\alpha}{\bf Q}_z^{\alpha}{\bf D}^{2(\alpha-1)}{\bf Q}_z^{\alpha}{\bf D}^{1-\alpha}{\bf J}({\bf I}_K-{\bf 1}_K{\bf c}^{\sf T}){\bf M}({\bf I}_K-{\bf c1}_K^{\sf T})}_{({\bf I})} \\ &-\underbrace{\frac{1}{\sqrt{n}({\bf q}^{\sf T}{\bf 1}_n)}{\bf D}^{1-\alpha}{\bf Q}_z^{\alpha}{\bf D}^{2(\alpha-1)}{\bf Q}_z^{\alpha}{\bf D}^{-\alpha}{\bf X1}_n{\bf 1}_K^{\sf T}}_{({\bf II})}.
	 \end{align}
	 We can get rid of the term $ ({\bf II}) $ since after residue calculus, we will get (similar to Equation~\eqref{projector}) $ {\bf 1}_K^{\sf T}{\bf v}_{r,\rho}=0 $ which cancels out the whole term. Let us now concentrate on the term $({\bf I})$.
	 At this point, we need to introduce the following result which, for any deterministic vectors of bounded Euclidean norm $ {\bf a} $, $ {\bf b} $ and any deterministic diagonal matrix $ {\bm \Xi} $, approximates the random quantity $ {\bf a}^{\sf T}{\bf Q}_{z_1}^\alpha{\bm \Xi}{\bf Q}_{z_2}^\alpha{\bf b} $ by a deterministic equivalent.
	 \begin{lemma}[Second deterministic equivalents]\label{determinst3}
	For all $ z \in \mathbb{C}\setminus \mathcal{S}^{\alpha}$, we have the following deterministic equivalent 
	\begin{align*}
\mathbf{Q}_{z_{1}}^{\alpha}{\bm \Xi}\mathbf{Q}_{z_{2}}^{\alpha} &\leftrightarrow \bar{\mathbf{Q}}_{z_{1}}^{\alpha}{\bm \Xi}\bar{\mathbf{Q}}_{z_{2}}^{\alpha} + \bar{\mathbf{Q}}_{z_{1}}^{\alpha}\mathcal{D}\left[ \left(\mathbf{I}_n-{\bm \Upsilon}_{z_{1},z_{2}}\right)^{-1}{\bm \Upsilon}_{z_{1},z_{2}}\diag\left({\bm \Xi}\right)\right] \bar{\mathbf{Q}}_{z_{2}}^{\alpha} \\
	\end{align*}
	where $ {\bm \Xi} $ is any diagonal matrix, $\bar{\mathbf{Q}}_{z}^{\alpha}$ is given  in Lemma~\ref{lemma3} and
	\begin{align*}
		\Upsilon_{z_{1},z_{2}}(i,j)=\frac{1}{n}\frac{q_i^{1-2\alpha}q_j^{1-2\alpha}\left(1-q_{i}q_{j}\right)}{\left(-z_{1}-e_{11}^{\alpha}(z_{1})q_i^{1-2\alpha}+e_{21}^{\alpha}(z_{1})q_i^{2-2\alpha}\right)\left(-z_{2}-e_{11}^{\alpha}(z_{2})q_j^{1-2\alpha}+e_{21}^{\alpha}(z_{2})q_j^{2-2\alpha}\right)}.
	\end{align*}
	The equivalence relation $ \leftrightarrow $ is as defined in Lemma~\ref{lemma3}.
\end{lemma}
	 Thanks to Lemma~\ref{determinst3} (proof provided in Appendix~\ref{apendB}), a deterministic approximation of the term $({\bf I})$ in Equation~\eqref{eq10} can be obtained. We get in particular 
	 \begin{align}
	 \label{equivalent1}
	 \frac{1}{n}{\bf J}^{\sf T}{\bf D}^{1-\alpha}{\bf Q}_z^{\alpha}{\bf D}^{2(\alpha-1)}{\bf Q}_z^{\alpha}{\bf D}^{1-\alpha}{\bf J}&=\frac{1}{n}{\bf J}^{\sf T}{\bf D}^{1-\alpha}\bar{{\bf Q}}_z^{\alpha}{\bf D}^{2(\alpha-1)}\bar{{\bf Q}}_z^{\alpha}{\bf D}^{1-\alpha}{\bf J} \\ &+ \frac{1}{n}{\bf J}^{\sf T}{\bf D}^{1-\alpha}\bar{{\bf Q}}_z^{\alpha}\mathcal{D}\left[\left({\bf I}_n-{\bm \Upsilon}_{z,z} \right)^{-1}{\bm \Upsilon}_{z,z}{\bf d}^{\alpha}\right]\bar{{\bf Q}}_z^{\alpha}{\bf D}^{1-\alpha}{\bf J}
	 \end{align}
	 where $ {\bf d}^{\alpha}=\{q_{i}^{2(\alpha-1)}\}_{i=1}^{n} $ and $ {\bm \Upsilon}_{z_{1},z_{2}} $ was defined in Lemma~\ref{determinst3}. Using similar argument as in Equation~\eqref{argument1}, we can easily show that the first right hand side term of \eqref{equivalent1} converges almost surely to $ e_{00;2}^{\alpha}\mathcal{D}({\bf c}) $. It then remains to estimate the second right-most term of \eqref{equivalent1}. $ {\bm \Upsilon}_{z_1,z_2} $ (as defined in Lemma~\ref{determinst3}) may be written as the sum of two rank-one matrices 
	 \[{\bm \Upsilon}_{z_{1},z_{2}}=\frac{1}{n}\left({\bf a}_{z_1}{\bf a}_{z_2}^{\sf T}-{\bf b}_{z_1}{\bf b}_{z_2}^{\sf T}\right)\]
	 where $ {\bf a}_{z}=\left\{\frac{q_{j}^{1-2\alpha}}{-z-q_{j}^{1-2\alpha}e_{11}^{\alpha}(z)+q_{j}^{2-2\alpha}e_{21}^{\alpha}(z)}\right\}_{j=1}^{n} $ and $ {\bf b}_{z}=\left\{\frac{q_{j}^{2-2\alpha}}{-z-q_{j}^{1-2\alpha}e_{11}^{\alpha}(z)+q_{j}^{2-2\alpha}e_{21}^{\alpha}(z)}\right\}_{j=1}^{n} $.
	 
The matrix ${\bm \Upsilon}_{z,z}$ can thus be further written ${\bm \Upsilon}_{z,z}=\frac{1}{n}\begin{pmatrix} {\bf a}_{z} & {\bf b}_{z}\end{pmatrix}{\bf I}_2\begin{pmatrix} {\bf a}_{z}^{\sf T}/n \\ -{\bf b}_{z}^{\sf T}/n \end{pmatrix} $	 . Using matrix inversion lemmas, we have
\[\left({\bf I}_n-{\bm \Upsilon}_{z,z} \right)^{-1}{\bm \Upsilon}_{z,z}{\bf d}^{\alpha}=\begin{pmatrix} {\bf a}_{z} & {\bf b}_{z}\end{pmatrix}\begin{pmatrix} 1-\frac{{\bf a}_{z}^{\sf T}{\bf a}_{z}}{n}  & -\frac{{\bf a}_{z}^{\sf T}{\bf b}_{z}}{n} \\ \frac{{\bf b}_{z}^{\sf T}{\bf a}_{z}}{n} & 1+\frac{{\bf b}_{z}^{\sf T}{\bf b}_{z}}{n}  \end{pmatrix}^{-1} \begin{pmatrix} \frac{{\bf a}_{z}^{\sf T}{\bf d}^{\alpha}}{n} \\ -\frac{{\bf b}_{z}^{\sf T}{\bf d}^{\alpha}}{n}\end{pmatrix}.\]
Using again the argument in Equation~\eqref{argument1} , we can easily show that $\frac{{\bf a}_{z}^{\sf T}{\bf a}_{z}}{n}$, $\frac{{\bf a}_{z}^{\sf T}{\bf b}_{z}}{n}$, $\frac{{\bf b}_{z}^{\sf T}{\bf b}_{z}}{n}$, $ \frac{{\bf a}_{z}^{\sf T}{\bf d}^{\alpha}}{n} $ and $ \frac{{\bf b}_{z}^{\sf T}{\bf d}^{\alpha}}{n} $ converge for large $ n $ almost surely respectively to $ e_{22;2}^{\alpha}(z) $, $ e_{32;2}^{\alpha}(z) $, $ e_{42;2}^{\alpha}(z) $, $ e_{-1;0}^{\alpha}(z) $ and $ e_{00}^{\alpha}(z) $ with $ e_{ij;2}^\alpha $ defined in Equation~\eqref{eDouble}. This given, we can show that
\[\frac{1}{n}{\bf J}^{\sf T}{\bf D}^{1-\alpha}\bar{{\bf Q}}_z^{\alpha}\mathcal{D}\left[\left({\bf I}_n-{\bm \Upsilon}_{z,z} \right)^{-1}{\bm \Upsilon}_{z,z}{\bf d}^{\alpha}\right]\bar{{\bf Q}}_z^{\alpha}{\bf D}^{1-\alpha}{\bf J} \asto \chi^{\alpha}(z)\mathcal{D}({\bf c})\]
where $ \chi^{\alpha}(z) $ is defined in Theorem~\ref{cor:means}. We thus have
\[ ({\bf u}_i^\alpha)^{\sf T}{\bf D}^{2(\alpha-1)}{\bf u}_i^\alpha \asto \tr \left(\lim_{z\to \rho} \left(e_{00;2}^\alpha(z)+\chi^{\alpha}(z)\right)(\mathcal{D}({\bf c})-{\bf cc}^{\sf T}){\bf M}({\bf I}_K-{\bf c1}_K^{\sf T})(\underline{\bf G}_{z}^\alpha)^{-1}\right).\]
By applying l'Hopital rule to evaluate this limit as in the previous section, we obtain 
\[({\bf u}_i^\alpha)^{\sf T}{\bf D}^{2(\alpha-1)}{\bf u}_i^\alpha \asto \frac{e_{00;2}(\rho)+\chi^{\alpha}(\rho)}{\left(e_{21}^{\alpha}(\rho)\right)^{'}}.\]
Finally,
\begin{equation}
\label{projectorMean}
\frac{1}{n}\frac{{\bf J}^{\sf T}{\bf D}^{\alpha-1}{\bf u}_i^\alpha({\bf u}_i^\alpha)^{\sf T}{\bf D^{\alpha-1}}{\bf J}}{({\bf u}_i^\alpha)^{\sf T}{\bf D}^{2(\alpha-1)}{\bf u}_i^\alpha} \asto \frac{\left(e_{00}^{\alpha}(\rho)\right)^{2}}{e_{00;2}(\rho)+\chi^{\alpha}(\rho)}\mathcal{D}\left(\mathbf{c} \right)^{1/2}{\bf v}_\rho({\bf v}_\rho)^{T}\mathcal{D}\left( \mathbf{c} \right)^{1/2}.
\end{equation}

We recall that one goal of this section is to estimate $ \nu_i^a=\frac{1}{n_a}\frac{{\bf u}_i^{\sf T}{\bf D}^{\alpha-1}{\bf j}_a}{\sqrt{{\bf u}_i^{\sf T}{\bf D}^{2(\alpha-1)}{\bf u}_i}}$, the square of which is $\frac{1}{n_a}\left[\frac{1}{n}\frac{\mathcal{D}({\bf c})^{-\frac12}{\bf J}^{\sf T}{\bf D}^{\alpha-1}{\bf u}_i^\alpha({\bf u}_i^\alpha)^{\sf T}{\bf D}^{\alpha-1}{\bf J}\mathcal{D}({\bf c})^{-\frac12}}{({\bf u}_i^\alpha)^{\sf T}{\bf D}^{2(\alpha-1)}{\bf u}_i^\alpha} \right]_{aa}.$ From Equation~\eqref{projectorMean}, the former quantity is easily retrieved and we have
$$ \left|\nu_i^a\right|^{2}=\frac{1}{n_a}\frac{e_{00}^{\alpha}(\lambda_i)}{\sqrt{e_{00;2}^{\alpha}(\lambda_i)+\chi^{\alpha}(\lambda_i)}} \left|v_{a}^{i}\right|. $$
This proves Theorem~\eqref{cor:means}.
\bigskip

In Section~\ref{subsec:eigenvect}, we have shown that to estimate the $ \sigma_{ij}^{a} $'s, we need to evaluate the more involved object $$ \frac{1}{n}\frac{{\bf J}^{\sf T}{\bf D}^{\alpha-1}{\bf u}_i^\alpha({\bf u}_i^\alpha)^{\sf T}{\bf D}^{\alpha-1}\mathcal{D}_a{\bf D}^{\alpha-1}{\bf u}_j^\alpha({\bf u}_j^\alpha)^{\sf T}{\bf D}^{\alpha-1}{\bf J}}{\left(({\bf u}_i^\alpha)^{\sf T}{\bf D}^{2(\alpha-1)}{\bf u}_i^\alpha\right)\left(({\bf u}_j^\alpha)^{\sf T}{\bf D}^{2(\alpha-1)}{\bf u}_j^\alpha\right)}.$$ Similarly to what was done previously for the estimation of $ \frac{1}{n}\frac{{\bf J}^{\sf T}{\bf D}^{\alpha-1}{\bf u}_i^\alpha({\bf u}_i^\alpha)^{\sf T}{\bf D^{\alpha-1}}{\bf J}}{({\bf u}_i^\alpha)^{\sf T}{\bf D^{2(\alpha-1)}}{\bf u}_i^\alpha} $, we need here to evaluate
	 \[ \left(\frac{1}{2\pi i}\right)^{2} \oint_{\Gamma_{\rho 1}}\oint_{\Gamma_{\rho 2}}\frac{1}{n}{\bf J}^{\sf T}{\bf D}^{\alpha-1}\left({\bf L}_\alpha-z_{1}{\bf I}_n\right)^{-1}{\bf D}^{\alpha-1}{\bf D}_a{\bf D}^{\alpha-1}\left({\bf L}_\alpha-z_{2}{\bf I}_n\right)^{-1}{\bf D}^{\alpha-1}{\bf J}\mathrm{d}z_{1}\mathrm{d}z_{2}\]
	 where $ \Gamma_{\rho_1} $ and $\Gamma_{\rho_2}$ are two positively oriented contours circling around some limiting isolated eigenvalues $ \rho_1 $ and $ \rho_2 $ respectively. We will use the same technique as in the proof of Theorem~\ref{cor:means} to evaluate this integrand. Namely, by applying the Woodburry identity to each of the inverse in the integrand, we get
\begin{align*}
	 \left(\frac{1}{2\pi i}\right)^{2} \oint_{\Gamma_{\rho_1}}\oint_{\Gamma_{\rho_2}}&\frac{1}{n} \mathbf{J}^{T}{\bf D}^{\alpha-1}\mathbf{Q}_{z_{1}}^{\alpha}\mathbf{U \Lambda}\left(\mathbf{I}_{K+1}+\mathbf{U}^{T}\mathbf{Q}_{z_{1}}^{\alpha}\mathbf{U\Lambda}\right)^{-1}\mathbf{U}^{T}\mathbf{Q}_{z_{1}}^{\alpha}{\bf D}^{\alpha-1}{\bf D}_{a}{\bf D}^{\alpha-1}\mathbf{Q}_{z_{2}}^{\alpha}\mathbf{U}\\ &\times {\bm \Lambda}
	 \left(\mathbf{I}_{K+1}+\mathbf{U}^{T}\mathbf{Q}_{z_{2}}^{\alpha}\mathbf{U\Lambda}\right)^{-1}\mathbf{U}^{T}\mathbf{Q}_{z_{2}}^{\alpha}{\bf D}^{\alpha-1}\mathbf{J}\mathrm{d}z_{1}\mathrm{d}z_{2}
	 \end{align*} 
	 where we have used the fact that the cross-terms $ \frac{1}{n}\mathbf{J}^{T}{\bf D}^{\alpha-1}\mathbf{Q}_{z_{i}}^\alpha{\bf D}^{\alpha-1}\mathbf{J}  $, $ i=1,2 $ will vanish asymptotically as the latter do not have poles in the considered contours. By using the identity $ \mathbf{\Lambda}\left(\mathbf{I}_{K+1}+\mathbf{U}^{T}\mathbf{Q}_{z_{1}}^\alpha\mathbf{U\Lambda}\right)^{-1}\mathbf{U}^{T} = \left(\mathbf{I}_{K+1}+\mathbf{\Lambda U}^{T}\mathbf{Q}_{z_{1}}^\alpha\mathbf{U}\right)^{-1}\mathbf{\Lambda}\mathbf{U}^{T}$, the previous integral writes 
	 \begin{align*}
	  \left(\frac{1}{2\pi i}\right)^{2} \oint_{\Gamma_{\rho_1}}\oint_{\Gamma_{\rho_2}} &\frac{1}{n} \mathbf{J}^{T}{\bf D}^{\alpha-1}\mathbf{Q}_{z_{1}}^{\alpha}\mathbf{U \Lambda}\left(\mathbf{I}_{K+1}+\mathbf{U}^{T}\mathbf{Q}_{z_{1}}^{\alpha}\mathbf{U\Lambda}\right)^{-1}\mathbf{U}^{T}\mathbf{Q}_{z_{1}}^{\alpha}{\bf D}^{\alpha-1}{\bf D}_{a}{\bf D}^{\alpha-1}\mathbf{Q}_{z_{2}}^{\alpha}\mathbf{U}\\
	 &\times \left(\mathbf{I}_{K+1}+{\bm \Lambda}\mathbf{U}^{T}\mathbf{Q}_{z_{2}}^{\alpha}\mathbf{U}\right)^{-1}{\bm \Lambda}\mathbf{U}^{T}\mathbf{Q}_{z_{2}}^{\alpha}{\bf D}^{\alpha-1}\mathbf{J}\mathrm{d}z_{1}\mathrm{d}z_{2}
	 \end{align*} 
	 Most of those quantities have been evaluated in the proof of Theorem~\ref{cor:means}. We thus obtain
	\begin{align*} 
	  \left(\frac{1}{2\pi i}\right)^{2} \oint_{\Gamma_{\rho_1}}\oint_{\Gamma_{\rho_2}} &\left[\frac{1}{n} \mathbf{J}^{T}{\bf D}^{\alpha-1}\mathbf{Q}_{z_{1}}^{\alpha}\mathbf{U \Lambda}\begin{pmatrix} (\underline{\bf G}_{z_{1}}^\alpha)^{-1} & 0 \\ 0  & 0 \end{pmatrix}
\mathbf{U}^{T}\mathbf{Q}_{z_{1}}^\alpha{\bf D}^{\alpha-1}{\bf D}_a{\bf D}^{\alpha-1}\mathbf{Q}_{z_{2}}^{\alpha}\mathbf{U}\right. \\
	  &\times \left.\begin{pmatrix} \left((\underline{\bf G}_{z_{2}}^\alpha)^{-1}\right)^{T} & 0 \\ 0  & 0 \end{pmatrix} \mathbf{\Lambda}\mathbf{U}^{T}\mathbf{Q}_{z_{2}}^{\alpha}{\bf D}^{\alpha-1}\mathbf{J}+\mathbf{R}_{4}(z_{1},z_{2})\right]\mathrm{d}z_{1}\mathrm{d}z_{2}
	 \end{align*} 
	 where $ \mathbf{R}_{4}(z_{1},z_{2}) $ has no poles in the considered contours. It is then sufficient to evaluate the top left entry of each of the matrices $ \mathbf{J}^{T}{\bf D}^{\alpha-1}\mathbf{Q}_{z_{1}}^{\alpha}\mathbf{U \Lambda} $, $ \mathbf{U}^{T}\mathbf{Q}_{z_{1}}^{\alpha}{\bf D}^{\alpha-1}{\bf D}_a{\bf D}^{\alpha-1}\mathbf{Q}_{z_{2}}^{\alpha}\mathbf{U} $ and $ \mathbf{\Lambda}\mathbf{U}^{T}\mathbf{Q}_{z_{2}}^{\alpha}{\bf D}^{\alpha-1}\mathbf{J} $ to compute the whole integrand. The first and the third of the latter matrices have been evaluated in the proof of Theorem~\ref{cor:means}. We are then left with the top left entry of $ \mathbf{U}^{T}\mathbf{Q}_{z_{1}}^{\alpha}{\bf D}^{\alpha-1}{\bf D}_a{\bf D}^{\alpha-1}\mathbf{Q}_{z_{2}}^{\alpha}\mathbf{U} $ which is $ \frac{1}{n} \mathbf{J}^{T}{\bf D}^{\alpha-1}\mathbf{Q}_{z_{1}}^\alpha{\bf D}^{\alpha-1}{\bf D}_a{\bf D}^{\alpha-1}\mathbf{Q}_{z_{2}}^\alpha{\bf D}^{\alpha-1}\mathbf{J}$ from Theorem~\ref{approx}.
	The former quantity has already been evaluated in the previous section but for $ (z,z) $ replaced by $ (z_{1},z_{2}) $ and the diagonal matrix between $ {\bf Q}_{z_{1}}^{\alpha} $ and $ {\bf Q}_{z_{2}}^{\alpha} $ being here $ {\bf D}^{\alpha-1}{\bf D}_a{\bf D}^{\alpha-1} $ instead of $ {\bf D}^{2(\alpha-1)} $. We thus have
	 \[ \left(\mathbf{U}^{T}\mathbf{Q}_{z_{1}}^{\alpha}{\bf D}^{\alpha-1}{\bf D}_a{\bf D}^{\alpha-1}\mathbf{Q}_{z_{2}}^{\alpha}\mathbf{U}\right)_{11} \asto c_{a}\left(e_{00;2}^{\alpha}(z_{1},z_{2})\mathcal{D}\left({\bm \delta_{i=a}}\right)_{i=1}^{K}+\chi^{\alpha}(z_{1},z_{2})\mathcal{D}({\bf c})\right).\] Finally, we are left to evaluate
	 \begin{align*} 
	\left(\frac{1}{2\pi i}\right)^{2} \oint_{\Gamma_{\rho_1}}\oint_{\Gamma_{\rho_2}} &\frac{1}{n} \left[e_{00}^{\alpha}(z_{1})(\mathcal{D}({\bf c})-{\bf cc}^{\sf T}){\bf M}({\bf I}_K-{\bf c1}_K^{\sf T})-\beta^{\alpha}(z_{1}){\bf c1}_K^{\sf T}\right] (\underline{\bf G}_{z_{1}}^{\alpha})^{-1}\\
&\times c_{a}\left(e_{00;2}^{\alpha}(z_{1},z_{2})\mathcal{D}\left({\bm \delta_{i=a}}\right)_{i=1}^{K}+\chi^{\alpha}(z_{1},z_{2})\mathcal{D}({\bf c})\right) \\
	&\times \left((\underline{\bf G}_{z_{2}}^\alpha)^{-1}\right)^{\sf T} \left[e_{00}^{\alpha}(z_{2})({\bf I}_K-{\bf 1}_K{\bf c}^{\sf T}){\bf M}(\mathcal{D}({\bf c})-{\bf cc}^{\sf T})-\beta^{\alpha}(z_{2}){\bf 1}_K{\bf c}^{\sf T} \right]\mathrm{d}z_{1}\mathrm{d}z_{2}.
	 \end{align*} 
	We can then perform a residue calculus similar to what was done in the proof of Theorem~\ref{cor:means}. Additionnaly, we use the fact that the eigenvectors $ \mathbf{v}_{\rho_{1}} $ and $ \mathbf{v}_{\rho_{2}} $ corresponding to distinct eigenvalues $ \rho_{1} $ and $ \rho_{2} $ of the symmetric matrix $\mathcal{D}({\bf c})^{\frac12}\left({\bf I}_{K} -\mathbf{1}_K{\bf c}^{\sf T}\right)\mathbf{M}\left(\mathbf{I}_{K}-\mathbf{c}{\bf 1}_K^{T}\right)\mathcal{D}({\bf c})^{\frac12}$ are orthogonals.
	All calculus done, we get
	\begin{align}
	\label{doubleProjector}
&\left(\frac{1}{n}\frac{{\bf J}^{\sf T}{\bf D}^{\alpha-1}{\bf u}_i^\alpha({\bf u}_i^\alpha)^{\sf T}{\bf D}^{\alpha-1}{\bf D}_a{\bf D}^{\alpha-1}{\bf u}_j^\alpha({\bf u}_j^\alpha)^{\sf T}{\bf D}^{\alpha-1}{\bf J}}{\left(({\bf u}_i^\alpha)^{\sf T}{\bf D}^{2(\alpha-1)}{\bf u}_i^\alpha\right)\left(({\bf u}_j^\alpha)^{\sf T}{\bf D}^{2(\alpha-1)}{\bf u}_j^\alpha\right)}\right)_{ef} \asto \nonumber \\ 
&\frac{e_{00}^{\alpha}(\rho_i)e_{00}^{\alpha}(\rho_j)}{\left(e_{00;2}^{\alpha}(\rho_i,\rho_i)+\chi^{\alpha}(\rho_i)\right)\left(e_{00;2}^{\alpha}(\rho_j,\rho_j)+\chi^{\alpha}(\rho_j)\right)}\nonumber \\ & \times \left[e_{00;2}^{\alpha}(\rho_i,\rho_j)\sqrt{c_{e}c_f}v^i_ev^i_f\left(v^i_a\right)^2 +\delta_{\rho_i=\rho_j}c_{a}\chi^{\alpha}(\rho_{i})\sqrt{c_{e}c_f}v^i_ev^i_f\right].
\end{align}
We are thus now ready to evaluate the $ \sigma_{ij}^{a} $'s.
By definition, \begin{equation}
\label{eq:sigma}
\sigma_{ij}^{a}=\left[\frac{({\bf u}_i^\alpha)^{\sf T}{\bf D}^{\alpha-1}{\bf D}_a{\bf D}^{\alpha-1}{\bf u}_j^\alpha}{\sqrt{({\bf u}_i^\alpha)^{\sf T}{\bf D}^{2(\alpha-1)}{\bf u}_i^\alpha}\sqrt{({\bf u}_j^\alpha)^{\sf T}{\bf D}^{2(\alpha-1)}{\bf u}_j^\alpha}} - \frac{1}{n_a}\frac{({\bf u}_i^\alpha)^{\sf T}{\bf D}^{\alpha-1}{\bf j}_a}{\sqrt{({\bf u}_i^\alpha)^{\sf T}{\bf D}^{2(\alpha-1)}{\bf u}_i^\alpha}}\frac{({\bf u}_j^\alpha)^{\sf T}{\bf D}^{\alpha-1}{\bf j}_a}{\sqrt{({\bf u}_j^\alpha)^{\sf T}{\bf D}^{2(\alpha-1)}{\bf u}_j^\alpha}}\right].
\end{equation}
The first right hand side term is estimated by dividing $ \left(\frac{1}{n}\frac{{\bf J}^{\sf T}{\bf D}^{\alpha-1}{\bf u}_i^\alpha({\bf u}_i^\alpha)^{\sf T}{\bf D}^{\alpha-1}{\bf D}_a{\bf D}^{\alpha-1}{\bf u}_j^\alpha({\bf u}_j^\alpha)^{\sf T}{\bf D}^{\alpha-1}{\bf J}}{\left(({\bf u}_i^\alpha)^{\sf T}{\bf D}^{2(\alpha-1)}{\bf u}_i^\alpha\right)\left(({\bf u}_j^\alpha)^{\sf T}{\bf D}^{2(\alpha-1)}{\bf u}_j^\alpha\right)}\right)_{ef}$ (Equation~\eqref{doubleProjector}) by $ \frac{1}{\sqrt{n}}\frac{({\bf u}_i^\alpha)^{\sf T}{\bf D}^{\alpha-1}{\bf j}_e}{\sqrt{({\bf u}_i^\alpha)^{\sf T}{\bf D}^{2(\alpha-1)}{\bf u}_i^\alpha}} \neq 0  $ and $ \frac{1}{\sqrt{n}}\frac{({\bf u}_j^\alpha)^{\sf T}{\bf D}^{\alpha-1}{\bf j}_f}{\sqrt{({\bf u}_j^\alpha)^{\sf T}{\bf D}^{2(\alpha-1)}{\bf u}_j^\alpha}} \neq 0 $ for any couple of indexes $ (e,f) $ such that the aforementioned quantities are non zeros. Indeed from Theorem~\ref{cor:means}, we have obtained
\begin{equation}
\label{simpleProjector}
\frac{1}{\sqrt{n}}\frac{({\bf u}_i^\alpha)^{\sf T}{\bf D}^{\alpha-1}{\bf j}_e}{\sqrt{({\bf u}_i^\alpha)^{\sf T}{\bf D}^{2(\alpha-1)}{\bf u}_i^\alpha}} \asto \sqrt{c_e}\frac{e_{00}^{\alpha}(\rho)}{\sqrt{e_{00;2}^{\alpha}(\rho,\rho)+\chi^{\alpha}(\rho)}}\left|v_{e}^{i}\right|.
\end{equation}
Theorem~\eqref{cor:cov} is thus proved by combining the previous estimates \eqref{doubleProjector} and \eqref{simpleProjector} as per the Definition~\eqref{eq:sigma} of the $ \sigma_{ij}^{a} $'s.

\appendix
\section{Intermediary results}
\subsection{Proof of Lemma~\ref{lm:consistentEstimate}}
\label{subsect:consistentEstimate}
We need to prove that $ \sum_{n=1}^{\infty} \mathbb{P}\left(\max_{1\leq i \leq n}\left|q_i-\hat{q}_i\right|>\eta\right) < \infty$ for any $ \eta>0$ so that we can conclude from the first Borel Cantelli lemma (Theorem $ 4.3 $ in~\cite{billingsley1995probability}) that $ \mathbb{P}\left(\limsup_n \max_{1\leq i \leq n}\left|q_i-\hat{q}_i\right|>\eta \right)=0 $ from which Lemma~\ref{lm:consistentEstimate} unfolds. 
We have that
\begin{align}
\label{probab1}
\mathbb{P}\left(\max_{1\leq i \leq n}\left|q_i-\hat{q}_i\right|>\eta\right) &\leq \sum_{i=1}^{n} \mathbb{P}(\left|q_i-\hat{q}_i\right|>\eta) \nonumber \\
& \leq \sum_{i=1}^{n} \mathbb{P}(\hat{q}_i-q_i>\eta) + \mathbb{P}(q_i-\hat{q}_i>\eta).
\end{align}
Let us treat for instance the term $ \mathbb{P}(\hat{q}_i-q_i>\eta) $ in the following.
Since $ A_{ij}=q_{i}q_{j}+q_{i}q_{j}\frac{M_{g_ig_j}}{\sqrt{n}}+X_{ij}$ with $ X_{ij} $ a zero mean random variable, we have \\ $ \frac1n\sum_{j=1}^n \mathbb{E}A_{ij} \rightarrow q_i m_{\mu} $ and $ \frac{1}{n^2}\sum_{i,j} \mathbb{E}A_{ij} \rightarrow m_{\mu}^2 $ in the limit $ n \rightarrow \infty$. For $ \hat{q}_i=\frac{\sum_{j=1}^n A_{ij}}{\sqrt{\sum_{i,j}A_{ij}}} $ , we can write 
\begin{align*}
\hat{q}_i-q_i &=\underbrace{\frac{\frac1n\sum_{j=1}^n (A_{ij}-\mathbb{E}A_{ij})}{\sqrt{\frac{1}{n^2}\sum_{i,j}\mathbb{E}A_{ij}}}}_{A} +\underbrace{\frac{\frac1n\sum_{j=1}^n A_{ij}}{\sqrt{\frac{1}{n^2}\sum_{i,j} A_{ij}}}-\frac{\frac1n\sum_{j=1}^n A_{ij}}{\sqrt{\frac{1}{n^2}\sum_{i,j}\mathbb{E}A_{ij}}}}_{B} \\&+ \underbrace{\frac{\frac1n\sum_{j=1}^n \mathbb{E}A_{ij}}{\sqrt{\frac{1}{n^2}\sum_{i,j}\mathbb{E}A_{ij}}}-q_i}_{C}.
\end{align*}
Since $ A $, $ B $ and $ C $ tend to zero in the limit $ n \rightarrow \infty$, we will next use the fact that $ \mathbb{P}(\hat{q}_i-q_i>\eta)\leq \mathbb{P}(A>\eta/3)+\mathbb{P}(B>\eta/3)+\mathbb{P}(C>\eta/3) $ and show that all those individual probabilities vanish asymptotically. Since the term $C$ is deterministic and tends to zero in the limit $ n \rightarrow \infty$, we have $ \mathbb{P}(C>\eta/3)=0$ for all large $ n.$ Let us then control $ \mathbb{P}(A>\eta/3) $ and $ \mathbb{P}(B>\eta/3) $. We have 
\begin{align}
\label{controlA}
\mathbb{P}(A>\eta/3) &= \mathbb{P}\left(\frac1n\sum_{j=1}^n (A_{ij}-\mathbb{E}A_{ij})>\frac{\eta m_{\mu}}{3}+o(1)\right) \nonumber \\
&\leq \exp \left[-\frac{n\eta^2 m_{\mu}^2}{18\left(\sigma^2+\eta m_{\mu}/9)\right)}+o(1)\right]
\end{align}
with $ \sigma^2=\limsup_n \max_{1\leq i\leq n} q_i (\sum_j q_j)-q_i^2(\sum_j q_j^2)$ and where in the last inequality of~\eqref{controlA}, we have used Bernstein's inequality (Theorem $ 3 $ in~\cite{boucheron2013concentration}) since the $ A_{ij}$'s are independent Bernoulli random variables with variance $ \sigma_{ij}^2=q_iq_j(1-q_iq_j)+\mathcal{O}(n^{-\frac12}).$ For the term $B$ we have 
\begin{align}
\mathbb{P}(B>\eta/3) &= \mathbb{P} \left(\frac1n\sum_{j=1}^n A_{ij}\frac{\sqrt{\frac{1}{n^2}\sum_{i,j}\mathbb{E}A_{ij}}-\sqrt{\frac{1}{n^2}\sum_{i,j} A_{ij}}}{\sqrt{\frac{1}{n^2}\sum_{i,j} A_{ij}}}>\frac{\eta m_{\mu}}{3}+o(1)\right) \nonumber \\
&\stackrel{(1)}{\leq} \mathbb{P} \left(\left|\frac{\sqrt{\frac{1}{n^2}\sum_{i,j}\mathbb{E}A_{ij}}-\sqrt{\frac{1}{n^2}\sum_{i,j} A_{ij}}}{\sqrt{\frac{1}{n^2}\sum_{i,j} A_{ij}}}\right|>\frac{\eta m_{\mu}}{3}+o(1)\right)\nonumber \\
&\stackrel{(2)}{\leq} \mathbb{P} \left(\left|\sqrt{\frac{1}{n^2}\sum_{i,j}\mathbb{E}A_{ij}}-\sqrt{\frac{1}{n^2}\sum_{i,j}A_{ij}}\right|>\frac{\eta (m_{\mu}+o(1))\sqrt{\frac{1}{n^2}\sum_{i,j}A_{ij}}}{3},\frac{1}{n^2}\sum_{i,j}A_{ij}>\psi\right)\nonumber \\ &+ \mathbb{P} \left(\left|\sqrt{\frac{1}{n^2}\sum_{i,j}\mathbb{E}A_{ij}}-\sqrt{\frac{1}{n^2}\sum_{i,j}A_{ij}}\right|>\frac{\eta m_{\mu}\sqrt{\frac{1}{n^2}\sum_{i,j}A_{ij}}}{3}+o(1),\frac{1}{n^2}\sum_{i,j}A_{ij}\leq \psi\right) \nonumber \\
&\stackrel{(3)}{\leq} \underbrace{\mathbb{P} \left(\left|\sqrt{\frac{1}{n^2}\sum_{i,j}\mathbb{E}A_{ij}}-\sqrt{\frac{1}{n^2}\sum_{i,j}A_{ij}}\right|>\frac{\eta m_{\mu}\sqrt{\psi}}{3}+o(1)\right)}_{B_1} + \underbrace{\mathbb{P} \left(\frac{1}{n^2}\sum_{i,j}A_{ij}\leq \psi\right)}_{B_2}  
\end{align}
where in the inequality $ (1) $ we have used the fact that $n^{-1}\sum_{j=1}^n A_{ij}\leq 1 $; in the inequality $ (2) $ $ \psi>0 $ is any constant smaller than $m_{\mu}^2$ and in the inequality $ (3) $ we have used $ \sqrt{n^{-2}\sum_{i,j} A_{ij}}>\sqrt{\psi} $ and the fact that the probability of the intersection between two events is always smaller than the probability of one of those events. It then remains to control $B_1$ and $B_2$. For $B_2$ we have 
\begin{align}
\label{controlB2}
\mathbb{P} \left(\frac{1}{n^2}\sum_{i,j}A_{ij}\leq \psi\right) &\leq \exp \left[-\frac{n(m_{\mu}^2-\psi)^2}{2\left(\sigma^2+(m_{\mu}^2-\psi)/3\right)}+o(1)\right]
\end{align}
where the inequality follows from Bernstein's inequality with the similar arguments as previously. Finally for the term $ B_1 $ we have
\begin{align}
\label{controlB1}
&\mathbb{P} \left(\left|\sqrt{\frac{1}{n^2}\sum_{i,j}\mathbb{E}A_{ij}}-\sqrt{\frac{1}{n^2}\sum_{i,j}A_{ij}}\right|>\frac{\eta m_{\mu}\sqrt{\psi}}{3}+o(1)\right) \nonumber \\ &=\mathbb{P} \left(\left|\frac{\frac{1}{n^2}\sum_{i,j}\mathbb{E}A_{ij}-\frac{1}{n^2}\sum_{i,j}A_{ij}}{\sqrt{\frac{1}{n^2}\sum_{i,j}\mathbb{E}A_{ij}}+\sqrt{\frac{1}{n^2}\sum_{i,j}A_{ij}}}\right|>\frac{\eta m_{\mu}\sqrt{\psi}}{3}+o(1)\right) \nonumber \\
&\stackrel{(1)}{\leq} \mathbb{P}\left(\left|\frac{1}{n^2}\sum_{i,j}\mathbb{E}A_{ij}-\frac{1}{n^2}\sum_{i,j} A_{ij}\right|>\frac{\eta m_{\mu}\sqrt{\psi}(m_{\mu}+\sqrt{\psi})}{3}+o(1)\right) +\mathbb{P}\left(\frac{1}{n^2}\sum_{i,j}A_{ij}\leq \psi\right) \nonumber \\
&\stackrel{(2)}{\leq} \exp \left[-\frac{n\psi\left[\eta m_{\mu}(m_{\mu}+\sqrt{\psi})\right]^2}{18\left(\sigma^2+\eta m_{\mu}\sqrt{\psi}(m_{\mu}+\sqrt{\psi})/9\right)}+o(1)\right]+\exp \left[-\frac{n(m_{\mu}^2-\psi)^2}{2\left(\sigma^2+(m_{\mu}^2-\psi)/3\right)}+o(1)\right]
\end{align}
where in the inequality $ (1) $ of Equation~\eqref{controlB1} we have used the same arguments as in the inequalities $ (2)-(3) $ of Equation~\eqref{controlA} and in the inequality $ (2) $ we have used Bernstein's inequality along with Equation~\eqref{controlB2}. From Equations~\eqref{controlA}\eqref{controlB2}\eqref{controlB1}, we conclude that $ \sum_{n=1}^{\infty} \sum_{i=1}^n \mathbb{P}(\hat{q}_i-q_i>\eta) < \infty$ since $m_{\mu}^2-\psi>0$. It follows the same lines to show that $ \sum_{n=1}^{\infty} \sum_{i=1}^n \mathbb{P}(q_i-\hat{q}_i>\eta) < \infty$ which concludes the proof.

\subsection{Stein Lemma and Nash Poincare inequality}
\label{SteinNash}
\begin{lemma}
\label{stein}
Let $ x $ be a standard real Gaussian random variable and $ f: \mathbb{R}\rightarrow \mathbb{R} $ be a $ \mathbb{C}^{1} $ function with first derivative $ f^{'}(x) $ having at most polynomial growth. Then,
$$ \mathbb{E}[xf(x)]=\mathbb{E}[f^{'}(x)]. $$
\end{lemma}
\begin{lemma}
\label{nash}
Let $ x $ be a standard real Gaussian random variable and $ f: \mathbb{R}\rightarrow \mathbb{R} $ be a $ \mathbb{C}^{1} $ function with first derivative $ f^{'}(x) $. Then, we have
$$ \Var[f(x)] \leq \mathbb{E}[|f^{'}(x)|^{2}]. $$
\end{lemma}
The proofs of those lemma can be found in \cite{pastur2011eigenvalue}.
\section{First deterministic equivalents}
\label{apendA}
Let $ {\bf Q}_z^\alpha=\left(\bar{{\bf X}}-z{\bf I}_n\right)^{-1} $ with $\bar{{\bf X}}$ a symmetric random matrix having independent entries $ \bar{X}_{ij} $ which are Gaussian random variables with zero mean and variance $ \frac{\sigma_{ij}^2}{n} $. For short, we shall denote $ {\bf Q}_z^\alpha $ by $ {\bf Q} $. We want to find a deterministic equivalent $ \bar{{\bf Q}} $ of $ {\bf Q} $ in the sense that $\frac{1}{n} \tr \mathbf{C}{\bf Q}- \frac{1}{n} \tr \mathbf{C}\bar{{\bf Q}} \to 0$ and $ \mathbf{d}_1^{\sf T}({\bf Q}-\bar{{\bf Q}})\mathbf{d}_{2} \to 0 $ almost surely, for all deterministic Hermitian matrix $ \mathbf{C} $ and deterministic vectors $\mathbf{d}_{i}$ of bounded norms (spectral norm for matrices and Euclidian norm for vectors). To this end, we will evaluate $ \mathbb{E}({\bf Q}) $ since using Lemma~\ref{nash}, one can show that $ n^{-1}\tr({\bf CQ}) $ and $ {\bf a}^{\sf T}{\bf Q}{\bf b} $ concentrate respectively around $n^{-1}\tr({\bf A}\mathbb{E}{Q}) $ and $ \mathbf{d}_1^{\sf T}\mathbb{E}{\bf Q}\mathbf{d}_{2}$ for all bounded norm matrix $ {\bf C} $ and vectors $ \mathbf{d}_{1},\mathbf{d}_{2}$. For the computations, we use standard Gaussian calculus introduced in~\cite{pastur2011eigenvalue}. Using the resolvent identity (for two invertible matrices $ \textbf{A} $ and $ \textbf{B} $, $ \textbf{A}^{-1}-\textbf{B}^{-1}=-\textbf{A}^{-1}(\textbf{A}-\textbf{B})\textbf{B}^{-1}$), one has \begin{equation} \label{resolventId} {\bf Q}=\frac{1}{z}\bar{{\bf X}}{\bf Q}-\frac{1}{z}{\bf I}_n. \end{equation} We then first compute $ \mathbb{E}(\bar{{\bf X}}\textbf{Q}).$ By writing $ \bar{X}_{il}=\frac{\sigma_{il}}{\sqrt{n}}Z_{il} $ where $ Z_{il}$ is a random variable with zero mean and unit variance, we thus have 
\begin{align*}
\mathbb{E}(\bar{{\bf X}}\textbf{Q})_{ij} &= \sum_{l=1}^{n} \frac{\sigma_{il}}{\sqrt{n}}\mathbb{E}(Z_{il}Q_{lj}).
\end{align*}
By applying Stein's Lemma (Lemma~\ref{stein} in Section~\ref{SteinNash}), we have
\begin{align*}
\mathbb{E}(Z_{il}Q_{lj}) &= \mathbb{E}\left( \frac{\partial (\bar{{\bf X}}-z\textbf{I})^{-1}_{lj}}{\partial {Z_{il}}}\right)\\
&= \mathbb{E}\left( -(\bar{{\bf X}}-z\textbf{I})^{-1}\frac{\partial \bar{{\bf X}}}{\partial {Z_{il}}}(\bar{{\bf X}}-z\textbf{I})^{-1}\right)_{lj}\\
&= \mathbb{E}\left( -(\bar{{\bf X}}-z\textbf{I})^{-1}\frac{\sigma_{il}}{\sqrt{n}}(\textbf{E}_{il}+\textbf{E}_{li})(\bar{{\bf X}}-z\textbf{I})^{-1}\right)_{lj}\\
\end{align*} 
where $ \textbf{E}_{il} $ is the matrix with all entries equal to $ 0 $ but the entry $(i,l) $ which is equal to $ 1 $. Using simple algebra, we have 
\[ \left((\bar{{\bf X}}-z\textbf{I})^{-1}\textbf{E}_{il}(\bar{{\bf X}}-z\textbf{I})^{-1}\right)_{lj}=(\bar{{\bf X}}-z\textbf{I})^{-1}_{li}(\bar{{\bf X}}-z\textbf{I})^{-1}_{lj}\]
and 
\[ \left((\bar{{\bf X}}-z\textbf{I})^{-1}\textbf{E}_{li}(\bar{{\bf X}}-z\textbf{I})^{-1}\right)_{lj}=(\bar{{\bf X}}-z\textbf{I})^{-1}_{ll}(\bar{{\bf X}}-z\textbf{I})^{-1}_{ij}.\]
We thus get \[ \mathbb{E}(\bar{{\bf X}}\textbf{Q})_{ij}=\sum_{l=1}^{n} -\frac{\sigma_{il}^2}{n} \left( \mathbb{E}\left[Q_{li}Q_{lj}\right]+\mathbb{E}\left[Q_{ll}Q_{ij}\right] \right).\]
Going back to~\eqref{resolventId}, we thus have
\begin{align}
\label{eq:resolvent1}
 \mathbb{E}(Q_{ij})&=-\frac{1}{z}\sum_{l=1}^{n} \frac{\sigma_{il}^{2}}{n}\mathbb{E}\left[Q_{li}Q_{lj}\right]-\frac{1}{z}\sum_{l=1}^{n} \frac{\sigma_{il}^{2}}{n}\mathbb{E}\left[Q_{ll}Q_{ij}\right] -\frac{1}{z}\delta_{ij} \nonumber \\
 &= -\frac{1}{z} \mathbb{E}\left[\textbf{Q}\frac{{\bm \Sigma}_{i}}{n} \textbf{Q}\right]_{ij}-\frac{1}{z} \mathbb{E}\left[Q_{ij}\tr\left( \frac{{\bm \Sigma}_{i}\textbf{Q}}{n}\right) \right] -\frac{1}{z}\delta_{ij}
\end{align}
where $ {\bm \Sigma}_{i}=\mathcal{D} \left( \sigma_{ij}^2\right)_{j=1}^{n} $. Since the goal is to retrieve $ \mathbb{E}(Q_{ij}) $, the following lemma allows to split $ \mathbb{E}\left[Q_{ij}\tr\left( \frac{{\bm \Sigma}_{i}\textbf{Q}}{n}\right) \right] $ into $ \mathbb{E}\left[Q_{ij}\right] $ and $ \mathbb{E}\left[\tr\left( \frac{{\bm \Sigma}_{i}\textbf{Q}}{n}\right) \right] $.
\begin{lemma}
\label{lemma7}
For $ {\bf Q}=(\bar{{\bf X}}-z{\bf I}_n)^{-1}$ and $ {\bm \Sigma}_i=\mathcal{D}(\sigma_{ij}^2)_{j=1}^n $, where $ \bar{{\bf X}} $ is a symmetric random matrix having independent entries (up to the symmetry) of zero mean and variance $ \frac{\sigma_{ij}^2}{n} $, we have
$$ \mathbb{E}\left[Q_{ij}\tr\left( \frac{{\bm \Sigma}_{i}\textbf{Q}}{n}\right)\right]= \mathbb{E}\left[Q_{ij}\right] \mathbb{E}\left[\tr\left( \frac{{\bm \Sigma}_{i}\textbf{Q}}{n}\right) \right] + o(1).$$
\end{lemma}

\begin{proof}
For two real random variables $ x $ and $ y $, by Cauchy-Shwarz's inequality,
\[ \left|\mathbb{E}\left[(x-\mathbb{E}(x))(y-\mathbb{E}(y))\right] \right| \leq \sqrt{\Var(x)}\sqrt{\Var(y)}\]
which, for $ x= \tr\left( \frac{{\bm \Sigma}_{i}\textbf{Q}}{n}\right)$ and $ y=Q_{ij}-\mathbb{E}(Q_{ij}) $ gives
\[ \left| \mathbb{E}\left[Q_{ij}\tr\left( \frac{{\bm \Sigma}_{i}\textbf{Q}}{n}\right)\right] - \mathbb{E}\left[Q_{ij}\right] \mathbb{E}\left[ \tr\left( \frac{{\bm \Sigma}_{i}\textbf{Q}}{n}\right) \right] \right| \leq \sqrt{\Var(x)}\sqrt{\Var(y)}\]
since $ \mathbb{E}(y) $ is equal to $ 0 $ in that case. Using Nash Poincar\'e inequality (Lemma~\ref{nash} in Section~\ref{SteinNash}), one can show that $ \Var(x)=\mathcal{O}\left( \frac{1}{n^{2}}\right) $\cite{hachem2007deterministic}. Additionally, $ \forall i,j $ and $ z \in \mathbb{C}^{+}$, $\left| Q_{ij}\right| \leq \frac{1}{\left| \Im(z) \right|}$. This finally implies that $ \mathbb{E}\left[Q_{ij}\tr\left( \frac{{\bm \Sigma}_{i}\textbf{Q}}{n}\right)\right] - \mathbb{E}\left[Q_{ij}\right]\mathbb{E}\left[ \tr\left( \frac{{\bm \Sigma}_{i}\textbf{Q}}{n}\right) \right] = \mathcal{O}(n^{-1}) $. 
\end{proof}
Since $ \Im (-z-\mathbb{E}\tr (\frac{{\bm \Sigma}_{i}\textbf{Q}}{n}))<-\Im (z) $ for $ z \in \mathbb{C}^{+} $, $ -z-\mathbb{E}\tr( \frac{{\bm \Sigma}_{i}\textbf{Q}}{n}) $ does not vanish asymptotically. Going back to $ \mathbb{E}(Q_{ij}) $ in Equation~\eqref{eq:resolvent1}, we may then write
\begin{equation}\label{e2}
\mathbb{E}(Q_{ij}) = \frac{\mathbb{E}\left[\textbf{Q}\frac{{\bm \Sigma}_{i}}{n} \textbf{Q}\right]_{ij}+\delta_{ij}}{-z-\mathbb{E}\left[\tr\left( \frac{{\bm \Sigma}_{i}\textbf{Q}}{n}\right) \right]}+\mathcal{O}(n^{-1}).
\end{equation}
Multiplying Equation~\eqref{e2} by $ \frac{\sigma^{2}_{ki}}{n} $, taking $j=i$, summing over $ i $ and scaling by $ n $, we get
\[ \tr \mathbb{E}\left(\frac{{\bm \Sigma}_{k}\textbf{Q}}{n}\right)= \sum_{i=1}^{n} \left( \frac{\mathbb{E}\left[\frac{{\bm \Sigma}_{k}}{n}\textbf{Q}\frac{{\bm \Sigma}_{i}}{n}\textbf{Q}\right]_{ii}+\frac{\sigma^{2}_{ki}}{n}}{-z-\mathbb{E}\left[\tr\left( \frac{{\bm \Sigma}_{i}\textbf{Q}}{n}\right) \right]} \right)+\mathcal{O}(n^{-1}).\]
Using a similar approach to the proof of Lemma~\ref{lemma7}, we can show that $ \sum_{i=1}^{n} \mathbb{E}\left[\frac{{\bm \Sigma}_{k}}{n}\textbf{Q}\frac{{\bm \Sigma}_{i}}{n}\textbf{Q}\right]_{ii}=\mathcal{O}(n^{-1}).$
We thus have
\[ \frac{1}{n}\tr\mathbb{E}\left({\bm \Sigma}_{k}\textbf{Q}\right)= \sum_{i=1}^{n} \frac{\frac{\sigma^{2}_{ki}}{n}}{-z-\frac{1}{n}\mathbb{E}\left[\tr\left({\bm \Sigma}_{i}\textbf{Q}\right) \right]}+o(1).\]
By using standard techniques \cite{hachem2007deterministic}, one can show that the unique solution $ e_i(z) $ to $ e_{i}(z)=\frac{1}{n} \sum_{j=1}^{n} \frac{\sigma^{2}_{ij}}{-z-e_{j}(z)}$ is such that $ \frac{1}{n}\tr\mathbb{E}\left({\bm \Sigma}_{i}\textbf{Q}\right)-e_i(z) \asto 0.$
Going back to Equation~\eqref{e2}, we can thus write for large $ n $
\begin{equation}
\label{eq:resolvent2}
 \mathbb{E}\left[ (-z\textbf{I}-\mathcal{D}(e_{i}(z)))\textbf{Q}\right]_{ij}=\mathbb{E}\left[\textbf{Q}\frac{{\bm \Sigma}_i}{n}\textbf{Q}\right]_{ij} + \delta_{ij}+o(1).
\end{equation}
Let us denote $ {\bm \Xi}=-z\textbf{I}-\mathcal{D}(e_{i}(z)) $. Since $ -z-\mathbb{E}\tr\left( \frac{{\bm \Sigma}_{i}\textbf{Q}}{n}\right)$ is away from zero for $z \in \mathbb{C}^{+}$ so is $ -z-e_i(z) $ and thus $ {\bm \Xi} $ is invertible and bounded. For large $ n $, we can write for a given deterministic matrix $ \textbf{C} $ of bounded norm
\begin{align*}
 \mathbb{E}\left[ \frac{1}{n} \tr \textbf{CQ} \right] &=  \frac{1}{n} \sum_{i,j} ({\bf C}{\bm \Xi}^{-1})_{ji}\mathbb{E}\left( {\bm \Xi}{\bf Q}\right)_{ij}\\
&\stackrel{(1)}{=} \frac{1}{n} \sum_{i,j} ({\bf C}{\bm \Xi}^{-1})_{ji} \left(\mathbb{E}\left[\textbf{Q}\frac{{\bm \Sigma}_i}{n}\textbf{Q}\right]_{ij} + \delta_{ij}\right) +o(1)\\
&= \frac{1}{n} \tr\mathbb{E} \left(\textbf{C}{\bm \Xi}^{-1}\textbf{Q}\frac{{\bm \Sigma}_i}{n}\textbf{Q}\right)+  \frac{1}{n} \tr(\textbf{C}{\bm \Xi}^{-1})+o(1)
\end{align*}
where $ (1) $ follows from Equation~\eqref{eq:resolvent2}. We can then prove that $ \frac{1}{n} \tr \mathbb{E} \left(\textbf{C}{\bm \Xi}^{-1}\textbf{Q}\frac{{\bm \Sigma}_i}{n}\textbf{Q}\right)=\mathcal{O}(n^{-1})$ using a similar approach to the proof of Lemma~\ref{lemma7}.
Hence for large $ n $
\begin{equation}\label{e3}
\mathbb{E}\left[ \frac{1}{n} \tr \textbf{CQ}\right]=\frac{1}{n} \tr (\textbf{C}{\bm \Xi}^{-1})+o(1).
\end{equation}
Similarly, for any vectors $ \textbf{a} $, $ \textbf{b} $ of bounded norms, we may write
\begin{align*}
 \mathbb{E}\left[ \textbf{a}^{*} \textbf{Q}\textbf{b} \right] &=  \sum_{i,j} (\textbf{a}^{*}{\bm \Xi}^{-1})_{i}\mathbb{E}\left( {\bm \Xi}{\bf Q}\right)_{ij}\textbf{b}_{j}\\
&= \sum_{i,j} (\textbf{a}^{*}{\bm \Xi}^{-1})_{i} \mathbb{E}\left[\textbf{Q}\frac{{\bm \Sigma}_i}{n}\textbf{Q}\right]_{ij}\textbf{b}_{j}  +  \textbf{a}^{*}{\bm \Xi}^{-1}\textbf{b} +o(1).
\end{align*}
We also have that $ \sum_{i,j} (\textbf{a}^{*}{\bm \Xi}^{-1})_{i} \mathbb{E}\left[\textbf{Q}\frac{{\bm \Sigma}_i}{n}\textbf{Q}\right]_{ij}\textbf{b}_{j}=\mathcal{O}(n^{-1})$. This can be proved similarly to the proof of Lemma~\ref{lemma7}.
Hence,
\begin{equation}\label{e4}
\mathbb{E}\left[ \textbf{a}^{*} \textbf{Q}\textbf{b} \right] = \textbf{a}^{*}{\bm \Xi}^{-1}\textbf{b}+o(1).
\end{equation}
\section{Second deterministic equivalents}
\label{apendB}
Our goal is to find a deterministic equivalent to the random quantity $ {\bf Q}_{z_1}^{\alpha}{\bm \Xi}{\bf Q}_{z_2}^{\alpha} $ for any diagonal deterministic matrix $ {\bm \Xi} $ where we recall that $ {\bf Q}_{z_1}^{\alpha}=\left(\frac{\bar{{\bf X}}}{\sqrt{n}}-z_1{\bf I}_n\right)^{-1} $ with $ \bar{{\bf X}} $ defined previously in Appendix~\ref{apendA}. The proof follows the same techniques as the proof of the first deterministic equivalent $ {\bf Q}_{z}^{\alpha} $ in Appendix~\ref{apendA} but here, the resolvent identity is either applied on $ {\bf Q}_{z_1}^{\alpha} $ or $ {\bf Q}_{z_2}^{\alpha} $. The technical details will be omitted as the key techniques have already been developped in Appendix~\ref{apendA}. For the sake of readability, we will denote $ {\bf Q}_{z_1}^{\alpha}\equiv {\bf Q}_1$ and $ {\bf Q}_{z_2}^{\alpha}\equiv {\bf Q}_2$. As in Appendix~\ref{apendA}, we will evaluate $ \mathbb{E}({\bf Q}_1{\bm \Xi}{\bf Q}_2).$
By the resolvent identity, we have 
\begin{align*}
\mathbb{E}({\bf Q}_1{\bm \Xi}{\bf Q}_2)_{ij}&=-\frac{1}{z_{1}}\mathbb{E}({\bm \Xi}{\bf Q}_2)_{ij}+\frac{1}{z_{1}}\mathbb{E}({\bf XQ}_1{\bm \Xi}{\bf Q}_2)_{ij} \\
&=-\frac{1}{z_{1}}{\Xi}_{ii}\mathbb{E}({\bf Q}_2)_{ij}+\frac{1}{z_{1}}\mathbb{E} \sum_{k,l}X_{ik}(Q_1)_{kl}{\Xi}_{ll}(Q_2)_{lj}.
\end{align*}
We have from Lemma~\ref{stein}, $ \mathbb{E} \sum_{k,l}X_{ik}(Q_1)_{kl}{\Xi}_{ll}(Q_2)_{lj}=\sum_{k,l}\frac{\sigma_{ik}}{\sqrt{n}}{\Xi}_{ll}\mathbb{E}\frac{\partial [(Q_1)_{kl}(Q_2)_{lj}]}{\partial Z_{ik}} $. By expanding all terms and all calculus done, we obtain 
\begin{align*}
\mathbb{E}({\bf Q}_1{\bm \Xi}{\bf Q}_2)_{ij}&=-\frac{1}{z_{1}}\mathbb{E}({\bm \Xi}{\bf Q}_2)_{ij}-\frac{1}{z_{1}}\sum_{k,l}\frac{\sigma_{ik}^{2}}{n}{\Xi}_{ll} \mathbb{E}\left[\underbrace{(Q_1)_{ki}(Q_1)_{kl}(Q_2)_{lj}}_{(1)}+\underbrace{(Q_1)_{kk}(Q_1)_{il}(Q_2)_{lj}}_{(2)}\right. \\ &\left. +\underbrace{(Q_1)_{kl}(Q_2)_{li}(Q_2)_{kj}}_{(3)}+\underbrace{(Q_1)_{kl}(Q_2)_{lk}(Q_2)_{ij}}_{(4)}\right].
\end{align*}
Asymptotically, the non vanishing terms are $ (2) $ and $ (4) $ so that 
\begin{align}
\label{eq:IJ} 
 \mathbb{E}({\bf Q}_1{\bm \Xi}{\bf Q}_2)_{ij}&=-\frac{1}{z_{1}}\mathbb{E}({\bm \Xi}{\bf Q}_2)_{ij}-\frac{1}{z_{1}}\sum_{k,l}\frac{\sigma_{ik}^{2}}{n}{\Xi}_{ll} \mathbb{E}\left[(Q_1)_{kk}(Q_1)_{il}(Q_2)_{lj}+(Q_1)_{kl}(Q_2)_{lk}(Q_2)_{ij}\right] \nonumber \\ &+o(1) \nonumber \\
&=-\frac{1}{z_{1}}\mathbb{E}({\bm \Xi}{\bf Q}_2)_{ij}-\frac{1}{z_{1}}\frac{1}{n}\mathbb{E}\left[\tr({\bm \Sigma}_i{\bf Q}_1)({\bf Q}_1{\bm \Xi}{\bf Q}_2)_{ij}\right]-\frac{1}{z_{1}}\frac{1}{n}\mathbb{E}\left[\tr({\bm \Sigma}_i{\bf Q}_2{\bm \Xi}{\bf Q}_1)({\bf Q}_2)_{ij}\right] \nonumber \\ &+o(1).
\end{align}
Similarly to what was done in the proof of of Lemma~\ref{lemma7}, we can show that \\ $ \mathbb{E}\frac{1}{n}\tr({\bm \Sigma}_i{\bf Q}_1)\mathbb{E}({\bf Q}_1{\bm \Xi}{\bf Q}_2)_{ij}=\mathbb{E}\left(\frac{1}{n}\tr({\bm \Sigma}_i{\bf Q}_1)\right)\mathbb{E}\left(({\bf Q}_1{\bm \Xi}{\bf Q}_2)_{ij}\right)+o(1).$ We can then write from~\eqref{eq:IJ}
\begin{align}
\label{eq:doubleD1}
&\mathbb{E}\left(\left({\bf I}_n+\frac{1}{z_{1}}\mathcal{D}\left(\frac{1}{n}\tr({\bm \Sigma}_i{\bf Q}_1)\right)_{i=1}^{n}\right){\bf Q}_1{\bm \Xi}{\bf Q}_2\right)= \nonumber \\ & -\frac{1}{z_{1}}\mathbb{E}\left({\bm \Xi}+\mathcal{D}\left(\frac{1}{n}\tr({\bm \Sigma}_i{\bf Q}_2{\bm \Xi}{\bf Q}_1)\right)_{i=1}^{n}\right){\bf Q}_2 +o(1).
\end{align}
From~\eqref{eq:doubleD1} and the result of Lemma~\ref{lemma3}, this entails 
\begin{equation}
\label{eq:doubleD2}
\mathbb{E}\left({\bf Q}_1{\bm \Xi}{\bf Q}_2\right) \longleftrightarrow \bar{{\bf Q}}_1{\bm \Xi}\bar{{\bf Q}}_2+\bar{{\bf Q}}_1\mathcal{D}\left(\mathbb{E}\frac{1}{n}\tr({\bm \Sigma}_i{\bf Q}_2{\bm \Xi}{\bf Q}_1)\right)_{i=1}^{n}\bar{{\bf Q}}_2.
\end{equation}
Every object in~\eqref{eq:doubleD2} is known but $ \mathbb{E}\frac{1}{n}\tr({\bm \Sigma}_i{\bf Q}_2{\bm \Xi}{\bf Q}_1) $ which we need to evaluate now. By left-multiplying~\eqref{eq:doubleD2} by $ {\bm \Sigma}_i $ and taking the normalized trace, we get
\begin{equation}
\label{eq:32}
\mathbb{E} \frac{1}{n}\tr({\bm \Sigma}_i{\bf Q}_2{\bm \Xi}{\bf Q}_1)=\frac{1}{n}\tr\left({\bm \Sigma}_i\bar{{\bf Q}}_1{\bm \Xi}\bar{{\bf Q}}_2\right)+\mathbb{E} \frac{1}{n}\tr\left({\bm \Sigma}_i\bar{{\bf Q}}_1\mathcal{D}\left(\frac{1}{n}\tr({\bm \Sigma}_i{\bf Q}_2{\bm \Xi}{\bf Q}_1)\right)_{i=1}^{n}\bar{{\bf Q}}_2\right).
\end{equation}
By denoting $ f_i=\frac{1}{n}\mathbb{E}\left(\tr({\bm \Sigma}_i{\bf Q}_2{\bm \Xi}{\bf Q}_1)\right)$, Equation~\eqref{eq:32} leads to
\begin{equation*}
{\bf f}=\left\{\frac{1}{n}\tr\left({\bm \Sigma}_i\bar{{\bf Q}}_1{\bm \Xi}\bar{{\bf Q}}_2\right)\right\}_{i=1}^{n}+\frac{1}{n}\left\{\left(\bar{{\bf Q}}_2{\bm \Sigma}_i\bar{{\bf Q}}_1\right)_{jj}\right\}_{i,j=1}^{n}{\bf f}
\end{equation*}
which finally entails 
\begin{equation*}
{\bf f}=\left({\bf I}_n-\frac{1}{n}\left\{\left(\bar{{\bf Q}}_2{\bm \Sigma}_i\bar{{\bf Q}}_1\right)_{jj}\right\}_{i,j=1}^{n}\right)^{-1}\frac{1}{n}\left\{\bar{{\bf Q}}_2{\bm \Sigma}_i\bar{{\bf Q}}_1\right\}_{i,j=1}^{n}\diag({\bm \Xi}).
\end{equation*}
To complete the proof of Lemma~\ref{determinst3}, we need to show that $ \Var\left(\frac{1}{n}\tr({\bf Q}_1{\bm \Xi}{\bf Q}_2)\right) $\\ and $ \Var\left(\frac{1}{n}\tr({\bm \Sigma}_i{\bf Q}_2{\bm \Xi}{\bf Q}_1)\right) $ are asymptotically summable so that by the Borell Cantelli Lemma, $ \frac{1}{n}\tr({\bf Q}_1{\bm \Xi}{\bf Q}_2) $ and $ \frac{1}{n}\tr({\bm \Sigma}_i{\bf Q}_2{\bm \Xi}{\bf Q}_1) $ converge respectively almost surely to their expectations. Those follow directly by using Nash Poincare inequality (Lemma~\ref{nash}) similarly to what was done in the proof of Lemma~\ref{lemma7}.

%
\section{Non informative eigenvectors}
\label{sec:nonInfoEig}
The objective of this section is to show that the eigenvectors $\tilde{{\bf u}}^\alpha $ of $ {\bf L}_\alpha $ associated to the limiting eigenvalue $ \tilde{\rho} $ for which $ 1+\theta^{\alpha}(\tilde{\rho})=0 $ (Remark~\ref{twocases}) are not useful for the classification. 

Let us write as in Section~\ref{app:th3}
\begin{equation}
\label{eigenvectorStruct2}
\tilde{{\bf u}}^\alpha = \sum_{a=1}^{K} \tilde{\nu}^{a}{\bf j}_{a} + \sqrt{\tilde{\sigma}_{ii}^{a}}{\bf w}^{a}
\end{equation}
where $ {\bf w}^{a} \in \mathbb{R}^{n} $ is a random vector orthogonal to $ {\bf j}_{a}$ of norm $ \sqrt{n_a} $, supported on the indices of $ \mathcal{C}_a $ with identically distributed entries. We shall show that $ \tilde{\nu}^{a} $ is independent of class $ \mathcal{C}_a $ and thus, any correct classification cannot be done using $ \tilde{{\bf u}}^\alpha $. From~\eqref{eigenvectorStruct2}, $ \tilde{\nu}^{a}=\frac{(\tilde{{\bf u}}^\alpha)^{\sf T}{\bf j}_a}{n_a} $ which can be retrieved from the diagonal elements of $ \frac{1}{n}{\bf J}^{\sf T}\tilde{{\bf u}}^\alpha(\tilde{{\bf u}}^\alpha)^{\sf T}{\bf J} $. We will evaluate this object by using the same technique as in Section~\ref{app:th3}.
By the residue formula, we have
\begin{align}
\label{eq:projection}
\frac{1}{n}{\bf J}^{\sf T}\tilde{{\bf u}}^\alpha(\tilde{{\bf u}}^\alpha)^{\sf T}{\bf J}&=-\frac{1}{2\pi i}\oint_{\Gamma_{\tilde{\rho}}} \frac{1}{n}{\bf J}^{\sf T}\left({\bf L}_\alpha-z{\bf I}_n\right)^{-1}{\bf J}\mathrm{d}z \\ &= -\frac{1}{2\pi i}\oint_{\Gamma_{\tilde{\rho}}} \frac{1}{n}{\bf J}^{\sf T}{\bf Q}_z^{\alpha}{\bf J} \mathrm{d}z
+ \frac{1}{2\pi i}\oint_{\Gamma_{\tilde{\rho}}} \frac{1}{n}{\bf J}^{\sf T}{\bf Q}_z^{\alpha}{\bf U}{\bm \Lambda}\left({\bf I}_{K+1}+{\bf U}^{\sf T}{\bf Q}_z^{\alpha}{\bf U}{\bm \Lambda}\right)^{-1}{\bf U}^{\sf T}{\bf Q}_z^{\alpha}{\bf J}
\end{align}
for large $ n $ almost surely, where $\Gamma_{\tilde{\rho}}$ is a complex (positively oriented) contour circling around the limiting eigenvalue $\tilde{\rho}$ only. The first integral $-\frac{1}{2\pi i}\oint_{\Gamma_{\tilde{\rho}}} \frac{1}{n}{\bf J}^{\sf T}{\bf Q}_z^{\alpha}{\bf J} \mathrm{d}z $ is asymptotically zero since, from Proposition~\ref{prop1}, the integrand has no poles in the contour $ \Gamma_{\tilde{\rho}} $. We thus obtain similarly as in Section~\ref{app:th3}
\begin{equation}
\frac{1}{n}{\bf J}^{\sf T}\tilde{{\bf u}}^\alpha(\tilde{{\bf u}}^\alpha)^{\sf T}{\bf J}= \frac{1}{n}{\bf J}^{\sf T}{\bf Q}_{\tilde{\rho}}^{\alpha}{\bf U}{\bm \Lambda}\left[\lim_{z \to \tilde{\rho}} (z-\tilde{\rho})({\bf I}_{K+1}+{\bf U}^{\sf T}{\bf Q}_z^{\alpha}{\bf U}{\bm \Lambda})^{-1}\right]{\bf U}^{\sf T}{\bf Q}_{\tilde{\rho}}^{\alpha}{\bf J}.
\end{equation}
From~\eqref{inverseG}, the entries $ (1,2) $ and $ (2,2) $ of $ ({\bf I}_{K+1}+{\bf U}^{\sf T}{\bf Q}_z^{\alpha}{\bf U}{\bm \Lambda})^{-1} $ do not contain $ (\underline{{\bf G}}_z^\alpha)^{-1} $ since $ \left[\underline{\bf G}_{z}^\alpha-\frac{\gamma(z)m_\mu e_{21}^{\alpha}(z)}{ze_{10}^{\alpha}(z)}{\bf c1}_K^{\sf T}\right]^{-1}{\bf c} = -\frac{m_{\mu}}{ze_{10}^{\alpha}(z)}{\bf c}$ and thus, the above limit will give zero for those entries. We thus get
\begin{equation}
\frac{1}{n}{\bf J}^{\sf T}\tilde{{\bf u}}^\alpha(\tilde{{\bf u}}^\alpha)^{\sf T}{\bf J}= \frac{1}{n}{\bf J}^{\sf T}{\bf Q}_{\tilde{\rho}}^{\alpha}{\bf U}{\bm \Lambda}\left[\lim_{z \to \tilde{\rho}} (z-\tilde{\rho})\begin{pmatrix}
	 (\underline{\bf G}_{z}^\alpha)^{-1} & 0 \\
	\frac{\gamma(z)m_{\mu}}{ze_{21}^{\alpha}(z)}{\bf 1}_K^{\sf T}(\underline{\bf G}_{z}^{\alpha})^{-1}  & 0
	 \end{pmatrix}\right]{\bf U}^{\sf T}{\bf Q}_{\tilde{\rho}}^{\alpha}{\bf J}.
\end{equation}
We recall that in the case under study ($ 1+\theta^{\alpha}(\tilde{\rho})=0 $), $ {\bf 1}_K $ and $ {\bf c} $ are respectively left and right eigenvectors of $ \underline{{\bf G}}_z^{\alpha} $ associated to the vanishing eigenvalue. We can thus write $  \underline{{\bf G}}_z^{\alpha}=\lambda_z{\bf c}{\bf 1}_K^{\sf T}+ \tilde{\mathbf{V}}_{r,z}\tilde{\mathbf{\Sigma}}_{z}\tilde{\mathbf{V}}_{l,z}^{T}$ where $ \lambda_z $ is the vanishing eigenvalue when $ z \to \tilde{\rho} $ and $\tilde{\mathbf{V}}_{r,z}$ and  $ \tilde{\mathbf{V}}_{l,z}$ are respectively sets of right and left eigenspaces associated with non vanishing eigenvalues. Hence, we have
\begin{equation}
\lim_{z \to \tilde{\rho}} (z-\tilde{\rho}) (\underline{\bf G}_{z}^\alpha)^{-1}\stackrel{(1)}{=}\lim_{z \to \tilde{\rho}} \frac{{\bf c1}_K^{\sf T}}{ \lambda_z^{'}}\stackrel{(2)}{=}\lim_{z \to \tilde{\rho}} \frac{{\bf c1}_K^{\sf T}}{{\bf 1}_K^{\sf T}(\underline{\bf G}_{z}^\alpha)^{'}{\bf c}}\stackrel{(3)}{=}\lim_{z \to \tilde{\rho}} \frac{{\bf c1}_K^{\sf T}}{(\theta^{\alpha}(z))^{'}}\stackrel{(4)}{=}\frac{{\bf c1}_K^{\sf T}}{(\theta^{\alpha}(\tilde{\rho}))^{'}}
\end{equation}
where in $ (1) $ we have used the l'Hopital rule, in $ (2) $ we used the fact that $ \lambda_z $ can be written $ \lambda_z={\bf 1}^{\sf T}_K\underline{\bf G}_{z}^\alpha{\bf c}$ and in $ (3) $ we have used $(\underline{\bf G}_{\tilde{\rho}}^\alpha)^{'}=\left(e_{21}^{\alpha}(\tilde{\rho})\right)^{'}\left(\mathcal{D}\left( \mathbf{c} \right) -\mathbf{cc}^{\sf T}\right)\mathbf{M}\left({\bf I}_K-{\bf c1}_K^{T}\right)+(\theta^{\alpha}(\tilde{\rho}))^{'}{\bf c1}_K^{T}$ and $ {\bf 1}_K^{\sf T}{\bf c}=1.$
We then have
\begin{align}
\label{nonInfoProj}
\frac{1}{n}{\bf J}^{\sf T}\tilde{{\bf u}}^\alpha(\tilde{{\bf u}}^\alpha)^{\sf T}{\bf J}&= \frac{1}{n}({\bf J}^{\sf T}{\bf Q}_{\tilde{\rho}}^{\alpha}{\bf U}{\bm \Lambda})_{11}\frac{{\bf c1}_K^{\sf T}}{(\theta^{\alpha}(\tilde{\rho}))^{'}}({\bf U}^{\sf T}{\bf Q}_{\tilde{\rho}}^{\alpha}{\bf J})_{11} \\ &+\frac{1}{n}({\bf J}^{\sf T}{\bf Q}_{\tilde{\rho}}^{\alpha}{\bf U}{\bm \Lambda})_{12}\frac{\gamma(\tilde{\rho})m_{\mu}{\bf 1}_K^{\sf T}}{\tilde{\rho} e_{21}^{\alpha}(\tilde{\rho})(\theta^{\alpha}(\tilde{\rho}))^{'}}({\bf U}^{\sf T}{\bf Q}_{\tilde{\rho}}^{\alpha}{\bf J})_{11}.
\end{align}
All calculus done similarly as in Section~\ref{app:th3}, we get
\begin{align*}
\frac{1}{n}{\bf J}^{\sf T}\tilde{{\bf u}}^\alpha(\tilde{{\bf u}}^\alpha)^{\sf T}{\bf J} \asto \frac{e_{1,\frac12}^{\alpha}(\tilde{\rho})}{(\theta^{\alpha}(\tilde{\rho}))^{'}} &\left[e_{1,\frac12}^{\alpha}(\tilde{\rho})\left(\mathcal{D}\left( \mathbf{c} \right) -\mathbf{cc}^{\sf T}\right)\mathbf{M}\left({\bf I}_K-{\bf c1}_K^{T}\right) \right. \\ &\left. -\frac{1}{m_\mu}\left(\int t^{\alpha}\mu(\mathrm{d}t)+e_{0,\frac12}^{\alpha}(\tilde{\rho})\right){\bf c1}_K^{\sf T}\right]{\bf cc}^{\sf T} \\ &-(e_{1,\frac12}^{\alpha}(\tilde{\rho}))^2\frac{\gamma(\tilde{\rho})m_{\mu}}{\tilde{\rho} e_{21}^{\alpha}(\tilde{\rho})(\theta^{\alpha}(\tilde{\rho}))^{'}}{\bf cc}^{\sf T}.
\end{align*}
Finally,
\begin{align}
\label{NonInfoProj2}
\frac{1}{n}{\bf J}^{\sf T}\tilde{{\bf u}}^\alpha(\tilde{{\bf u}}^\alpha)^{\sf T}{\bf J} &\asto -\frac{e_{1,\frac12}^{\alpha}(\tilde{\rho})}{m_{\mu}(\theta^{\alpha}(\tilde{\rho}))^{'}}\left[\int t^{\alpha}\mu(\mathrm{d}t)+e_{0,\frac12}^{\alpha}(\tilde{\rho})+\frac{e_{1,\frac12}^\alpha(\tilde{\rho})\gamma(\tilde{\rho})m_{\mu}^2}{\tilde{\rho}e_{21}^{\alpha}(\tilde{\rho})}\right]{\bf cc}^{\sf T}.
\end{align}
By recalling that $ \tilde{\nu}^{a}=\frac{(\tilde{{\bf u}}^\alpha)^{\sf T}{\bf j}_a}{n_a}=\sqrt{\frac{1}{n_a}\left[\frac{1}{n}\mathcal{D}({\bf c})^{-\frac12}{\bf J}^{\sf T}\tilde{{\bf u}}^\alpha(\tilde{{\bf u}}^\alpha)^{\sf T}{\bf J}\mathcal{D}({\bf c})^{-\frac12}\right]_{aa}}$, from~\eqref{NonInfoProj2} we deduce that
\begin{align*}
\tilde{\nu}^{a} &\asto -\frac{1}{\sqrt{n}}\frac{e_{1,\frac12}^{\alpha}(\tilde{\rho})}{m_{\mu}(\theta^{\alpha}(\tilde{\rho}))^{'}}\left[\int t^{\alpha}\mu(\mathrm{d}t)+e_{0,\frac12}^{\alpha}(\tilde{\rho})+\frac{e_{1,\frac12}^\alpha(\tilde{\rho})\gamma(\tilde{\rho})m_{\mu}^2}{\tilde{\rho}e_{21}^{\alpha}(\tilde{\rho})}\right]
\end{align*}
which is independent of the class information (class proportions or inter-class affinities). This concludes the proof.

\bibliography{biblio1,BSTLclustering}

\end{document}